\documentclass[article,10pt]{article}
\usepackage{stmaryrd}
\usepackage{amsfonts}
\usepackage{bbm}
\usepackage{amscd}
\usepackage{mathrsfs}
\usepackage{latexsym,amssymb,amsmath,amscd,amscd,amsthm,amsxtra}
\usepackage[dvips]{graphicx}
\usepackage[utf8]{inputenc}
\usepackage[T1]{fontenc}
\usepackage{lmodern}
\usepackage{amssymb}
\usepackage[all]{xy}
\usepackage{nicefrac,mathtools,enumitem}
\usepackage{microtype}
\usepackage{CJK}
\usepackage{amssymb}
\usepackage{epstopdf}
\usepackage{graphicx}
\usepackage{graphics}
\usepackage[T1]{fontenc}
\usepackage[utf8]{inputenc}
\usepackage{authblk}
\usepackage{subfigure}
\usepackage{caption}
\usepackage{framed}
\usepackage{mathrsfs}
\usepackage[noend]{algpseudocode}
\usepackage{algorithmicx,algorithm}
\usepackage{framed}
\usepackage{booktabs}
\usepackage{threeparttable}
\usepackage{color}
\usepackage{subfigure}
\usepackage{diagbox}
\usepackage{verbatim}
\usepackage{amsmath, bm}

\usepackage{amsfonts}
\usepackage{amssymb,amsmath,amsthm, enumerate}
\usepackage{latexsym}
\usepackage{fullpage}

\usepackage{cite}

\usepackage{hyperref}

\usepackage{mathtools}

\numberwithin{equation}{section}
\newtheorem{thm}{Theorem}
\newtheorem{lem}{Lemma}

\newtheorem{pro}{Proposition}

\let\today\relax
\makeatletter
\def\ps@pprintTitle{%
	\let\@oddhead\@empty
	\let\@evenhead\@empty
	\def\@oddfoot{\footnotesize\itshape
		{Submitted preprint} \hfill\today}%
	\let\@evenfoot\@oddfoot
}
\makeatother

\newcommand {\emptycomment}[1]{}

\newcommand{\be }{\begin{equation}}
	\newcommand{\ee }{\end{equation}}

\newcommand{\huaA}{\mathcal{A}}
\newcommand{\huaL}{\mathcal{L}}
\newcommand{\huaR}{\mathcal{R}}
\newcommand{\huaE}{\mathcal{E}}

\newcommand{\huaG}{\mathcal{G}}

\newcommand{\huaV}{\mathcal{V}}

\newcommand{\huaX}{\mathcal{X}}
\newcommand{\huaY}{\mathcal{Y}}
\newcommand{\huaQ}{\mathcal{Q}}
\newcommand{\huaP}{\mathcal{P}}

\newcommand{\huaH}{\mathcal{H}}

\newcommand{\huaO}{\mathcal{O}}
\newcommand{\huaT}{\mathcal{T}}

\newcommand{\huaN}{\mathcal{N}}

\newcommand{\nono}{\nonumber}
\newcommand{\noi}{\noindent}

\allowdisplaybreaks

\newcommand{\f}{\frac}

\newcommand{\aaa}{\alpha}
\newcommand{\nn}{\langle}
\newcommand{\mm}{\rangle}
\newcommand{\ww}{\widetilde}
\newcommand{\la}{\lambda}
\newcommand{\mbb}{\mathbb}

\newcommand{\kkk}{\kappa}
\newcommand{\rrrr}{\right}
\newcommand{\llll}{\left}
\newcommand{\wh}{\widehat}

\def\bea{\begin{eqnarray}}
	\def\eea{\end{eqnarray}}
\def\be{\begin{equation}}
	\def\ee{\end{equation}}
\def\blm{\begin{lem}}
	\def\elm{\end{lem}}
\def\bes{\begin{eqnarray*}}
	\def\ees{\end{eqnarray*}}
\def\beal{\begin{aligned}}
	\def\eeal{\end{aligned}}

\def\sd{\overline{S_{D}^*}\bm{y}_{D}}
\def\sdv{\overline{S_{D_v}^*}\bm{y}_{D_v}}
\def\sdvk{\overline{S_{D_{v_k}}^*}\bm{y}_{D_{v_k}}}

\def\muv{\big[\bm{M}\big]_{uv}}

\def\pd{\huaP_{D,\la}}
\def\pdvi{\huaP_{D_v,\la_1}}
\def\pdvii{\huaP_{D_v,\la_2}}
\def\pdviii{\huaP_{D_v,\la_3}}
\def\qd{\huaQ_{D,\la}}
\def\qdi{\huaQ_{D,\la_1}}

\usepackage{booktabs}
\usepackage{diagbox} 
\usepackage{multirow} 
\usepackage{makecell} 
\usepackage{longtable} 
\usepackage{array}
\usepackage{tabularx}
\usepackage{caption}	
\usepackage{graphicx}
\graphicspath{ {./images/} }
\usepackage{subfigure} 
\begin{document}
	\title{Theory of Decentralized Robust Kernel-Based Learning
		}\vspace{2mm}
	\date{}
	

\author{Zhan Yu\\ \small Department of Mathematics,   Hong Kong Baptist University\\ \small Waterloo Road, Kowloon, Hong Kong\\ \small   Email: mathyuzhan@gmail.com  
	\and
	Zhongjie Shi\\ \small School of Computing and Data Science,   University of Hong Kong\\ \small Pok Fu Lam Road, Hong Kong\\ \small   Email: zhongjshi2@gmail.com
	\and
	Ding-Xuan Zhou\\ \small School of Mathematics and Statistics, University of Sydney \\ \small Sydney NSW 2006, Australia\\
	\small   Email: dingxuan.zhou@sydney.edu.au
}

	\maketitle
	
	\begin{abstract}
We propose a new decentralized robust kernel-based learning algorithm within the framework of reproducing kernel Hilbert spaces (RKHSs) by utilizing a networked system that can be represented as a connected graph. The robust loss function $\huaL_\sigma$ induced by a windowing function $W$ and a robustness scaling parameter $\sigma>0$ can encompass a broad spectrum of robust losses. Consequently, the proposed algorithm effectively provides a unified decentralized learning framework for robust regression, which  fundamentally differs from the existing distributed robust kernel-based learning schemes, all of which are  divide-and-conquer based. We rigorously establish a learning theory and offer comprehensive convergence analysis for the algorithm. We show each local robust estimator generated from the decentralized algorithm can be utilized to approximate the regression function. Based on kernel-based integral operator techniques, we derive general high confidence convergence bounds for the local approximating sequence in terms of the mean square distance, RKHS norm, and generalization error, respectively. Moreover,  we provide rigorous selection rules for local sample size and show that, under properly selected step size and scaling parameter $\sigma$, the decentralized robust algorithm can achieve optimal  learning rates (up to logarithmic factors) in both norms. The parameter $\sigma$ is shown to be essential for enhancing robustness and ensuring favorable convergence behavior. The intrinsic connection among decentralization, sample selection, robustness of the algorithm, and its convergence is clearly reflected.
	\end{abstract}
	
	
	\section{Introduction}\vspace{0.000000000000000000000001mm}

In the past two decades, distributed computing, distributed optimization and distributed learning theory have experienced remarkable advancements to tackle the challenges posed by big data in the information era. These developments have catalyzed numerous beneficial and revolutionary transformations across fields such as machine learning \cite{gfd2006}, systems science  \cite{xb2004}, \cite{xbk2007}, computational mathematics \cite{lwz2021}, optimization theory \cite{daw2011}, \cite{no2009}, \cite{rnv2010}, \cite{yu2023} and data mining \cite{wzwd2013}. Instead of processing the entire training dataset in a single machine model, the distributed learning scheme facilitates significant computational efficiency by dividing the dataset into local subsets, allowing different machines or agents to handle them independently and parallel \cite{zdw2015}. As a result, distributed learning is a viable solution for overcoming big data challenges and meanwhile enhancing privacy protection. Practical realization of distributed learning has been witnessed in a variety of real-world domains such as financial markets, medical systems, sensor network and social activity mining.

  In this work, we primarily focus on developing distributed learning schemes within the literature of robust kernel-based regression, which has become increasingly crucial in machine learning, statistics and inverse problems in recent years \cite{fhsys2015},  \cite{ghw2020}, \cite{ghs2018}, \cite{hg2024}, \cite{hwz2020}, \cite{hwz2021}, \cite{hr2011}, \cite{lpp2007}, \cite{wf2024}, \cite{wh2019a}, \cite{yhsz2021}, \cite{zff2024}. In the past two decades, kernel-based regression has been widely studied in the literature of learning theory  \cite{cv2007}, \cite{cz2007}, \cite{ghw2020}, \cite{gcs2024}, \cite{ghs2018}, \cite{hg2024}, \cite{hwz2020}, \cite{lgz2017}, \cite{lz2018}, \cite{sz2007},  \cite{sc2008}, \cite{tong2021}, \cite{wf2024}, \cite{wyz2007},  \cite{yrc2007}, \cite{yp2008}, \cite{yfsz2024}, \cite{yhsz2021},
   \cite{zhou2003}. Let $\rho$ be a Borel probability measure  defined on $\huaX\times\huaY$, where $\huaX$ is a compact metric space (input space) and $\huaY\subset\mbb R$ (output space). Let the sample set $D=\{(x_i,y_i)\}_{i=1}^{|D|}\subset \huaX\times\huaY$ be independently drawn according to $\rho$. Our main objective is the regression function defined by
  \bea
  f_\rho(x)=\int_\huaY yd\rho(y|x), \ x\in\huaX,
  \eea
  where $\rho(\cdot|x)$ is the conditional probability distribution at $x$ induced by $\rho$. In this paper, we  consider utilizing a robust loss function \bea
  \huaL_\sigma(u)=W\left(\f{u^2}{\sigma^2}\right)   \label{lossf}
  \eea
   to approximate the target regression function $f_\rho$. Here, for any $x>0$, the windowing function $W:\mbb R_+\rightarrow\mbb R$  satisfies 
   \bea
   W'(x)>0 \ \text{for} \ x>0, \ W_+'(0)>0 \ \ \text{and} \ \sup_{x\in (0,+\infty)}|W'(x)|\leq C_W; \ \label{window_ass1}
   \eea
    additionally, there exists some $c_p>0$ with $p>0$ such that
     \bea
    |W'(x)-W_+'(0)|\leq c_p|x|^p  \label{window_ass2}
    \eea
     for all $x>0$. Here, $W_+'(0)$ denotes the right derivative of the function $W$ at $x=0$. It is important to note that the traditional least squares  regression scheme is one of the most widely used method in the literature. This approach relies solely on the mean squared error and falls under second-order statistics. While least squares regression is optimal for handling Gaussian noise, it becomes suboptimal in the presence of non-Gaussian noise. In practice, samples are frequently affected by non-Gaussian noise and outliers. Furthermore, least squares estimators in regression models are highly sensitive to outliers, and their performance tends to deteriorate when the noise deviates from Gaussian distributions. Compared with standard least squares loss functions, the robustness of the traditional learning schemes to non-Gaussian noise and outliers is fully enhanced after introducing robust loss functions \cite{hcp2017}. By choosing an appropriate windowing function 
  $W$ and robustness scaling parameter 
  $\sigma$, the loss function can generate a diverse array of significant robust loss function classes \cite{hr2011}, for example, the Cauchy loss $\huaL_\sigma(u)=\log(1+\f{u^2}{2\sigma^2})$ with $W(x)=\log(1+\f{x}{2})$; the Welsch  loss $\huaL_\sigma(u)=1-\exp(-\f{u^2}{2\sigma^2})$ with $W(x)=1-\exp(-\f{x}{2})$; the Fair loss: $\huaL_\sigma(u)=\f{|u|}{\sigma}-\log(1+\f{|u|}{\sigma})$, with $W(x)=\sqrt{x}-\log(1+\sqrt{x})$. It is also noteworthy that the robust loss $\huaL_\sigma$ in our setting can be non-convex, leading to more efficient robust estimators that can successfully overcome gross outliers while maintaining a prediction accuracy comparable to that of least squares loss (see e.g. \cite{fw2021}, \cite{ghw2020}, \cite{gcs2024}). Over the past two decades, robust learning algorithms induced by different types of robust loss functions $\huaL_\sigma$ have experienced significant growth and development \cite{hr2011}. The remarkable progress has been reflected in various research fields, including, for example, maximum correntropy criterion (MCC) based learning  \cite{fhsys2015}, \cite{hg2024}, \cite{lpp2007}, learning theory of minimum error entropy (MEE) \cite{ghw2020},  \cite{wh2019a}, support vector machines for  regression with robust loss \cite{cs2007}, \cite{sc2011}, robust learning for functional regression \cite{wf2024}, \cite{yhsz2021}, and deep neural network based robust learning \cite{zff2024}. In kernel-based robust learning, there are two common approaches to enhance the computational efficiency of robust algorithms for large-scale data. One approach is online learning, which requires only one or part of the training samples for updating in each step \cite{gcs2024}, \cite{wh2019a}. The other is the divide-and-conquer (DAC) based distributed learning, which either decomposes a given data set as needed or accommodates scenarios where the data set naturally appears in a distributed manner \cite{ghw2020}, \cite{ghs2018}, \cite{hg2024}, \cite{hwz2020}, \cite{hwz2021}.  In this paper, we mainly go along the line of the second approach and  aim at improving the existing distributed schemes for robust learning by introducing a decentralized  robust kernel-based learning scheme and rigorously establishing theoretical foundations.

 Before introducing our main  algorithm, we provide an overview of the related work in the literature of kernel-based distributed learning. In the realm of distributed learning, various algorithms have been developed to tackle the challenges posed by large-scale data. Among these, kernel-based distributed learning methods, particularly those falling under DAC category, have emerged as particularly influential, for example, the regularized least squares DAC algorithms   \cite{lgz2017}, \cite{sl2022},  \cite{zdw2015}, the DAC (stochastic) gradient descent algorithms \cite{hwz2020}, \cite{lz2018}, \cite{nkm2024}, the DAC spectral algorithms \cite{glz2017}, the DAC interpolation \cite{lwz2021}, and the DAC robust regression algorithms \cite{ghw2020}, the DAC regularized functional linear regression algorithms \cite{ls2022}, the DAC gradient descent for functional linear regression \cite{yfsz2024}.  On the other hand, another approach to developing distributed learning algorithms is known as decentralization. In kernel-based learning, several decentralized schemes have been proposed recently. Existing well-known schemes include, for example, decentralized Nystr\"om approximation based kernel gradient descent \cite{ll2023}, the consensus-based decentralized kernel SGD in RKHS \cite{kprr2018}, decentralized random feature based kernel gradient descent \cite{rrr2020}, and decentralized 
 communication-censored ADMM-based approach for kernel learning \cite{xwct2021}. However, compared to the rigorous development of DAC-based kernel learning schemes, the development of decentralized approaches in the realm of kernel-based learning theory has only begun recently, and the research in this direction is still far from maturity and deserves further development.
 
  For the distributed kernel-based learning algorithms mentioned above, a comprehensive theoretical foundation regarding learning rates and convergence bounds has been gradually established over the past decade. Notably,  distributed learning schemes have been developed for robust learning algorithms  within the DAC framework \cite{zdw2015}. These existing DAC approaches can be categorized as either Tikhonov regularization-based or gradient descent-based DAC robust learning methods, and they primarily consist of three key steps: one first partitions the training data set $D=\{(x_i,y_i)\}_{i=1}^{|D|}$ drawn from an unknown probability distribution $\rho$ into $m$ disjoint subsets $\{D_v\}_{v=1}^m$, namely, $D=\bigcup_{v=1}^mD_v$ with $D_u\cap D_v=\emptyset$, $u\neq v$. Meanwhile, each subset $D_v$ of the training sample is sent to an individual local machine $v$. In each local machine, based on each data subset $D_v$, the local machine performs a robust learning algorithm by utilizing aforementioned robust loss and obtains some local estimators. In what follows,  these local estimators are communicated to a central master/processor by taking some weighted averaging summation. In the existing literature of robust learning, the DAC approaches mainly include two categories,
the first is the regularized DAC-based robust algorithm  which performs the Tikhonov-regularized robust algorithm with some regularization parameter $\lambda>0$ and robustness scaling parameter $\sigma>0$, and obtains some local estimators $\{f_{D_v,\lambda}^{\sigma}\}_{v=1}^m$, and the central server performs the weighted average $\overline{f}_{D,\lambda}^\sigma=\sum_{v=1}^m \frac{|D_v|}{|D|}f_{D_v,\lambda}^\sigma$ (see e.g. \cite{ghw2020}, \cite{hg2024}). Another approach is the popular DAC-based distributed gradient descent robust learning approach  (see e.g. \cite{ghs2018}, \cite{hwz2020}). In step $t$, each local machine (processor) $v\in\mathcal{V}$ updates by producing a local estimator $f_{t,D_v}^\sigma$ based on robust kernel-based gradient descent, and the central server obtains a global estimator $\overline{f}_{t,D}^\sigma=\sum_{v=1}^m \frac{|D_v|}{|D|}f_{t,D_v}^\sigma$. The estimators $\overline{f}_{D,\lambda}^\sigma$ and $\overline{f}_{t,D}^\sigma$ mentioned above are two canonical DAC robust kernel learning estimators.

The preceding discussion highlights a structural limitation of DAC distributed robust learning algorithms: their dependence on \emph{a central master/server} for aggregating information from all local processors. During each update iteration, the central server must await the transmission of data from all local servers before proceeding with the updates, which substantially hampers computational efficiency, especially in scenarios involving a large number of local nodes. Furthermore, in contemporary computational environments characterized by multi-agent systems or multi-processor networks \cite{bhot2005}, \cite{daw2011}, \cite{no2014}, \cite{no2009}, \cite{rnv2010}, \cite{xbk2007},  the prevalence of node failures or transmission disruptions poses considerable challenges, as the DAC scheme necessitates the participation of all local nodes in the updating process. Therefore, it is imperative to explore the development of a decentralized robust learning framework. Notably, the decentralized robust kernel-based learning theory remains undeveloped, with no theoretical results established for kernel-based robust regression. This paper aims to fill this gap.
 To implement this idea, inspired by the decentralization mechanism from consensus-based distributed optimization, systems science and decentralized kernel learning   \cite{daw2011},  \cite{no2009}, \cite{rnv2010}, \cite{rrr2020},
we introduce a network modeled as a connected graph $(\huaV,\huaE)$ where $\huaV=\{1,2,...,m\}$ is the node set and  $\huaE=\{(i,j)|i,j\in\huaV,i\neq j\}$ is the edge set. Each vertex of the
graph is referred to as an agent \cite{daw2011}. In this paper, we address a scenario in which we need to handle a data set, and the data is either large-scale or naturally arrives in a distributed manner for privacy preserving consideration, making it impractical for a single processor to execute the robust kernel learning algorithm. Consequently, a distributed approach is necessary. Given a sample set $D$ satisfying the decomposition $D=\bigcup_{u=1}^mD_u$, $D_u\cap D_v=\emptyset$, $u\neq v$ with the total sample size $|D|=\sum_{u\in\huaV}|D_u|$, each agent $u\in\huaV$ possesses the collection of independent and identically distributed (i.i.d.) training sample $D_u=\{(x_i^u,y_i^u)\}_{i=1}^{|D_u|}$ drawn according to probability measure $\rho$. The edge $(i,j)\in\huaE$ indicates
that agent $i$ and agent $j$ can establish a bidirectional and information communication link each other. The communication weight matrix of the graph is denoted by an $m\times m$ matrix $\bm{M}$ with entries $[\bm{M}]_{ij}\geq0$, $i,j\in\huaV$ and satisfies that $[\bm{M}]_{ij}>0$ only if $(i,j)\in\huaE$. In a reproducing kernel Hilbert space (RKHS) $(\mathcal{H}_K,\|\cdot\|_K)$ induced by a Mercer kernel $K:\huaX\times\huaX\rightarrow\mbb R$, denote the function $K_x=K(x,\cdot)$ for $x\in\huaX$, then our decentralized kernel-based robust gradient descent   algorithm with the windowing function $W$ and robustness parameter $\sigma$ is defined by $f_{0,D_v}=0$, $v\in\huaV$ (initialization for each local node) and 
\bea
&&\phi_{t,D_v}={f}_{t, D_v}-\frac{\aaa}{|D_v|} \sum_{(x, y) \in D_v}W'\left(\frac{\xi_{t,D_v}^2\left(z\right)}{ \sigma^2}\right) \xi_{t,D_v}\left(z\right) K_{x}, \label{main_alg1}\\
&&{f}_{t+1,D_{u}}=\sum_{v}\big[\bm{M}\big]_{uv}\phi_{t,D_v},  \label{main_alg2}
\eea
where $\xi_{t,D_v}\left(z\right)={f}_{t, D_v}(x)-y, z=(x, y)$.
From our decentralized robust learning scheme \eqref{main_alg1}-\eqref{main_alg2}, each node 
$v\in\huaV$ is empowered to manage its own dataset $D_v$
and to execute a local robust gradient descent algorithm \eqref{main_alg1} utilizing random sample $D_v$. The proposed algorithm facilitates communication exclusively among neighboring nodes to update local estimators. To achieve this, we have utilized the communication matrix 
$\bm{M}$ in \eqref{main_alg2} to encapsulate the communication dynamics among local processors. This approach effectively eliminates the necessity for a central server to aggregate information from all local estimators at each iteration,  suffered by the previously discussed DAC-based robust learning estimators $\{\overline{f}_{D,\la}^\sigma\}$ and $\{\overline{f}_{t,D}^\sigma\}$ which are centralized. In fact, in our algorithm, each local sequence $\{f_{t,D_u}\}_{u\in\huaV}$ can serve as the approximating sequence for the regression function $f_\rho$, in contrast to previous DAC-based robust learning algorithms, where a central estimator  $\{\overline{f}_{D,\la}^\sigma\}$ or $\{\overline{f}_{t,D}^\sigma\}$ has to be utilized to realize the approximation of $f_\rho$. Moreover, the windowing function $W$ can be selected to be a variety of robust loss functions with the scaling parameter $\sigma$ that can be flexibly chosen, hence our main algorithm provides a novel unified decentralized robust learning framework for  kernel-based learning, improving existing  counterparts of distributed kernel-based algorithms involving \cite{ghw2020},  \cite{ghs2018}, \cite{hg2024}, \cite{hwz2020}, \cite{ll2023}, \cite{lz2018} in various  aspects.

In this work, we  investigate the learning capability of the decentralized robust kernel-based gradient descent learning algorithm \eqref{main_alg1}-\eqref{main_alg2}. We rigorously provide general capacity-dependent convergence bounds for the algorithm in both the mean square distance and RKHS distance. We establish explicit selection rules for the local sample size based on the spectral gap of the communication weight matrix $\bm M$ and the global sample size $|D|$. Under the selection rules, we demonstrate that, with appropriately selected robustness scaling parameter $\sigma$ and  stepsize $\aaa$, each of the local approximating sequence $\{f_{t,D_u}\}_{u\in\huaV}$ generated from the main algorithm is able to approximate $f_\rho$ in a satisfactory way in terms of the mean square distance,  RKHS norm and generalization error, all of which are optimal (up to logarithmic factors) in the minimax sense. This differs significantly from the approximation used in DAC-based distributed learning schemes which necessitate a central server to aggregate local estimates and form a central global sequence for realization of the approximation. Our main results reveal the clear gap relation between the decentralized robust estimator $\{f_{t,D_u}\}_{u\in\huaV}$ and the kernel-based gradient descent estimator for the centralized least squares regression in \cite{yrc2007} in a quantitative manner. The results also uncover the far-reaching relationship among the local sample size, the spectral gap, and the robustness scaling parameter, to ensure the convergence of the algorithm. Additionally, they highlight the intrinsic connections among decentralization, sample selection, robustness of the algorithm, and its convergence. Finally, due to the  generality of the windowing functions considered in this paper, the developed theoretical results can provide essential insights for the future developments of specific decentralized robust learning algorithms, such as decentralized MEE algorithm, decentralized MCC algorithm and other decentralized kernel-based information-theoretic learning algorithms.  \\

\noi \textbf{Notation} We use $\mathbb N_+$ to denote the set of positive integers. In calculations involving multiple indices, we always use notation $\sum_v$ to represent $\sum_{v\in\huaV}$ in this paper. Throughout this paper, we use index $v_0$ to refer to index $u$. For a square matrix $Q$ and $p\in\mbb N_+$, we use $Q^p$ to denote the matrix power. For $t$ real numbers $q_1$, $q_2$, ..., $q_t$, we use $\prod_{s=1}^tq_s$ to denote the product $q_1q_2\cdots q_t$. For $t$ $m\times m$ matrices $Q_1$, $Q_2$, ..., $Q_t$, we use $\prod_{s=1}^tQ_s$ to denote the matrix product $Q_1Q_2\cdots Q_t$. For two numbers $a,b\in\mbb R$, we use $a\vee b$ ($a\wedge b$) to denote the maximum (minimum) between $a$ and $b$. For two data-based functions $p$ and $q$ that may depend on $|D|$, $m$, $n$, $t$, $\bar t$, $\f{1}{1-\gamma_{\bm{M}}}$ defined in this paper,  we say $p\lesssim q$ if there exists an absolute constant $c>0$ independent of $|D|$, $m$, $n$, $t$, $\bar t$, $\f{1}{1-\gamma_{\bm{M}}}$ such that $p\leq cq$. For the sake of convenience in the proof,  we say $p\lesssim_{\delta} q$ if there exists an absolute constant $c$ independent of $|D|$, $m$, $n$, $t$, $\bar t$, $\f{1}{1-\gamma_{\bm{M}}}$  up to logarithmic factors which are independent of $\delta$ (the $\log$ factors here might involve $m, n, |D|, t, \bar t$) such that $p\leq cq$. We say $p\cong  q$ if $p\lesssim q$ and $q\lesssim p$.

\section{Main results and discussions}
In this sections, we present our main results. Before coming to the main results, we first introduce the background, fix some necessary notations and provide some standard assumptions.

\subsubsection*{Network topology and decentralization}
In this paper, we employ a  multi-agent network to construct a decentralized robust gradient descent algorithm. Within the realm of systems science (see e.g. \cite{bhot2005}, \cite{no2014}, \cite{no2009},  \cite{rnv2010}, \cite{xbk2007}, \cite{yhy2022}),  each local processor is commonly considered as a local agent 
$v\in\huaV$ in the multi-agent system, and all processors connected by appropriate  links collectively form a multi-agent network. We model this network as a connected graph $\huaG=(\huaV,\huaE,\bm{M})$. In $\huaG$, $\huaV$ represents the set of nodes indexed by $\huaV=\{1,2,...,m\}$, where $m$ denotes the total number of the local agents (machines). The set $\huaE\subset\huaV\times\huaV$ represents the set of edges of the graph $\huaG$. The matrix  $\bm{M}=([\bm{M}]_{uv})_{m\times m}$ is  a non-negative matrix representing the adjacency weights of edges, such that $[\bm{M}]_{uv}>0$ if $(v,u)\in\huaE$ and $[\bm{M}]_{uv}=0$ otherwise. Here $\big[\bm{M}\big]_{uv}$ denotes the element of matrix $\bm{M}$ of $u$-th row and $v$-th column. The matrix $\bm{M}$ is also referred to as the communication matrix of the multi-agent network. It  follows naturally that the edge set $\huaE$ can be expressed as $\huaE=\left\{(u,v)\in\huaV\times\huaV\Big|\big[\bm{M}\big]_{uv}>0\right\}$. We also define the neighbor set of agent $u\in\huaV$ as $\huaN_u=\left\{v\in\huaV\big|(v,u)\in\huaE\right\}$. Here, we assume that $u\in\huaN_u$ for all $u\in\huaV$. Throughout the paper, we assume the communication weight matrix $\bm{M}$ is doubly-stochastic, namely, 
\bea
\bm{M}\textbf{1}=\textbf{1} \  \text{and} \  \bm{M}^T\textbf{1}=\textbf{1},
\eea
 where $\textbf{1}$ denotes the $m$-dimension vector with all its components $1$. This double stochasticity assumption is widely adopted in the literature of distributed optimization and systems science (e.g. \cite{daw2011}, \cite{no2009}, \cite{yhy2022}). For convenience of analysis,
 we assume the absolute value of the second largest eigenvalue $\gamma_{\bm{M}}$ of the matrix $\bm{M}$  satisfies $0<\gamma_{\bm{M}}<1$.

We note that achieving this network model in a distributed scenario is relatively straightforward. For instance, when bidirectional communication between nodes is permitted, doubly stochasticity can be attained by enforcing symmetry on the node communication matrix. There are several standard choices for the communication weight matrix $\bm{M}$. One simple approach is to consider the equi-neighbor weights (see e.g. \cite{bhot2005}, \cite{no2009}): each agent assigns equal weight to its own information and to the information received from neighboring agents. Specifically, $[\bm{M}]_{uv}=1/(1+n_u)$ for each $u\in\huaV$, and those neighbors $v$ of $u$; otherwise, set $[\bm{M}]_{uv}=0$. Here, $n_u$ denotes the number of agents communicating with agent  $u$. Another weight assignment method that can be utilized is the least squares consensus weight rule \cite{xbk2007}. For more details on the construction of weight matrices in various contexts, we refer to references \cite{no2009}, \cite{xb2004}. 

Now, we can elucidate the mechanism behind the main algorithm defined by equations \eqref{main_alg1} and \eqref{main_alg2} in greater detail. In our algorithm, during the first sub-step \eqref{main_alg1},  each node $v\in\huaV$ updates its local estimate using a robust kernel-based gradient descent approach to derive an intermediate estimate $\phi_{t,D_v}$ based on data set $D_v$. Subsequently, in the second sub-step \eqref{main_alg2},  node $u\in\huaV$ receives estimate $\phi_{t,D_v}$ from all of its neighbors $v\in\huaN_u$.  It then computes a locally weighted summation of all the received estimates to obtain a local variable $f_{t+1,D_u}$, facilitated by introducing the communication weight matrix $\bm M$. This sub-step represents a typical network-based distributed computation. The weights for this summation consist of all non-zero elements in $\{[\bm{M}]_{u1},[\bm{M}]_{u2},...,[\bm{M}]_{u m}\}$. The local estimator $f_{t,D_u}$ is updated through communication between node $u$ and its neighbors in $\huaN_u$. It is worth noting that, in this work, each local estimator $f_{t,D_u}$, $u\in\huaV$,  can be utilized for the purpose of approximating regression function $f_\rho$. This approach essentially differs from DAC-based distributed algorithms, where a global weighted average is required to form the final global estimator.

\subsubsection*{Analyais framework and main results}

Decompose the Borel probability measure $\rho$ into a marginal distribution $\rho_\huaX$ on input space $\huaX$ and the conditional probability measure $\rho(\cdot|x)$ on output space $\huaY$ given $x$. Let $(L_{\rho_\huaX}^2,\|\cdot\|_{L_{\rho_\huaX}^2})$ be the Hilbert space of $\rho_\huaX$ square integrable functions on $\huaX$.  It is well-known that the reproducing property $f(x)=\nn f,K_x\mm_K$ holds for any $x\in\huaX$ and $f\in\huaH_K$. 
 As a result of the compactness of $\huaX$, the constant $\kkk=\sup_{x\in\huaX}\sqrt{K(x,x)}<\infty$. The reproducing property directly implies that $\|f\|_\infty\leq\kkk\|f\|_K$. Define the integral operator $L_{K}: L_{\rho_\huaX}^2 \rightarrow L_{\rho_\huaX}^2$ associated with the Mercer kernel $K$ by
$$
L_{K}({f})= \int_{\mathcal{X}}\left\langle{f}, K_{x}\right\rangle_{K} K_{x}  d \rho_{\mathcal{X}}, \quad {f} \in L_{\rho_\huaX}^2.
$$
We recall the isometry  relation  $L_K^{1/2}:L^2_{\rho_\huaX}\rightarrow\huaH_{ K}$, which indicates that $\|f\|_{L^2_{\rho_\huaX}}=\|L_{ K}^{1/2}f\|_{ K}$, $f\in L^2_{\rho_\huaX}$.  Throughout the paper, for the output variable, we assume the moment condition: there exist constants $B_\rho>0$ and $M_\rho>0$ such that 
\bea
\int_{\huaY}|y|^p\rho(y|x)\leq B_\rho p!M_\rho^p, \  \forall p\in\mbb N_+, \ x\in\huaX.   \label{moment_condition}
\eea
Condition \eqref{moment_condition} commonly referred to as the Bernstein condition, is frequently encountered in the literature on kernel-based learning theory e.g. \cite{ghw2020},  \cite{tong2021}, \cite{wh2019a}, \cite{Wellner2013}, \cite{yfsz2024}. This assumption establishes standard restrictions on the behavior of random variables. It is a natural extension of the uniform boundedness assumption of the output variable as seen in e.g. \cite{cz2007}, \cite{gcs2024}, \cite{hwz2020}, \cite{msz2023}. Types of noise that satisfy \eqref{moment_condition} include well-known categories commonly observed in practice, such as Gaussian noise, sub-Gaussian noise, the noise with compactly supported distributions, and noise associated with certain exponential distributions.

To measure the capacity of the underlying space $\huaH_K$, we use the well-known effective dimension defined by
\bea
\huaN(\la)=\text{Tr}\left[L_{ K}(\la I+L_{ K})^{-1}\right],
\eea
where $\text{Tr}$ is used to denote the trace of the operator (see e.g. \cite{ghw2020}, \cite{glz2017}, \cite{hwz2020}, \cite{lgz2017}, \cite{lz2018}, \cite{yhsz2021}).
 We assume that there exist some $0<s\leq1$ and a constant $C_0>0$ such that the effective dimension $\huaN(\la)$ satisfies
\bea
\huaN(\la)\leq C_0\la^{-s}, \ \forall \la>0. \label{capacity_ass}
\eea
The following assumption on the regularity of the target function $ f_\rho$ is also assumed:
\bea
 f_{\rho}=L_{ K}^rg_\rho, \ \text{for} \ \text{some} \ r>0 \ \text{and} \ g_\rho\in L_{\rho_\huaX}^2. \label{regularity_ass}
\eea
  This standard regularity condition has been widely considered in the literature of learning theory (see e.g. \cite{ghw2020}, \cite{gcs2024}, \cite{ghs2018},  \cite{hwz2020}, \cite{ll2023}, \cite{lgz2017}, \cite{lz2018}, \cite{yp2008}, \cite{yhsz2021}).

Before coming to state our main results,  for the data set $D$, we recall the following classical kernel-based gradient descent sequence $\{\wh f_{t,D}\}$ defined in \cite{lz2018}, \cite{yrc2007} with stepsize $\aaa W_+'(0)$ which is    defined by, $\wh f_{0,D}=0$ and 
\bea
\wh{f}_{t+1, D}= \wh{f}_{t, D}-\frac{\aaa  W_+'(0)}{|D|} \sum_{(x, y)\in D} \left( \wh{f}_{t, D}(x)-y\right)K_{x}.  \label{classical_kernelGD}
\eea
 Our first main result pertains to the convergence in the mean square distance, which establishes capacity-dependent high probability upper bounds. It reveals the clear gap between the decentralized local sequence $\{f_{t,D_u}\}_{u\in\huaV}$  generated from the decentralized robust kernel-based algorithm \eqref{main_alg1}-\eqref{main_alg2} and the centralized sequence $\{\wh f_{t,D}\}$ generated from the classical centralized kernel-based gradient descent \eqref{classical_kernelGD} for the least squares regression.
\begin{thm}\label{mainthm_errorbdd}
	Assume \eqref{moment_condition}, \eqref{capacity_ass} with $0<s\leq1$, \eqref{regularity_ass}  with $r>\f{1}{2}$, the stepsize $\aaa$ satisfies $0<\aaa\leq\f{1}{\kkk^2}\min\{\f{1}{W_+'(0)},\f{1}{C_W}\}$, the windowing function $W$ satisfies basic conditions \eqref{window_ass1} and \eqref{window_ass2}. If $|D_u|=\f{|D|}{m}=n$, $u\in\huaV$, then, for  $t,\bar t\in\mbb N_+$ with $t\geq2\bar t\geq4$, for any $0<\delta<1$, we have, for any $u\in\huaV$, any $0<\delta<1$, with confidence at least $1-\delta$, 
	\bes
	\beal
	\left\|f_{t,D_u}-\wh f_{t,D}\right\|_{L_{\rho_\huaX}^2}\lesssim_\delta&\left(\log\f{256}{\delta}\right)^{4\vee (2p+2)}\Bigg[\aaa^{\f{1}{2}}\left(\f{1}{1-\gamma_{\bm{M}}}\right)\left(\f{\sqrt{m}}{\sqrt{n}}\right)+\aaa^{\f{3}{2}}\bar t^{\f{3}{2}}\f{1}{n}+\aaa^{\f{3}{2}} \bar t^{\f{1}{2}}t\f{1}{n}\\
	&+\aaa t\left(\sqrt{m}\gamma_{\bm{M}}^{\bar t}\wedge 1\right)\f{1}{\sqrt{n}}+(\aaa\bar t\vee 1)^{\f{1}{2}}\left[(\aaa\bar t\vee 1)^{2}+\aaa t\sqrt{m}\gamma_{\bm{M}}^{\bar t}\right]\aaa t\f{1}{\sqrt{n}}\f{1}{\sqrt{|D|}}\\
	&+\left(\f{(\aaa t)^{\f{s}{2}}}{\sqrt{n}}+\f{(\aaa t)^{\f{1}{2}}}{n}\right)\f{1}{\sqrt{n}}\f{1}{\sqrt{|D|}}\aaa t\left(\aaa t\sqrt{m}\gamma_{\bm{M}}^{\bar t}+\aaa\bar t\right)\\
	&+\aaa^{\f{1}{2}}\left(t^{p+1}\sigma^{-2p}+\f{1}{\sqrt{n}}t^{p+2}\sigma^{-2p}\right)\Bigg].
	\eeal
	\ees
\end{thm}
The above result provides a general high probability mean square distance gap in terms of all crucial quantities associated with the main algorithm \eqref{main_alg1}-\eqref{main_alg2}. The next main result indicates that, under slightly milder conditions, when the local sample size satisfies a benchmark condition \eqref{n_condition_L2_norm} described by  the global sample size $|D|$, and $\bar t\cong\f{1}{1-\gamma_{\bm{M}}}$, the proposed decentralized robust kernel-based learning algorithm can achieve tighter high-probability upper bounds for the mean square distance $\|f_{t,D_u}-f_\rho\|_{L_{\rho_\huaX}^2}$ between $\{f_{t,D_u}\}$ and the target regression function $f_\rho$. This finding underscores the efficacy of the algorithm in handling varying sample sizes while maintaining robust performance across decentralized settings. Moerover, the next result reveals that, when the robustness scaling parameter $\sigma$ satisfies a mild condition \eqref{sigma_condition_L2norm}, the proposed decentralized robust kernel-based gradient descent algorithm is able to achieve the optimal learning rates (convergence rates) of $\huaO(|D|^{-\f{r}{2r+s}})$ in $L_{\rho_\huaX}^2$ norm (up to logarithmic terms). 
\begin{thm}\label{mainthm_bdd_with_sigma}
	Assume \eqref{moment_condition}, \eqref{capacity_ass} with $0<s\leq1$, \eqref{regularity_ass} holds with $r>\f{1}{2}$ and $r+s>1$, the stepsize $\aaa$ satisfies $0<\aaa\leq\f{1}{\kkk^2}\min\{\f{1}{W_+'(0)},\f{1}{C_W}\}$ with $\aaa\cong1$, the windowing function $W$ satisfies basic conditions \eqref{window_ass1} and \eqref{window_ass2}. When $|D_u|=\f{|D|}{m}=n$, $u\in\huaV$, then, if  $t=|D|^{\f{1}{2r+s}}$, $\bar t=\f{2(2r+s+1)}{(2r+s)(1-\gamma_{\bm{M}})}\log(|D|)$ with $|D|^{\f{1}{2r+s}}\geq\f{4(2r+s+1)}{(2r+s)(1-\gamma_{\bm{M}})}\log(|D|)$  and the local sample size $n$ satisfy
	\bea
	n\geq\bar t|D|^{\f{2r+\f{s}{2}}{2r+s}}\vee\bar t^{\f{3}{2}}|D|^{\f{r}{2r+s}}\vee\bar t^5|D|^{\f{2-s}{2r+s}},
	\label{n_condition_L2_norm}
	\eea
	we have,  for any  $0<\delta<1$, with probability at least $1-\delta$, 
	\bes
	\beal
	\left\|f_{t,D_u}-f_\rho\right\|_{L_{\rho_\huaX}^2}\lesssim_\delta&\left(\log\f{512}{\delta}\right)^{4\vee (2p+2)}\Bigg[|D|^{-\f{r}{2r+s}}+\left(|D|^{\f{p+1}{2r+s}}\sigma^{-2p}+\f{1}{\sqrt{n}}|D|^{\f{p+2}{2r+s}}\sigma^{-2p}\right)\Bigg], u\in\huaV.
	\eeal
	\ees
Moreover, when the robustness scaling parameter $\sigma>0$ satisfies
	\bea
	\sigma\geq|D|^{\f{p+r+1}{2p(2r+s)}}\vee\f{|D|^{\f{p+r+2}{2p(2r+s)}}}{n^{\f{1}{4p}}},  \label{sigma_condition_L2norm}
	\eea
	we have,  for any $0<\delta<1$, with probability at least $1-\delta$,
	\bes
	\left\|f_{t,D_u}-f_\rho\right\|_{L_{\rho_\huaX}^2}\lesssim_\delta\left(\log\f{512}{\delta}\right)^{4\vee (2p+2)}|D|^{-\f{r}{2r+s}}, \ u\in\huaV.
	\ees
\end{thm}
It is noteworthy that, in Theorem \ref{mainthm_errorbdd}, the inverse dependence of this $L_{\rho_\huaX}^2$ gap on the spectral gap $1-\gamma_{\bm{M}}$ of the communication matrix $\bm{M}$ is  reflected in the convergence bound.  The spectral gap $1-\gamma_{\bm{M}}$ is closely related to the network topologies and have the scaling relation $\f{1}{1-\gamma_{\bm M}}=\huaO(m^\xi)$ ($\xi\geq0$) with $\xi=0$ for a bounded degree expander, $\xi=1$ for a two-dimensional grid, $\xi=2$ for a single cycle graph (see e.g. \cite{daw2011}, \cite{rrr2020}). Accordingly, the benchmark condition for $n$ can be improved to be 
\bea
n\geq|D|^{\f{\xi+\f{2r+\f{s}{2}}{2r+s}}{\xi+1}}\vee|D|^{\f{\f{3}{2}\xi+\f{r}{2r+s}}{\f{3}{2}\xi+1}}\vee|D|^{\f{5\xi+\f{2-s}{2r+s}}{5\xi+1}}  \label{n_condtion_xi}
\eea
for these well-known network topologies.  
 In this position, let us recall the well-known definition of the generalization error for a function $f:\huaX\rightarrow\huaY$  defined as 
\bes
\huaR(f)=\int_{\huaX\times\huaY}(f(x)-y)^2d\rho(x,y).
\ees
Based on the above results and related analysis, we are able to provide the following main result regarding the generalization error  $\huaR(f_{t,D_u})-\huaR(f_\rho)$, $u\in\huaV$.
\begin{thm}\label{mainthm_prediction_error}
	Under assumptions of Theorem \ref{mainthm_bdd_with_sigma}, when $|D_u|=\f{|D|}{m}=n$, $u\in\huaV$, then, if  $t=|D|^{\f{1}{2r+s}}$, $\bar t=\f{2(2r+s+1)}{(2r+s)(1-\gamma_{\bm{M}})}\log(|D|)$ with $|D|^{\f{1}{2r+s}}\geq\f{4(2r+s+1)}{(2r+s)(1-\gamma_{\bm{M}})}\log(|D|)$  and the local sample size $n$ satisfy \eqref{n_condition_L2_norm}, then we have, for any  $0<\delta<1$, with confidence at least $1-\delta$,
	\bes
	\beal
	\huaR(f_{t,D_u})-\huaR(f_\rho)\lesssim_\delta&\left(\log\f{512}{\delta}\right)^{8\vee (4p+4)}\Bigg[|D|^{-\f{2r}{2r+s}}+\left(|D|^{\f{2p+2}{2r+s}}\sigma^{-4p}+\f{1}{n}|D|^{\f{2p+4}{2r+s}}\sigma^{-4p}\right)\Bigg], u\in\huaV.
	\eeal
	\ees
 Moreover, if the robustness parameter $\sigma$ satisfies  \eqref{sigma_condition_L2norm}, then we have,  for any $0<\delta<1$, with probability at least $1-\delta$,
	\bea
	\huaR(f_{t,D_u})-\huaR(f_\rho)\lesssim_\delta \left(\log\f{512}{\delta}\right)^{8\vee (4p+4)}|D|^{-\f{2r}{2r+s}}, \ u\in\huaV.
	\eea
	\end{thm}
In the upcoming results, we will focus on the approximation in the RKHS norm.
 The next main results provide a general convergence  bound for the gap between the decentralized robust estimator $\{f_{t,D_u}\}_{u\in\huaV}$,  and the classical gradient estimator $\{\wh f_{t,D}\}$ in $\huaH_K$.

\begin{thm}\label{mainthm_K_norm}
	Assume \eqref{moment_condition}, \eqref{capacity_ass} with $0<s\leq1$, \eqref{regularity_ass}  with $r>\f{1}{2}$, the stepsize $\aaa$ satisfies $0<\aaa\leq\f{1}{\kkk^2}\min\{\f{1}{W_+'(0)},\f{1}{C_W}\}$ and $\aaa\cong1$, the windowing function $W$ satisfies basic conditions \eqref{window_ass1} and \eqref{window_ass2}. If  $|D_u|=\f{|D|}{m}=n$, $u\in\huaV$, then, for $t,\bar t\in\mbb N_+$ with $t\geq2\bar t\geq4$, and $\bar t=\f{2\log(|D|t)}{1-\gamma_{\bm{M}}}$,
	we have,  for any $0<\delta<1$, with probability at least $1-\delta$, there holds, 
	\bes
	\beal
	\left\|f_{t,D_u}-\wh f_{t,D}\right\|_K\lesssim_\delta&\left(\log\f{512}{\delta}\right)^{4\vee (2p+2)}\Bigg[\bar t\left(\f{\sqrt{m}}{\sqrt{n}}\right)+\bar t^{2}\f{1}{n}+\bar tt\f{1}{n}+\f{1}{\sqrt{n}}+\bar t^3t\f{1}{\sqrt{n}}\f{1}{\sqrt{|D|}}\\
	&+\f{\bar tt^{\f{s+3}{2}}}{n|D|^{\f{1}{2}}}+\f{\bar tt^2}{n^{\f{3}{2}}|D|^{\f{1}{2}}}+\left(t^{p+\f{3}{2}}\sigma^{-2p}+\f{1}{\sqrt{n}}t^{p+\f{5}{2}}\sigma^{-2p}\right)\Bigg], \ u\in\huaV.
	\eeal
	\ees
\end{thm}
Corresponding to Theorem \ref{mainthm_bdd_with_sigma}, the next main result establishes a crucial high-probability convergence bound for the decentralized robust estimator $\{f_{t,D_u}\}$ when approximating the target function $f_\rho$ in $\huaH_K$.
We remark that, the convergence in $\huaH_K$ itself holds significant importance. As mentioned in \cite{gcs2024}, \cite{sz2007} and \cite{zhou2003}, if $K\in C^{2n}(\huaX\times\huaX)$, then the convergence in $\huaH_K$ implies convergence in $C^n(\huaX)$ with $\|f\|_{C^n(\huaX)}=\sup_{|s|\leq n}\|D^sf\|_\infty$. Therefore, convergence in 
$\huaH_K$ is relatively stronger, ensuring the meaningfulness of the approximation in RKHS, and the estimator $\{f_{t,D_u}\}_{u\in\huaV}$  can not only approximate the regression function itself but also its derivatives, providing much flexibility for the algorithm in more application domains. The next result establishes the benchmark conditions for the local sample size $n$ to ensure the optimal minimax learning rates in RKHS norm, and also presents an effective selection rule \eqref{sigma_K_rule} for the robustness scaling parameter $\sigma$, ensuring that the main algorithm attains optimal learning rates in the RKHS norm.
\begin{thm}\label{mainthm_K_with_sigma}
Under assumptions of Theorem \ref{mainthm_K_norm}, when $|D_u|=\f{|D|}{m}=n$, $u\in\huaV$, then, if  $t=|D|^{\f{1}{2r+s}}$, $\bar t=\f{2(2r+s+1)}{(2r+s)(1-\gamma_{\bm{M}})}\log(|D|)$ with $|D|^{\f{1}{2r+s}}\geq\f{4(2r+s+1)}{(2r+s)(1-\gamma_{\bm{M}})}\log(|D|)$  and the local sample size $n$ satisfy 
\bea
n\geq\bar t|D|^{\f{2r+\f{s}{2}-\f{1}{2}}{2r+s}}\vee\bar t^2|D|^{\f{r-\f{1}{2}}{2r+s}}\vee\bar t|D|^{\f{r+\f{1}{2}}{2r+s}}\vee\bar t^6|D|^{\f{1-s}{2r+s}}, \label{n_condition_for__Knorm}
\eea
we have,  for any $0<\delta<1$, with probability at least $1-\delta$, 
	\bes
\beal
\left\|f_{t,D_u}- f_{\rho}\right\|_K\lesssim_\delta&\left(\log\f{512}{\delta}\right)^{4\vee (2p+2)}\Bigg[|D|^{-\f{r-\f{1}{2}}{2r+s}}+\left(|D|^{\f{p+\f{3}{2}}{2r+s}}\sigma^{-2p}+\f{1}{\sqrt{n}}|D|^{\f{p+\f{5}{2}}{2r+s}}\sigma^{-2p}\right)\Bigg], u\in\huaV.
\eeal
\ees
Moreover, if the robustness scaling parameter $\sigma>0$ satisfies
\bea
\sigma\geq|D|^{\f{p+r+1}{2p(2r+s)}}\vee\f{|D|^{\f{p+r+2}{2p(2r+s)}}}{n^{\f{1}{4p}}},  \label{sigma_K_rule}
\eea
we have,  for any $0<\delta<1$, with probability at least $1-\delta$,
\bes
\left\|f_{t,D_u}-f_\rho\right\|_K\lesssim_\delta\left(\log\f{512}{\delta}\right)^{4\vee (2p+2)}|D|^{-\f{r-\f{1}{2}}{2r+s}}, \ u\in\huaV.
\ees
	\end{thm}

Based on the benchmark condition \eqref{n_condition_for__Knorm} on local sample size $n$ for approximation in RKHS norm, we can also derive the following condition 
\bea
n\geq|D|^{\f{\xi+\f{2r+\f{s}{2}-\f{1}{2}}{2r+s}}{\xi+1}}\vee|D|^{\f{2\xi+\f{r-\f{1}{2}}{2r+s}}{2\xi+1}}\vee|D|^{\f{\xi+\f{r+\f{1}{2}}{2r+s}}{\xi+1}}\vee|D|^{\f{6\xi+\f{1-s}{2r+s}}{6\xi+1}}
\eea
for  bounded degree expander ($\xi=0$),   two-dimensional grid ($\xi=1$),  single cycle graph ($\xi=2$). It is noteworthy that in the RKHS norm estimates presented in Theorems \ref{mainthm_K_norm}-\ref{mainthm_K_with_sigma}, the condition $r+s>1$ is no longer necessary to ensure the  high probability convergence bound results.  This change reflects a broader set of regularity index 
$r$  and capacity index $s$ for Theorems \ref{mainthm_K_norm}-\ref{mainthm_K_with_sigma} to hold in the RKHS norm compared to Theorems \ref{mainthm_errorbdd}-\ref{mainthm_prediction_error}. Additionally, it is important to highlight that to establish a tight convergence bound for $\{f_{t,D_u}\}$ in terms of the $L_{\rho_\huaX}^2$ norm in Theorem \ref{mainthm_bdd_with_sigma} and the RHKS norm in Theorem \ref{mainthm_K_with_sigma}, there is a clear distinction between the benchmark conditions for the local sample size $n$. Specifically, this is illustrated by \eqref{n_condition_L2_norm} from Theorem \ref{mainthm_bdd_with_sigma} and \eqref{n_condition_for__Knorm} from Theorem \ref{mainthm_K_with_sigma}. It is intriguing to observe that, in the setting of approximation in $\huaH_K$, to ensure the optimal minimax convergence rate in $\huaH_K$, \eqref{n_condition_for__Knorm} requires a larger order of $\bar t\cong\f{1}{1-\gamma_{\bm{M}}}$ as well as a smaller order of $|D|$, compared to \eqref{n_condition_L2_norm} for $L_{\rho_\huaX}^2$ approximation. Other deep intrinsic trade-offs on the requirement between the network-based spectral gap $1-\gamma_{\bm M}$ and the global sample size $|D|$ deserve to be further explored. It is also interesting to observe that, in Theorem \ref{mainthm_bdd_with_sigma} and Theorem \ref{mainthm_K_with_sigma}, as discussed above, in order to realize optimal minimax learning rates for the algorithm in terms of $L_{\rho_\huaX}^2$ and RKHS norm, the selections of the robustness scaling parameter $\sigma$ depend intrinsically on the spectral gap of the communication matrix $\bm M$ and hence also on the network topologies. This fact reflects a profound intrinsic relationship between the robustness parameter selections and network topologies, grounded in the assurance of optimal learning rates. Throughout main results of this paper, we have demonstrated the crucial status of the robustness scaling parameter $\sigma$ for enhancing robustness while ensuring favorable convergence behavior of our decentralized robust algorithm. From multiple different perspectives, the results also extend the recently emerging theory of decentralized kernel learning providing theoretical assurance for the algorithm to handle tough noise environment with outliers, non-Gaussian noise in an effective decentralized manner. It is also easy to observe that, the windowing function $W$ in this work can be selected as many aforementioned crucial losses in modern robust learning. Hence, these results provide essential insights for future possible developments of some specific decentralized robust kernel-based learning algorithms, such as decentralized MEE algorithms and decentralized MCC algorithms.


\section{Key decomposition and basic lemmas} 
This section is dedicated to presenting the core error decomposition and introducing some essential foundational lemmas. Given a data set $D=\{(x_i,y_i)\}_{i=1}^{|D|}\subset\huaX\times\huaY$, here and in the following, $|D|$ denotes the cardinality of the set $D$ and $D(x):=\left\{x_i\right\}_{i=1}^{|D|}=\{x$ : there exists some $y$ such that $(x, y) \in D\}$.  For any $f\in\huaH_K$, define the sampling operator $S_D:\huaH_K\rightarrow\mbb R^{|D|}$ by $S_Df=(f(x_i))_{i=1}^{|D|}$. For a vector $\bm{y}_D=(y_i)_{i=1}^{|D|}\in\mbb R^{|D|}$, let $S_D^*:\mbb R^{|D|}\rightarrow\huaH_K$ be the adjoint operator of $S_D$ and it is given by $S_D^*\bm{y_D}=\sum_{i=1}^{|D|}y_iK_{x_i}$. We use $\overline {S_{D}^*}:\mbb R^{|D|}\rightarrow\huaH_K$ to denote the scaled operator of $S_{D}^*$ such that $\overline {S_{D}^*}\bm{y_D}=\f{1}{|D|}S_{D}^*\bm{y_D}$.
We also define the empirical operator $L_{K, D}$ on $\mathcal{H}_{K}$ as
$$
	L_{K, D}({f})  =\frac{1}{|D|} \sum_{i=1}^{|D|} \left\langle{f}, K_{x_i}\right\rangle_{K} K_{x_i} 
	=\frac{1}{|D|} \sum_{x \in D(x)}\left\langle{f}, K_{x}\right\rangle_{K} K_{x}, \quad {f} \in \mathcal{H}_{K}.
$$
According to the above notations and reproducing property, we know $L_{K,D}$ can be briefly written as $L_{K,D}=\overline {S_{D}^*}S_{D}$. 

For each local processor $v\in\huaV$, if we use $I$ to denote the identity operator, according to the definition of the operator $L_{{K},D_v}$ and the function $\xi_{t,D_v}$, \eqref{main_alg1} of our main algorithm can be represented by 
\bes
\begin{aligned}
\phi_{t,D_v}=&{f}_{t, D_v}-\frac{\aaa}{|D_v|} \sum_{(x, y) \in D_v}W'\left(\frac{\xi_{t,D_v}^2\left(z\right)}{ \sigma^2}\right) \xi_{t,D_v}\left(z\right) K_{x}\\
=&(I-\aaa  W_+'(0)L_{{K},D_v}){f}_{t,D_v}+\f{\aaa W_+'(0) }{|D_v|}\sum_{(x,y)\in D_v}y K_{x}+\aaa E_{t,D_v}\\
=&(I-\aaa W_+'(0)  L_{{K},D_v}){f}_{t,D_v}+\aaa W_+'(0) \overline{S_{D_v}^*}\bm{y}_{D_v}+\aaa E_{t,D_v},
\end{aligned}
	\ees
	where 
	\bea
E_{t, D_v}=-\frac{1}{|D_v|} \sum_{(x,y)\in D_v} \left[W'\left(\frac{\xi_{t,D_v}^2\left(z\right)}{\sigma^2}\right)-W_+'(0)\right]\left({f}_{t, D_v}\left(x\right)-y\right) K_x. \label{Etv_def}
	\eea
Then, we can change the main algorithm \eqref{main_alg1}-\eqref{main_alg2} into a compact form
\be
{f}_{t+1,D_{u}}=\sum_{v}\big[\bm{M}\big]_{uv}\Big[(I-\aaa W_+'(0)  L_{{K},D_v}){f}_{t,D_v}+\aaa W_+'(0) \overline{S_{D_v}^*}\bm{y}_{D_v}+\aaa E_{t,D_v}\Big].  \label{dec_robust_compact}
\ee 
For the data set $D$, we recall the definition of the sequence $\{\wh f_{t,D}\}$ in \eqref{classical_kernelGD}, 
following the above notations, we can represent this classical kernel-based gradient descent \eqref{classical_kernelGD} by
\be
\wh{f}_{t+1,D}=(I-\aaa  W_+'(0) L_{{K},D})\wh{f}_{t,D}+\aaa W_+'(0) \overline{S_{D}^*}\bm{y}_D,
\ee
which can be further expressed as 
\be
\wh{f}_{t+1,D}=(I-\aaa W_+'(0)  L_{{K},D_v})\wh{f}_{t,D}+\aaa W_+'(0)(L_{K,D_v}-L_{K,D})\wh{f}_{t,D}+\aaa W_+'(0)\overline{S_{D}^*}\bm{y}_D. \label{central_sequence_rewrite}
\ee
In this section, we aim to derive a crucial error decomposition for $f_{t,D_u}-\wh f_{t,D}$. To achieve this goal, we also need to introduce the following data-free auxiliary function sequence $\{\ww{f}_{t}\}$ with stepsize $\aaa W_+'(0)$ defined by $\ww f_0=0$ and
\bes
\begin{aligned}
\ww f_{t+1}=&\ww f_t-\aaa W_+'(0) L_K\llll(\ww f_t-f_\rho\rrrr)\\
=&\left(I-\aaa W_+'(0) L_K\right)\ww f_t+\aaa W_+'(0) L_Kf_\rho.
\end{aligned}
\ees
We can re-write this  data-free iteration as 
\be
\ww f_{t+1}=\llll(I-\aaa W_+'(0) L_{K,D_v}\rrrr)\ww f_t+\aaa W_+'(0)\llll(L_{K,D_v}-L_K\rrrr)\ww f_t+\aaa W_+'(0) L_Kf_\rho.\label{data_free_iteration}
\ee
Then subtraction between \eqref{dec_robust_compact} and \eqref{data_free_iteration}  yields that
\bea
\begin{aligned}
f_{t+1,D_u}-\ww f_{t+1}=&\sum_{v}\big[\bm{M}\big]_{uv}\Big[(I-\aaa W_+'(0) L_{{K},D_v})\left({f}_{t,D_v}-\ww{f}_{t}\right)+\aaa W_+'(0)\left(\overline{S_{D_v}^*}\bm{y}_{D_v}-L_Kf_\rho\right)\\
&-\aaa W_+'(0)\llll(L_{K,D_v}-L_K\rrrr)\ww f_t
+\aaa E_{t,D_v}\Big]. \label{maineq1}
\end{aligned}
\eea
Meanwhile,  \eqref{central_sequence_rewrite} and \eqref{data_free_iteration} also show that 
\bea
\begin{aligned}
\wh f_{t+1,D}-\ww f_{t+1}=&\sum_{v}\f{1}{m}\Big[(I-\aaa W_+'(0) L_{{K},D_v})\left(\wh{f}_{t,D}-\ww{f}_{t}\right)+\aaa W_+'(0)\left(L_{K,D_v}-L_{K,D}\right)\wh f_{t,D}\\
&+\aaa W_+'(0)\left(\overline{S_{D}^*}\bm{y}_{D}-L_Kf_\rho\right)-\aaa W_+'(0)(L_{K,D_v}-L_K)\ww f_t\Big]. \label{maineq2}
\end{aligned}
\eea
We observe that, when $|D_1|=|D_2|=\cdots=|D_m|=|D|/m$, it holds that
\bes
&&\sum_v\left(L_{K,D_v}-L_{K,D}\right)=\sum_v\frac{1}{|D_v|} \sum_{x \in D_v(x)}\left\langle\cdot, K_{x}\right\rangle_{K} K_{x}-\frac{m}{|D|} \sum_{x \in D(x)}\left\langle\cdot, K_{x}\right\rangle_{K} K_{x}=0,\\
&&\sum_{v}\left(\overline{S_{D_v}^*}\bm{y}_{D_v}-\overline{S_{D}^*}\bm{y}_{D}\right)=\sum_v\f{1}{|D_v|}\sum_{(x,y)\in D_v}yK_x-\f{m}{|D|}\sum_{(x,y)\in D}yK_x=0.
\ees
Hence, we obtain that
\bea
\begin{aligned}
	\wh f_{t+1,D}-\ww f_{t+1}=&\sum_{v}\f{1}{m}\Big[(I-\aaa W_+'(0) L_{{K},D_v})\left(\wh{f}_{t,D}-\ww{f}_{t}\right)\\
	&+\aaa W_+'(0)\left(\overline{S_{D_v}^*}\bm{y}_{D_v}-L_Kf_\rho\right)-\aaa W_+'(0)(L_{K,D_v}-L_K)\ww f_t\Big]. \label{maineq3}
\end{aligned}
\eea
For any given data set $D$, let us now denote 
\bea
\Psi_{t,D}=\left(\sd-L_Kf_\rho\right)-\left(L_{K,D}-L_K\right)\ww f_t.
\eea
Accordingly, for each $v\in\huaV$, we have the representation
\bea
\Psi_{t,D_v}=\left(\sdv-L_Kf_\rho\right)-\left(L_{K,D_v}-L_K\right)\ww f_t.
\eea
Then we have
\bea
f_{t+1,D_u}-\ww f_{t+1}=\sum_{v}\big[\bm{M}\big]_{uv}\Big[(I-\aaa W_+'(0) L_{{K},D_v})\left({f}_{t,D_v}-\ww{f}_{t}\right)+\aaa W_+'(0)\Psi_{t,D_v}+\aaa E_{t,D_v}\Big].
\eea
If we denote the index $v_0=u$, then iterating the above equality yields that, 
\bea
\begin{aligned}
f_{t+1,D_u}-\ww f_{t+1}=&\aaa\sum_{k=1}^{t+1}\sum_{v_1,v_2,...,v_k}\prod_{s=1}^k\big[\bm{M}\big]_{v_{s-1}v_s}\prod_{w=1}^{k-1}\left(I-\aaa  W_+'(0)L_{ K,D_{v_w}}\right)\\
&\left( W_+'(0)\Psi_{t-k+1,D_{v_k}}+E_{t-k+1,D_{v_k}}\right).  \label{subtraction1}
\end{aligned}
\eea
In a similar way, it holds that
\bea
\wh f_{t+1,D}-\ww f_{t+1}=\aaa W_+'(0)\sum_{k=1}^{t+1}\sum_{v_1,v_2,...,v_k}\f{1}{m^k}\prod_{w=1}^{k-1}\left(I-\aaa W_+'(0) L_{ K,D_{v_w}}\right)\Psi_{t-k+1,D_{v_k}}. \label{subtraction2}
\eea
Subtraction between \eqref{subtraction1} and \eqref{subtraction2} yields that 
\bea
\begin{aligned}
{f}_{t+1,D_{u}}-\wh{f}_{t+1,D}=&\aaa W_+'(0)\sum_{k=1}^{t+1}\sum_{v_1,v_2,...,v_k}\left(\prod_{s=1}^k\big[\bm{M}\big]_{v_{s-1}v_s}-\f{1}{m^k}\right)\prod_{w=1}^{k-1}\left(I-\aaa W_+'(0) L_{ K,D_{v_w}}\right)\Psi_{t-k+1,D_{v_k}}\\
&+\aaa\sum_{k=1}^{t+1}\sum_{v_1,v_2,...,v_k}\prod_{s=1}^k\big[\bm{M}\big]_{v_{s-1}v_s}\prod_{w=1}^{k-1}\left(I-\aaa W_+'(0) L_{ K,D_{v_w}}\right)E_{t-k+1,D_{v_k}}.
\end{aligned}
\eea
Then we arrive at our key error decomposition which is summarized in the following proposition.
\begin{pro}\label{main_decomposition}
 Let $\{f_{t,D_u}\}_{u\in\huaV}$ and $\{\wh f_{t,D}\}$ be the sequence generated from the decentralized robust kernel-based learning algorithm \eqref{main_alg1}-\eqref{main_alg2} and kernel-based gradient descent algorithm \eqref{classical_kernelGD}, respectively. Then we have the following error decomposition
 \bes
 {f}_{t+1,D_{u}}-\wh{f}_{t+1,D}=\huaT_{1,t}+\huaT_{2,t}+\huaT_{3,t},
 \ees
 where
 \bes
 &&\huaT_{1,t}=\aaa W_+'(0)\sum_{k=1}^{t+1}\sum_{v_1,v_2,...,v_k}\left(\prod_{s=1}^k\big[\bm{M}\big]_{v_{s-1}v_s}-\f{1}{m^k}\right)\left(I-\aaa W_+'(0) L_{ K}\right)^{k-1}\Psi_{t-k+1,D_{v_k}},\\
 &&\huaT_{2,t}=\aaa W_+'(0)\sum_{k=1}^{t+1}\sum_{v_1,v_2,...,v_k}\left(\prod_{s=1}^k\big[\bm{M}\big]_{v_{s-1}v_s}-\f{1}{m^k}\right)\Bigg[\prod_{w=1}^{k-1}\left(I-\aaa W_+'(0) L_{ K,D_{v_w}}\right)\\
 &&\quad\quad\quad-\left(I-\aaa W_+'(0) L_{ K}\right)^{k-1}\Bigg]\Psi_{t-k+1,D_{v_k}},\\
 &&\huaT_{3,t}=\aaa\sum_{k=1}^{t+1}\sum_{v_1,v_2,...,v_k}\prod_{s=1}^k\big[\bm{M}\big]_{v_{s-1}v_s}\prod_{w=1}^{k-1}\left(I-\aaa W_+'(0) L_{ K,D_{v_w}}\right)E_{t-k+1,D_{v_k}}.
 \ees
\end{pro}
In the subsequent sections, we aim to provide corresponding detailed estimates for $\huaT_{1,t}$, $\huaT_{2,t}$, $\huaT_{3,t}$ which serve as core ingredients for proving our main results. Before coming to main analysis, we present several basic lemmas that will be used later on.
The following result (see e.g. \cite{daw2011}) is a useful mixing property of the transition matrix of the communication matrix $\bm{M}$. The lemma will be often utilized in subsequent analysis of main proofs.
\begin{lem}\label{Nedic_lemma}
	For all agents $i,j\in \huaV$ and all $t\geq s\geq 0$, there holds
	\begin{eqnarray}
		\sum_v\left|\Big[{\bm{\bm{M}}}^{t-s}\Big]_{uv}-\f{1}{m}\right|\leq 2(\sqrt{m}\gamma_{\bm{M}}^{t-s}\wedge1),
	\end{eqnarray}
	with $\gamma_{\bm{M}}$ the second largest eigenvalue of $\bm{M}$ in absolute value.
\end{lem}

We also need the following basic concentration inequalities for    Hilbert-valued   random variables (see e.g. \cite{pinelis1994}, \cite{wh2019a}).
\begin{lem}\label{hilbert_value_concentration}
	Let $(\huaH,\|\cdot\|_\huaH)$ be a separable Hilbert space, and let $\zeta$ be any random variable with values in $\huaH$ satisfying $\|\zeta\|_\huaH\leq \ww M<\infty$ almost surely. Let $\{\zeta_1, \zeta_2..., \zeta_N\}$ be a sample of $N$ independent observations for $\zeta$. Then for any $1<\delta<1$, there holds, with probability $1-\delta$, 
	\bes
	\left\|\f{1}{N}\sum_{i=1}^N\zeta_i-\mbb E(\zeta)\right\|_\huaH\leq\f{2\ww M\log(2/\delta)}{N}+\sqrt{\f{2\mbb E(\|\zeta\|_\huaH^2)\log(2/\delta)}{N}}.
	\ees
\end{lem}

\begin{lem}\label{unbddnoise_lem}
	Let  $(\huaH,\|\cdot\|_\huaH)$ be a separable Hilbert space,  and $\zeta$ be a random variable with values in $\huaH$ satisfying that, there exist constants $\ww M, B>0$, $\mbb E[\|\zeta\|_\huaH^p]\leq\f{B}{2}p!\ww M^{p-2}$ for any $2\leq p\in \mbb N_+$. Let $\{\zeta_1,\zeta_2,...,\zeta_N\}$ be a sample of $N$ independent observations for $\zeta$, then we have, for $0<\delta<1$,
	\be
	\left\|\f{1}{N}\sum_{i=1}^N\zeta_i-\mbb E[\zeta]\right\|_\huaH\leq\f{2\ww M}{N}\log\f{2}{\delta}+\sqrt{\f{2B}{N}\log\f{2}{\delta}}.
	\ee
\end{lem}
A special case of Lemma \ref{unbddnoise_lem} is the following lemma.
\begin{lem}\label{unbddnoise_lem_R}
	Let $\{\zeta_i\}_{i=1}^N$ be an independent random sequence satisfying $\mbb E\zeta_i=0$ and $\mbb E |\zeta_i|^p\leq\f{B}{2}p!\ww M^{p-2}$ for some constants $\ww M, B>0$ and any $2\leq p\in \mbb N_+$, $i=1,2,...,N$. Then, with probability $1-\delta$,
	\be
	\left|\f{1}{N}\sum_{i=1}^N\zeta_i-\mbb E[\zeta]\right|\leq\f{2\ww M}{N}\log\f{2}{\delta}+\sqrt{\f{2B}{N}\log\f{2}{\delta}}.
	\ee
\end{lem}
The following lemma (see. e.g. \cite{yfsz2024}) is basic for estimating operator norms in our estimates of subsequent proofs.
\begin{lem}\label{operator_lem}
	Let $U$ be a compact positive operator on a real separable Hilbert space, such that $\|U\|\leq C_*$ for some $C_*>0$. Let $l\leq k$ and $\beta_l$, $\beta_{l+1}$, ..., $\beta_k\in (0,1/C_*]$. Then when $\theta>0$, there holds,
	\bes
	\left\|U^\theta\prod_{i=l}^k\left(I-\beta_iU\right)\right\|\leq\sqrt{\f{(\theta/e)^{2\theta}+C_*^{2\theta}}{1+\left(\sum_{j=l}^k\beta_j\right)^{2\theta}}}, \ \text{and} \ \left\|\prod_{i=l}^k\left(I-\beta_iU\right)\right\|\leq1.
	\ees
\end{lem}

\section{Estimates on $\huaT_{1,t}$}

This section is devoted to  estimates on $\huaT_{1,t}$ defined in Proposition \ref{main_decomposition}. The core  estimates are included in the following two propositions.

\begin{pro}\label{T1L2_pro_first}
	Assume \eqref{moment_condition}, \eqref{regularity_ass} holds with $r>\f{1}{2}$, the stepsize $\aaa$ satisfies $\aaa\leq\f{1}{\kkk^2}\min\{\f{1}{W_+'(0)},\f{1}{C_W}\}$, then for any $0<\delta<1$, with confidence at least $1-\delta$,  there holds, for $t\in\mbb N_+$,
	\bes
	\|\huaT_{1,t}\|_{L_{\rho_\huaX}^2}\lesssim\left(\log\f{4}{\delta}\right)\aaa^{\f{1}{2}}\sum_{k=1}^{t+1}\sum_{v}\left|\Big[{\bm{M}}^k\Big]_{uv}-\f{1}{m}\right|\left(\f{\log m}{\sqrt{|D_v|}}\right).
	\ees
\end{pro}
\begin{proof}
Substituting the representation of $\Psi_{t,D_v}$ defined above and summing over the index $v_1,v_2,...,v_{k-1}$, we have
\bes
&&\huaT_{1,t}=\aaa W_+'(0)\sum_{k=1}^{t+1}\sum_{v}\left(\Big[{\bm{M}}^k\Big]_{uv}-\f{1}{m}\right)\left(I-\aaa W_+'(0) L_{ K}\right)^{k-1}\\
&&\quad\quad\left[\left(\sdvk-L_Kf_\rho\right)-\left(L_{K,D_{v_k}}-L_K\right)\ww f_{t-k+1}\right].
\ees
After taking $L_{\rho_\huaX}^2$ norms on both sides of the above equality, we have
\bea
\nono&&\|\huaT_{1,t}\|_{L_{\rho_{\huaX}}^2}=\left\|L_{ K}^{1/2}\huaT_{1,t}\right\|_{ K}\leq \aaa W_+'(0)\sum_{k=1}^{t+1}\sum_{v}\left|\Big[{\bm{M}}^k\Big]_{uv}-\f{1}{m}\right|\left\|L_{ K}^{1/2}\left(I-\aaa W_+'(0) L_{ K}\right)^{k-1}\right\|\\
&&\times\left[\left\|\sdvk-L_Kf_\rho\right\|_K+\left\|L_{K,D_{v_k}}-L_K\right\|\llll\|\ww f_{t-k+1}\rrrr\|_K\right]. \label{T1L2_first}
\eea
Denote the Hilbert-valued random variable $\zeta:\huaX\rightarrow\text{HS}(\huaH_{ K})$ by $\zeta(x)=\left\langle\cdot, K_{x}\right\rangle_{K} K_{x}$,
where $\text{HS}(\huaH_{ K})$ denotes the Hilbert space of Hilbert-Schmidt operators on $\huaH_{ K}$. Then we have $L_{ K,D}=\f{1}{|D|}\sum_{x\in D(x)}\zeta(x)$, $L_{ K,D_v}=\f{1}{|D_v|}\sum_{x\in D_v(x)}\zeta(x)$, $v\in\huaV$ 
and $\mbb E \zeta=L_{ K}$.  Lemma \ref{hilbert_value_concentration} indicates that, for any data set $D$, with confidence at least $1-m\delta$, $$\left\|L_{ K,D_v}-L_{ K}\right\|\lesssim \f{1}{\sqrt{|D_v|}}\left(\log\f{2}{\delta}\right), v\in\huaV.$$
Denote the random variable $\zeta':\huaX\times\huaY\rightarrow\huaH_K$ by
$\zeta'(x,y)=yK_{x}$.
Then it follows  from Lemma \ref{unbddnoise_lem}  that, 
for any data set $D$,	there holds, with confidence at least $1-m\delta$,
\bes
	\left\|\sdv-L_{ K} f_{\rho}\right\|_{ K}\lesssim \f{1}{\sqrt{|D_v|}}\left(\log\f{2}{\delta}\right).
	\ees
	By utilizing Lemma \ref{operator_lem} to $U= W_+'(0)L_{ K}$, noticing $\|W_+'(0)L_{ K}\|\leq W_+'(0)\kkk^2$ and using the fact that $0<\aaa\leq\f{1}{\kkk^2W_+'(0)}$, we know, when $k\geq2$,
	\bes
	\left\|( W_+'(0)L_{ K})^{1/2}(I-\aaa W_+'(0) L_{ K})^{k-1}\right\|\lesssim\f{1}{\sqrt{1+\sum_{j=1}^{k-1}\aaa}}\lesssim\aaa^{-\f{1}{2}}.
	\ees
	We also note that, when $k=1$, there holds $\aaa^{\f{1}{2}}\|(W_+'(0)L_{ K})^{1/2}(I-\aaa W_+'(0) L_{ K})^{k-1}\|=\aaa^{\f{1}{2}}\|(W_+'(0)L_{ K})^{1/2}\|\lesssim1$. Thus we have
	for $k\geq1$, 
		\bes
	\left\|L_{ K}^{1/2}(I-\aaa W_+'(0) L_{ K})^{k-1}\right\|\lesssim(\aaa W_+'(0))^{-\f{1}{2}}.
	\ees
	On the other hand,  according to \cite{yrc2007}, we know, when $r>\f{1}{2}$,
	\bes
	\|\ww f_{t}\|_{K}\leq \|\ww f_{t}-f_\rho\|_K+\|f_\rho\|_K \lesssim t^{-(r-\f{1}{2})}+\|f_\rho\|_K\lesssim1.
	\ees
	Combining the above inequalities with \eqref{T1L2_first}, we have, with probability at least $1-(1+m)\delta$,
	\bes
	&&\|\huaT_{1,t}\|_{L_{\rho_\huaX}^2}\lesssim \aaa^{\f{1}{2}}\sum_{k=1}^{t+1}\sum_{v}\left|\Big[{\bm{M}}^k\Big]_{uv}-\f{1}{m}\right|\f{1}{\sqrt{|D_v|}}\left(\log\f{4}{\delta}\right). 
	\ees
	Re-scaling $\delta$, we obtain, with confidence at least $1-\delta$,
	\bes
	\|\huaT_{1,t}\|_{L_{\rho_\huaX}^2}\lesssim\left(\log\f{4}{\delta}\right)\aaa^{\f{1}{2}}\sum_{k=1}^{t+1}\sum_{v}\left|\Big[{\bm{M}}^k\Big]_{uv}-\f{1}{m}\right|\left(\f{\log m}{\sqrt{|D_v|}}\right),
	\ees
	which completes the proof.
\end{proof}

\begin{pro} \label{T1t_pro_second} Under  assumptions of Proposition \ref{T1L2_pro_first},
if $|D_u|=\f{|D|}{m}=n$, $u\in\huaV$, then we have, for any $0<\delta<1$, with probability at least $1-\delta$,
\bes
\|\huaT_{1,t}\|_{L_{\rho_\huaX}^2}\lesssim \left(\log\f{4}{\delta}\right)\aaa^{\f{1}{2}}\left(\f{\log^2 m}{1-\gamma_{\bm{M}}}\right)\left(\f{\sqrt{m}}{\sqrt{n}}\right).
\ees
\end{pro}

\begin{proof}
	We know from Lemma \ref{Nedic_lemma} that 
	\bes
		\sum_v\left|\Big[{\bm{\bm{M}}}^k\Big]_{uv}-\f{1}{m}\right|\leq 2(\sqrt{m}\gamma_{\bm{M}}^{k}\wedge1).
	\ees
	 Proposition \ref{T1L2_pro_first} then implies that, with confidence at least $1-\delta$,
		\bes
	\|\huaT_{1,t}\|_{L_{\rho_\huaX}^2}\lesssim\left(\log\f{4}{\delta}\right)\aaa^{\f{1}{2}}\sum_{k=1}^{t+1}\left(\sqrt{m}\gamma_{\bm{M}}^{k}\wedge1\right)\left(\f{\log m}{\sqrt{|D_v|}}\right).
	\ees
	Then we can spit the right hand side by
		\bes
	\|\huaT_{1,t}\|_{L_{\rho_\huaX}^2}\lesssim\left(\log\f{4}{\delta}\right)\aaa^{\f{1}{2}}\left(\sum_{k=1}^{t_m}+\sum_{k=t_m+1}^{t+1}\right)\left(\sqrt{m}\gamma_{\bm{M}}^{k}\wedge1\right)\left(\f{\log m}{\sqrt{|D_v|}}\right),
	\ees
	where 
	$t_m=\left\lfloor \f{\log m}{2\log\f{1}{\gamma_{\bm{M}}}}\right\rfloor$.
	Noticing that $\sqrt{m}\gamma_{\bm{M}}^{t_m}\geq1$, $\sqrt{m}\gamma_{\bm{M}}^{t_m+1}\leq1$ and $t_m\lesssim\f{\log m}{1-\gamma_{\bm{M}}}$, we have with confidence at least $1-\delta$,
	\bes
	\|\huaT_{1,t}\|_{L_{\rho_\huaX}^2}&\lesssim&\left(\log\f{4}{\delta}\right)\aaa^{\f{1}{2}}\left(t_m+\f{\sqrt{m}}{1-\gamma_{\bm{M}}}\right)\left(\f{\log m}{\sqrt{|D_v|}}\right)\\
	&\lesssim&\left(\log\f{4}{\delta}\right)\aaa^{\f{1}{2}}\left(\f{\log^2 m}{1-\gamma_{\bm{M}}}\right)\left(\f{\sqrt{m}}{\sqrt{n}}\right).
	\ees
	The proof is complete.
\end{proof}

\section{Estimates on $\huaT_{2,t}$}
This section is used to obtain estimates for $\huaT_{2,t}$ in Proposition \ref{main_decomposition}. Note that, for
$t>2\bar t\geq4$, and $t+1\geq k\geq\bar t+2$, we have the decomposition 
\bes
&&\prod_{w=1}^{k-1}\left(I-\aaa W_+'(0) L_{ K,D_{v_w}}\right)-\left(I-\aaa W_+'(0) L_{ K}\right)^{k-1}\\
&&=\prod_{w=1}^{k-1}\left(I-\aaa W_+'(0) L_{ K,D_{v_w}}\right)-\prod_{w=1}^{k-\bar t-1}\left(I-\aaa  W_+'(0) L_{ K,D_{v_w}}\right)\left(I-\aaa  W_+'(0) L_{ K}\right)^{\bar t}\\
&&+\prod_{w=1}^{k-\bar t-1}\left(I-\aaa  W_+'(0) L_{ K,D_{v_w}}\right)\left(I-\aaa  W_+'(0) L_{ K}\right)^{\bar t}-
\left(I-\aaa W_+'(0) L_{ K}\right)^{k-1}
\ees
which can be further written as 
\bes
	&&\prod_{w=1}^{k-\bar t-1}\left(I-\aaa W_+'(0) L_{ K,D_{v_w}}\right)\left[\prod_{w=k-\bar t}^{k-1}\left(I-\aaa W_+'(0) L_{ K,D_{v_w}}\right)-\left(I-\aaa  W_+'(0)L_{ K}\right)^{\bar t}\right]\\
	&&+\left[\prod_{w=1}^{k-\bar t-1}\left(I-\aaa W_+'(0) L_{ K,D_{v_w}}\right)-\left(I-\aaa  W_+'(0)L_{ K}\right)^{k-\bar t-1}
	\right]\left(I-\aaa W_+'(0) L_{ K}\right)^{\bar t}\\
	&&=:\prod(v_{1:k-\bar t-1})\wh \prod(v_{k-\bar t:k-1})+\wh \prod(v_{1:k-\bar t-1})\left(I-\aaa W_+'(0) L_{ K}\right)^{\bar t},
\ees
where we have used the notation, for $p\leq q$ with  $p,q\in\mbb N+$,
\bes
&&\prod(v_{p:q}):=\prod_{w=p}^{q}\left(I-\aaa W_+'(0) L_{ K,D_{v_w}}\right),\\
&&\wh \prod(v_{p:q}):=\prod(v_{p:q})-\left(I-\aaa W_+'(0) L_{ K}\right)^{q-p+1}.
\ees
Then we have the error decomposition for $\huaT_{2,t}$ which is included in the following proposition.
\begin{pro}\label{decompositionT2t}
	Let $\{f_{t,D_u}\}_{u\in\huaV}$ and $\{f_{t,D}\}$ be the sequences generated from the decentralized robust kernel-based learning algorithm \eqref{main_alg1}-\eqref{main_alg2} and kernel-based gradient descent algorithm \eqref{classical_kernelGD}, respectively. 
	Let $\huaT_{2,t}$ be defined in Proposition \ref{main_decomposition}. Then for $t>2\bar t\geq4$, we have the following error decomposition
\bea
\huaT_{2,t}=\huaT_{2,t}^A+\huaT_{2,t}^B+\huaT_{2,t}^C,
\eea
where 
\bes
&&\huaT_{2,t}^A=\aaa W_+'(0)\sum_{k=1}^{2\bar t}\sum_{v_1,v_2,...,v_k}\left(\prod_{s=1}^k\big[\bm{M}\big]_{v_{s-1}v_s}-\f{1}{m^k}\right)\wh \prod(v_{1:k-1})\Psi_{t-k+1,D_{v_k}},\\
&&\huaT_{2,t}^B=\aaa W_+'(0)\sum_{k=2\bar t+1}^{t+1}\sum_{v_1,v_2,...,v_k}\left(\prod_{s=1}^k\big[\bm{M}\big]_{v_{s-1}v_s}-\f{1}{m^k}\right)\prod(v_{1:k-\bar t-1})\wh \prod(v_{k-\bar t:k-1})\Psi_{t-k+1,D_{v_k}},\\
&&\huaT_{2,t}^C=\aaa W_+'(0)\sum_{k=2\bar t+1}^{t+1}\sum_{v_1,v_2,...,v_k}\left(\prod_{s=1}^k\big[\bm{M}\big]_{v_{s-1}v_s}-\f{1}{m^k}\right)\wh \prod(v_{1:k-\bar t-1})(I-\aaa  W_+'(0) L_K)^{\bar t}\Psi_{t-k+1,D_{v_k}}.
\ees
\end{pro}
For deriving main results, we also need a further decomposition for  $\huaT_{2,t}^C$. Noticing that, for $k\geq\bar t+2$, we have the following decomposition:
\bes
	&&\sum_{v_{k-\bar t},...,v_{k-1}}\left(\prod_{s=1}^k\big[\bm{M}\big]_{v_{s-1}v_s}-\f{1}{m^k}\right)\\
	&&=\prod_{s=1}^{k-\bar t-1}\big[\bm{M}\big]_{v_{s-1}v_s}\left(\Big[{\bm{M}}^{\bar t+1}\Big]_{v_{k-\bar t-1} v_k}-\f{1}{m}\right)+\f{1}{m}\left(\prod_{s=1}^{k-\bar t-1}\big[\bm{M}\big]_{v_{s-1}v_s}-\f{1}{m^{k-\bar t-1}}\right).
\ees
Hence, for $2\bar t+1\leq k\leq t+1$, it holds that
\bes
&&\sum_{v_1,v_2,...,v_k}\left(\prod_{s=1}^k\big[\bm{M}\big]_{v_{s-1}v_s}-\f{1}{m^k}\right)\wh \prod(v_{1:k-\bar t-1})(I-\aaa W_+'(0) L_K)^{\bar t}\Psi_{t-k+1,D_{v_k}}\\
&&=\sum_{v_k}\sum_{v_1,...,v_{k-\bar t-1}}\sum_{v_{k-\bar t},...,v_{k-1}}\left(\prod_{s=1}^k\big[\bm{M}\big]_{v_{s-1}v_s}-\f{1}{m^k}\right)\wh \prod(v_{1:k-\bar t-1})(I-\aaa W_+'(0) L_K)^{\bar t}\Psi_{t-k+1,D_{v_k}}\\
&&=:\huaQ_{t,k}^A+\huaQ_{t,k}^B,
\ees
where
\bea
\begin{aligned}
\huaQ_{t,k}^A=&\sum_{v_k}\sum_{v_1,...,v_{k-\bar t-1}}\prod_{s=1}^{k-\bar t-1}\big[\bm{M}\big]_{v_{s-1}v_s}\left(\bigg[{\bm{M}}^{\bar t+1}\bigg]_{v_{k-\bar t-1} v_k}-\f{1}{m}\right)\\
&\wh \prod(v_{1:k-\bar t-1})(I-\aaa  W_+'(0) L_K)^{\bar t}\Psi_{t-k+1,D_{v_k}}  \label{QA_def}
\end{aligned}
\eea
and
\bes
\begin{aligned}
	\huaQ_{t,k}^B=&\sum_{v_k}\sum_{v_1,...,v_{k-\bar t-1}}\f{1}{m}\left(\prod_{s=1}^{k-\bar t-1}\big[\bm{M}\big]_{v_{s-1}v_s}-\f{1}{m^{k-\bar t-1}}\right)\\
	&\wh \prod(v_{1:k-\bar t-1})(I-\aaa  W_+'(0) L_K)^{\bar t}\Psi_{t-k+1,D_{v_k}}.
\end{aligned}
\ees
Due to the fact that 
\bes
\f{1}{m}\sum_v\Psi_{t,D_v}=\left(\sd-L_Kf_\rho\right)-\left(L_{K,D}-L_K\right)\ww f_t=\Psi_{t,D},
\ees
after summing over the index $v_k$, $\huaQ_{t,k}^B$ can be expressed as 
\bea
\huaQ_{t,k}^B=\sum_{v_1,...,v_{k-\bar t-1}}\left(\prod_{s=1}^{k-\bar t-1}\big[\bm{M}\big]_{v_{s-1}v_s}-\f{1}{m^{k-\bar t-1}}\right)\wh \prod(v_{1:k-\bar t-1})(I-\aaa  W_+'(0) L_K)^{\bar t}\Psi_{t-k+1,D}. \label{QB_def}
\eea
Hence, we have the error decomposition for $\huaT_{2,t}^C$, which can be summarized in the following proposition.
\begin{pro}
	We have the decomposition 
\bes
\huaT_{2,t}^C=\huaT_{2,t}^{C_1}+\huaT_{2,t}^{C_2},
\ees
with
\bes
\huaT_{2,t}^{C_1}=\aaa W_+'(0)\sum_{k=2\bar t+1}^{t+1}\huaQ_{t,k}^A,\ \huaT_{2,t}^{C_2}=\aaa W_+'(0)\sum_{k=2\bar t+1}^{t+1}\huaQ_{t,k}^B.
\ees
where $\huaQ_{t,k}^A$ and $\huaQ_{t,k}^B$ are defined in \eqref{QA_def} and \eqref{QB_def}, respectively.
\end{pro}
In the left part of this section, we will rigorously establish estimates for $\huaT_{2,t}^A$, $\huaT_{2,t}^B$ and $\huaT_{2,t}^C$ respectively.
\subsection{Estimates on $\huaT_{2,t}^A$}
The next proposition provides core estimates for $\huaT_{2,t}^A$.

\begin{pro}\label{T2tA_est}
	Assume \eqref{moment_condition}, \eqref{regularity_ass} holds with $r>\f{1}{2}$, the stepsize $\aaa$ satisfies $\aaa\leq\f{1}{\kkk^2}\min\{\f{1}{W_+'(0)},\f{1}{C_W}\}$. If $|D_u|=\f{|D|}{m}=n$, $u\in\huaV$, then for $t,\bar t \in\mbb N_+$, $t\geq2\bar t\geq4$, there holds, for any $0<\delta<1$, with confidence at least $1-\delta$, 
	\bes
	\left\|\huaT_{2,t}^A\right\|_{L_{\rho_\huaX}^2}\lesssim\aaa^{\f{3}{2}}\bar t^{\f{3}{2}}\left(\log\f{4m}{\delta}\right)^2\f{1}{n}.
	\ees
\end{pro}

\begin{proof}
	Taking $L_{\rho_\huaX}^2$ norms on both sides of $\huaT_{2,t}^A$, we have
\bes
&&\left\|\huaT_{2,t}^A\right\|_{L_{\rho_\huaX}^2}=\left\|L_K^{1/2}\huaT_{2,t}^A\right\|_K\\
&&\leq\aaa\sum_{k=1}^{2\bar t}\sum_{v_1,v_2,...,v_k}\left|\prod_{s=1}^k\big[\bm{M}\big]_{v_{s-1}v_s}-\f{1}{m^k}\right|\left\|L_K^{1/2}\wh \prod(v_{1:k-1})\Psi_{t-k+1,D_{v_k}}\right\|_K.
\ees
Noting that, there holds the algebra identity 
\bes
\wh \prod(v_{1:k-1})=\aaa  W_+'(0)\sum_{\ell=1}^{k-1}\left\{\prod_{w=1}^{\ell-1}\left(I-\aaa  W_+'(0) L_{ K,D_{v_w}}\right)\right\}\left(L_{ K}-L_{ K,D_{v_\ell}}\right)\left(I-\aaa  W_+'(0) L_K\right)^{k-\ell-1}.
\ees
Using this identity and noticing that, for any two self-adjoint operators $T_1$, $T_2$, there holds $\|T_1T_2\|=\|T_2T_1\|$, we have
\bes
\begin{aligned}
	\left\|L_K^{1/2}\wh \prod(v_{1:k-1})\Psi_{t-k+1,D_{v_k}}\right\|_K\leq&\aaa W_+'(0)\sum_{\ell=1}^{k-1}\left\|\prod_{w=1}^{\ell-1}\left(I-\aaa  W_+'(0) L_{ K,D_{v_w}}\right)\right\|\times\left\|L_{ K}-L_{ K,D_{v_\ell}}\right\|\\
	&\times\left\|L_K^{1/2}\left(I-\aaa  W_+'(0) L_K\right)^{k-\ell-1}\right\|\times\left\|\Psi_{t-k+1,D_{v_k}}\right\|_K.
\end{aligned}
\ees
Hence it follows that
\bea
\begin{aligned}
	\left\|\huaT_{2,t}^A\right\|_{L_{\rho_\huaX}^2}\leq&\aaa^2 W_+'(0)^2\max_{v\in\huaV}\Big\{\left\|L_K-L_{K,D_v}\right\|\Big\}\sum_{k=1}^{2\bar t}\sum_{v_1,v_2,...,v_k}\left|\prod_{s=1}^k\big[\bm{M}\big]_{v_{s-1}v_s}-\f{1}{m^k}\right|\\
	&\times\sum_{\ell=1}^{k-1}\left\|L_K^{1/2}\left(I-\aaa  W_+'(0) L_K\right)^{k-\ell-1}\right\|\times\left\|\Psi_{t-k+1,D_{v_k}}\right\|_K, \label{T2t_eq1}
\end{aligned}
\eea
which can be further bounded by
\bes
&&2\aaa^2 W_+'(0)^2\max_{v\in\huaV}\Big\{\left\|L_K-L_{K,D_v}\right\|\Big\}\sum_{k=1}^{2\bar t}\sum_{\ell=1}^{k-1}\left\|L_K^{1/2}\left(I-\aaa  W_+'(0) L_K\right)^{k-\ell-1}\right\|\sup_{t',v}\Big\{\left\|\Psi_{t',D_{v}}\right\|_K\Big\}.
\ees
Based on the fact that
\bes
\left\|L_K^{1/2}\left(I-\aaa  W_+'(0) L_K\right)^{k-\ell-1}\right\|\leq \f{ W_+'(0)^{-\f{1}{2}}}{\sqrt{\aaa(k-\ell-1)}}, 
\ees
we are able to derive
\bes
\sum_{\ell=1}^{k-1}\left\|L_K^{1/2}\left(I-\aaa  W_+'(0) L_K\right)^{k-\ell-1}\right\|\leq\|L_K^{1/2}\|+\sum_{\ell=1}^{k-2}\f{ W_+'(0)^{-\f{1}{2}}}{\sqrt{\aaa(k-\ell-1)}}\lesssim\aaa^{-\f{1}{2}}\sqrt{k}.
\ees
Then it holds that
\bea
\left\|\huaT_{2,t}^A\right\|_{L_{\rho_\huaX}^2}\lesssim{\aaa}^{\f{3}{2}} \max_{v\in\huaV}\Big\{\left\|L_K-L_{K,D_v}\right\|\Big\}\sum_{k=1}^{2\bar t}\sqrt{k}\sup_{t',v}\Big\{\left\|\Psi_{t',D_{v}}\right\|_K\Big\}. \label{T2t_eq2}
\eea
When $|D_1|=|D_2|=\cdots=|D_m|=\f{1}{n}$, we know from the proof of Proposition \ref{T1L2_pro_first} that, with confidence at least $1-2m\delta$,
\bes
&&\|L_{K,D_v}-L_K\|\lesssim\left(\log\f{2}{\delta}\right)\f{1}{\sqrt{n}}, \ v=1,2,...,m,\\
&&\left\|\sdv-L_{K,D_v}f_\rho\right\|_K\lesssim \left(\log\f{2}{\delta}\right)\f{1}{\sqrt{n}},\ v=1,2,...,m,
\ees
hold simultaneously. Hence it follows that, with probability at least $1-2m\delta$,
\bes
\left\|\huaT_{2,t}^A\right\|_{L_{\rho_\huaX}^2}\lesssim{\aaa}^{\f{3}{2}}\left(\log\f{2}{\delta}\right)^2\f{1}{\sqrt{n}}\sum_{k=1}^{2\bar t}\f{\sqrt{k}}{\sqrt{n}}.
\ees
After re-scaling and simplification, we finally obtain, with probability at least $1-\delta$,
\bes
\left\|\huaT_{2,t}^A\right\|_{L_{\rho_\huaX}^2}\lesssim\aaa^{\f{3}{2}}\bar t^{\f{3}{2}}\left(\log\f{4m}{\delta}\right)^2\f{1}{n},
\ees
which completes the proof.
\end{proof}

\subsection{Estimates on $\huaT_{2,t}^B$}
The next result provides core estimates for $\huaT_{2,t}^B$.
	\begin{pro}\label{T2tB_est}
		Assume \eqref{moment_condition}, \eqref{regularity_ass} holds with $r>\f{1}{2}$, the stepsize $\aaa$ satisfies $\aaa\leq\f{1}{\kkk^2}\min\{\f{1}{W_+'(0)},\f{1}{C_W}\}$. If $|D_u|=\f{|D|}{m}=n$, $u\in\huaV$, then for $t,\bar t \in\mbb N_+$, $t\geq2\bar t\geq4$, there holds, for any $0<\delta<1$, with confidence at least $1-\delta$, 
		\bes
		\left\|\huaT_{2,t}^B\right\|_{L_{\rho_\huaX}^2}\lesssim\aaa^{\f{3}{2}} \bar t^{\f{1}{2}}(t-2\bar t)\left(\log \f{4m}{\delta}\right)^2\f{1}{n}.
		\ees
\end{pro}
\begin{proof}
	After taking the $L_{\rho_\huaX}^2$ norm of $\huaT_{2,t}^B$, we have
\bes
\begin{aligned}
\left\|\huaT_{2,t}^B\right\|_{L_{\rho_\huaX}^2}\leq&\aaa W_+'(0)\sum_{k=2\bar t+1}^{t+1}\sum_{v_1,v_2,...,v_k}\left|\prod_{s=1}^k\big[\bm{M}\big]_{v_{s-1}v_s}-\f{1}{m^k}\right|\left\|\prod(v_{1:k-\bar t-1})\right\|\\
&\times\left\|L_K^{1/2}\wh \prod(v_{k-\bar t:k-1})\Psi_{t-k+1,D_{v_k}}\right\|_K.
\end{aligned}
\ees
According to the identity
\bes
\wh \prod(v_{k-\bar t:k-1})=\aaa W_+'(0)\sum_{\ell=k-\bar t}^{k-1}\left\{\prod_{w=k-\bar t}^{\ell-1}\left(I-\aaa  W_+'(0) L_{ K,D_{v_w}}\right)\right\}\left(L_{ K}-L_{ K,D_{v_\ell}}\right)\left(I-\aaa  W_+'(0) L_K\right)^{k-\ell-1},
\ees
we can obtain
\bes
&&\left\|L_K^{1/2}\wh \prod(v_{k-\bar t:k-1})\Psi_{t-k+1,D_{v_k}}\right\|_K\\
&&\leq\aaa W_+'(0)\sum_{\ell=k-\bar t}^{k-1}\max_{v\in\huaV}\Big\{\left\|L_K-L_{K,D_v}\right\|\Big\}\left\|L_K^{1/2}\left(I-\aaa  W_+'(0) L_K\right)^{k-\ell-1}\right\|\times\left\|\Psi_{t-k+1,D_{v_k}}\right\|_K\\
&&\leq\aaa W_+'(0)\max_{v\in\huaV}\Big\{\left\|L_K-L_{K,D_v}\right\|\Big\}\sum_{\ell=0}^{\bar t-1}\left\|L_K^{1/2}\left(I-\aaa  W_+'(0) L_K\right)^{\ell}\right\|\left\|\Psi_{t-k+1,D_{v_k}}\right\|_K.
\ees
Noticing that the following inequality holds,
\bea
\sum_{\ell=0}^{\bar t-1}\left\|L_K^{1/2}\left(I-\aaa  W_+'(0) L_K\right)^{\ell}\right\|\lesssim\aaa^{-\f{1}{2}} W_+'(0)^{-\f{1}{2}}\left(1+\sum_{\ell=1}^{\bar t-1}\f{1}{\sqrt{\ell}}\right)\lesssim\bar t^{\f{1}{2}}\aaa^{-\f{1}{2}} W_+'(0)^{-\f{1}{2}}, \label{T2tB_eq1}
\eea
we have
\bes
\left\|L_K^{1/2}\wh \prod(v_{k-\bar t:k-1})\Psi_{t-k+1,D_{v_k}}\right\|_K\lesssim \aaa^{\f{1}{2}} W_+'(0)^{\f{1}{2}}\bar t^{\f{1}{2}}\max_{v\in\huaV}\Big\{\left\|L_K-L_{K,D_v}\right\|\Big\}\max_{v\in\huaV}\left\{\left\|\Psi_{t-k+1,D_{v}}\right\|_K\right\}.
\ees
On the other hand, we know from the above discussions that, with confidence at least $1-2m\delta$,
\bes
&&\|L_{K,D_v}-L_K\|\lesssim\left(\log\f{2}{\delta}\right)\f{1}{\sqrt{n}}, \ v=1,2,...,m,\\
&&\left\|\sdv-L_{K,D_v}f_\rho\right\|_K\lesssim \left(\log\f{2}{\delta}\right)\f{1}{\sqrt{n}},\ v=1,2,...,m,
\ees
hold simultaneously. Then we have, with confidence at least $1-2m\delta$,
\bes
\left\|L_K^{1/2}\wh \prod(v_{k-\bar t:k-1})\Psi_{t-k+1,D_{v_k}}\right\|_K\lesssim\aaa^{\f{1}{2}} W_+'(0)^{\f{1}{2}}\bar t^{\f{1}{2}}\left(\log \f{2}{\delta}\right)^2\f{1}{n}.
\ees
Based on the above estimates, we finally obtain, with confidence at least $1-\delta$,
\bes
\left\|\huaT_{2,t}^B\right\|_{L_{\rho_\huaX}^2}\lesssim\aaa^{\f{3}{2}} W_+'(0)^{\f{3}{2}}\bar t^{\f{1}{2}}(t-2\bar t)\left(\log \f{4m}{\delta}\right)^2\f{1}{n}.
\ees
The proof is complete.
\end{proof}

\subsection{Estimates on $\huaT_{2,t}^{C_1}$}
This section provide an estimate for $\huaT_{2,t}^{C_1}$.
\begin{pro}\label{T2tC1_est}
	Assume \eqref{moment_condition},  \eqref{regularity_ass} holds with $r>\f{1}{2}$, the stepsize $\aaa$ satisfies $0<\aaa\leq\f{1}{\kkk^2}\min\{\f{1}{W_+'(0)},\f{1}{C_W}\}$. If $|D_u|=\f{|D|}{m}=n$, $u\in\huaV$, then for $t,\bar t \in\mbb N_+$, $t\geq2\bar t\geq4$, there holds, for any $0<\delta<1$, with confidence at least $1-\delta$, 
	\bes
	\left\|\huaT_{2,t}^{C_1}\right\|_{L_{\rho_\huaX}^2}\lesssim \left(\log\f{4m}{\delta}\right)\aaa (t-2\bar t)\left(\sqrt{m}\gamma_{\bm{M}}^{\bar t}\wedge 1\right)\f{1}{\sqrt{n}}.
	\ees
\end{pro}
\begin{proof}
We know from the previous analysis that
\bes
\begin{aligned}
\huaT_{2,t}^{C_1}=&\aaa W_+'(0)\sum_{k=2\bar t+1}^{t+1}\sum_{v_k}\sum_{v_1,...,v_{k-\bar t-1}}\prod_{s=1}^{k-\bar t-1}\big[\bm{M}\big]_{v_{s-1}v_s}\\
&\left(\Big[\bm{M}^{\bar t+1}\Big]_{v_{k-\bar t-1}v_k}-\f{1}{m}\right)\wh \prod(v_{1:k-\bar t-1})(I-\aaa  W_+'(0) L_K)^{\bar t}\Psi_{t-k+1,D_{v_k}}.
\end{aligned}
\ees
Applying Lemma \ref{Nedic_lemma} yields 
\bes
\sum_{v_k}\left|\Big[\bm{M}^{\bar t+1}\Big]_{v_{k-\bar t-1}v_k}-\f{1}{m}\right|\leq 2(\sqrt{m}\gamma_{\bm{M}}^{\bar t}\wedge1).
\ees
Additionally, considering the fact 
\bes
\sum_{v_1,...,v_{k-\bar t-1}}\prod_{s=1}^{k-\bar t-1}\big[\bm{M}\big]_{v_{s-1}v_s}=1 
\ees
which follows from double stochasticity of the matrix $\bm M$,
as well as the following basic inequalities
\bes
\left\|L_K^{1/2}\wh \prod(v_{1:k-\bar t-1})\right\|\leq 2\kkk^2, \ \left\|(I-\aaa  W_+'(0) L_K)^{\bar t}\right\|\leq1,
\ees
after taking $L_{\rho_\huaX}^2$-norms on both sides, we have, with probability at least $1-2m\delta$,
\bes
\begin{aligned}
	\left\|\huaT_{2,t}^{C_1}\right\|_{L_{\rho_\huaX}^2}\lesssim&\aaa\sum_{k=2\bar t+1}^{t+1}\sum_{v_k}\sum_{v_1,...,v_{k-\bar t-1}}\prod_{s=1}^{k-\bar t-1}\big[\bm{M}\big]_{v_{s-1}v_s}\\
	&\left|\Big[\bm{M}^{\bar t+1}\Big]_{v_{k-\bar t-1}v_k}-\f{1}{m}\right|\sup_{t', v}\left\{\left\|\Psi_{t',D_{v}}\right\|_K\right\}\\
	\lesssim&\left(\log\f{2}{\delta}\right)\aaa (t-2\bar t)\left(\sqrt{m}\gamma_{\bm{M}}^{\bar t}\wedge 1\right)\f{1}{\sqrt{n}}.
\end{aligned}
\ees
Re-scaling $\delta$ finally yields that, with confidence at least $1-\delta$,
\bes
\left\|\huaT_{2,t}^{C_1}\right\|_{L_{\rho_\huaX}^2}\lesssim \left(\log\f{4m}{\delta}\right)\aaa (t-2\bar t)\left(\sqrt{m}\gamma_{\bm{M}}^{\bar t}\wedge 1\right)\f{1}{\sqrt{n}}.
\ees
The proof is complete.
\end{proof}

\subsection{Estimates on $\huaT_{2,t}^{C_2}$}
\subsubsection{Preliminary representations}
In this subsection, we estimate the term 
\bes
\begin{aligned}
	\huaT_{2,t}^{C_2}=&\aaa W_+'(0)\sum_{k=2\bar t+1}^{t+1}\sum_{v_1,...,v_{k-\bar t-1}}\left(\prod_{s=1}^{k-\bar t-1}\big[\bm{M}\big]_{v_{s-1}v_s}-\f{1}{m^{k-\bar t-1}}\right)\\
	&\wh \prod(v_{1:k-\bar t-1})(I-\aaa  W_+'(0) L_K)^{\bar t}\Psi_{t-k+1,D}.
\end{aligned}
\ees
Before coming to the main estimate, we introduce an auxiliary sequence $\{g_{s,D_u}\}_{u\in\huaV}$ by iteration
\bes
\begin{aligned}
g_{s+1,D_u}=&\sum_v\big[\bm{M}\big]_{uv}\left(I-\aaa  W_+'(0) L_{K,D_v}\right)g_{s,D_v}\\
=&\sum_{v_1,\cdots,v_s}\prod_{\ell=1}^s\big[\bm{M}\big]_{v_{\ell-1}v_\ell}\prod(v_{1:s})g_{1,D_{v_s}},
\end{aligned}
\ees
with  $g_{1,D_u}=g\in\huaH_K$ and the index notation $v_0=u$. On the other hand, we introduce another auxiliary sequence  $\{\ww g_{s,D_u}\}_{u\in\huaV}$ with $\ww g_{1,D_u}=g$ as
\bes
\begin{aligned}
\ww g_{s+1,D_u}=&\sum_v\f{1}{m}\left(I-\aaa W_+'(0) L_{K,D_v}\right)\ww g_{s,D_v}=\sum_{v_1,\cdots, v_s}\f{1}{m^s}\prod(v_{1:s})\ww g_{1,D_{v_s}}.
\end{aligned}
\ees
We know from the above definition of $\{g_{s,D_u}\}_{u\in\huaV}$ and $\{\ww g_{s,D_u}\}_{u\in\huaV}$ that
\bea
\begin{aligned}
	\left\|g_{s+1,D_u}-\ww g_{s+1,D_u}\right\|_{L_{\rho_\huaX}^2}=&\left\|L_K^{1/2}(g_{s+1,D_u}-\ww g_{s+1,D_u})\right\|_K\\
	=&\left\|\sum_{v_1,\cdots,v_s}\left(\prod_{\ell=1}^s\big[\bm{M}\big]_{v_{\ell-1}v_\ell}-\f{1}{m^s}\right)L_K^{1/2}\prod(v_{1:s})g\right\|_K. \label{difference}
\end{aligned}
\eea
Define another sequence $\{\wh g_s\}$ starting from $\wh g_1\in\huaH_K$, with initial value $\wh g_1=\ww g_{1,D_u}=g_{1,D_u}=g$, $u=1,2,...,m$, by
\bes
\wh g_{s+1}=(I-\aaa  W_+'(0) L_K)^sg, \ s\in \mathbb N_+.
\ees
We know that $$\ww g_{s,D_u}=\wh g_s, \  s\in \mathbb N_+.$$ Then it follows that
\bes
	&&\hspace{-0.7cm}g_{s+1,D_u}=\sum_v\big[\bm{M}\big]_{uv}\Big[\left(I-\aaa  W_+'(0) L_{K,D}\right)g_{s,D_v}+\aaa W_+'(0)\left(L_{K,D}-L_{K,D_v}\right)g_{s,D_v}\Big]\\
	&&\hspace{-0.7cm}=(I-\aaa  W_+'(0) L_{K,D})^sg+\aaa W_+'(0)\sum_{k=1}^s\sum_v\big[\bm{M}^{s-k+1}\big]_{uv}(I-\aaa  W_+'(0) L_{K,D})^{s-k}(L_{K,D}-L_{K,D_v})g_{k,D_v},
\ees
which implies that
\bea
	g_{s+1,D_u}-\ww g_{s+1,D_u}=\aaa W_+'(0)\sum_{k=1}^s\sum_v\big[\bm{M}^{s-k+1}\big]_{uv}(I-\aaa  W_+'(0) L_{K,D})^{s-k}(L_{K,D}-L_{K,D_v})g_{k,D_v}.  \label{cha1}
\eea
Denote the average function $\bar g_s=\f{1}{m}\sum_vg_{s,D_v}$. We know from the structure of \eqref{cha1} that
\bea
\bar g_{s+1}-\wh g_{s+1}=\aaa W_+'(0)\sum_{k=1}^s\f{1}{m}\sum_v(I-\aaa  W_+'(0) L_{K,D})^{s-k}(L_{K,D}-L_{K,D_v})g_{k,D_v}.  \label{cha2}
\eea
Noting that
\bes
g_{s+1,D_u}-\ww g_{s+1,D_u}=(g_{s+1,D_u}-\bar g_{s+1})+(\bar g_{s+1}-\ww g_{s+1,D_u}),
\ees
we have 
\bea
\left\|g_{s+1,D_u}-\ww g_{s+1,D_u}\right\|_{L_{\rho_\huaX}^2}\leq\left\|L_K^{1/2}(g_{s+1,D_u}-\bar g_{s+1})\right\|_K+\left\|L_K^{1/2}(\bar g_{s+1}-\ww g_{s+1,D_u})\right\|_K. \label{auxiliary_core}
\eea
For the first term of \eqref{auxiliary_core}, subtraction between \eqref{cha1} and \eqref{cha2} yields that
\bea
\begin{aligned}
\left\|L_K^{1/2}(g_{s+1,D_u}-\bar g_{s+1})\right\|_K\leq& \aaa W_+'(0)\sum_{k=1}^s\sum_v\left|\Big[\bm{M}^{s-k+1}\Big]_{uv}-\f{1}{m}\right|\\
&\times\left\|L_K^{1/2}(I-\aaa  W_+'(0) L_{K,D})^{s-k}(L_{K,D}-L_{K,D_v})\right\|\|g_{k,D_v}\|_K=:\huaH_s^A. \label{lk_half1}
\end{aligned}
\eea
It is easy to see that, for any $u\in\huaV$,
\bes
\|g_{s+1,D_u}\|_K\leq\sum_v\muv\left\|(I-\aaa  W_+'(0) L_{K,D})g_{s,D_v}\right\|_K\leq\sum_v\muv\left\|g_{s,D_v}\right\|_K\leq\|g\|_K.
\ees
Using the fact that for any two self-adjoint operators $T_1$, $T_2$, $\|T_1T_2\|=\|T_2T_1\|$, we can decompose  $\|L_K^{1/2}(I-\aaa  W_+'(0) L_{K,D})^{s-k}(L_{K,D}-L_{K,D_v})\|$ as 
\bes
&&\left\|L_K^{1/2}(I-\aaa  W_+'(0) L_{K,D})^{s-k}(L_{K,D}-L_{K,D_v})\right\|=\big\|L_K^{1/2}(\la_1I+L_K)^{-1/2}(\la_1I+L_K)^{1/2}(\la_1I+L_{K,D})^{-1/2}\\
&&(\la_1I+L_{K,D})(I-\aaa  W_+'(0) L_{K,D})^{s-k}(\la_1I+L_{K,D})^{-1/2}(\la_1I+L_K)^{1/2}(\la_1I+L_K)^{-1/2}(L_{K,D}-L_{K,D_v})\big\|,
\ees
which can be further bounded by
\bea
\nono&&\left\|(\la_1I+L_K)^{1/2}(\la_1I+L_{K,D})^{-1/2}\right\|^2\left\|(\la_1I+L_{K,D})(I-\aaa  W_+'(0) L_{K,D})^{s-k}\right\|\\
&&\times\left\|(\la_1I+L_K)^{-1/2}(L_{K,D}-L_{K,D_v})\right\|.  \label{huaH_A_dec}
\eea
Due to the fact that $L_{K,D}=\f{1}{m}\sum_{i}L_{K,D_i}$, we know
\bes
&&\left\|(\la_1I+L_K)^{-1/2}(L_{K,D}-L_{K,D_v})\right\|\\
&&\leq\f{1}{m}\sum_i\left\|(\la_1I+L_K)^{-1/2}(L_{K,D_i}-L_{K})\right\|+\left\|(\la_1I+L_K)^{-1/2}(L_{K}-L_{K,D_v})\right\|.
\ees
For a data set $D$ and a  real number $\la>0$, if we denote the norms
\bes
&&\pd=\left\|(\la I+L_K)^{-1/2}(L_{K}-L_{K,D})\right\|,\\
&&\qd=\left\|(\la I+L_K)(\la I+L_{K,D})^{-1}\right\|,
\ees
then we have
\bes
\left\|(\la_1I+L_K)^{-1/2}(L_{K,D}-L_{K,D_v})\right\|\leq2\max_v\pdvi.
\ees
Therefore, we have
\bes
&&\left\|L_K^{1/2}(I-\aaa  W_+'(0) L_{K,D})^{s-k}(L_{K,D}-L_{K,D_v})\right\|\\
&&\leq2\qdi(\max_v\pdvi) \left\|(\la_1I+L_{K,D})(I-\aaa  W_+'(0) L_{K,D})^{s-k}\right\|. 
\ees
In this position, we consider another decomposition for later use:
\bea
\nono&&\left\|L_K^{1/2}(I-\aaa  W_+'(0) L_{K,D})^{s-k}(L_{K,D}-L_{K,D_v})\right\|\\
&&\leq2(\max_v\pdvi) \left\|L_K^{1/2}(I-\aaa  W_+'(0) L_{K,D})^{s-k}(\la_1I+L_K)^{1/2}\right\|.   \label{decTC1A_use}
\eea
Based on the above facts, we have
\bes
\begin{aligned}
\huaH_s^A\lesssim&\aaa \|g\|_K\qdi(\max_v\pdvi)\sum_{k=1}^s\left\|(\la_1I+L_{K,D})(I-\aaa  W_+'(0) L_{K,D})^{s-k}\right\|\\
&\times\sum_v\left|\big[\bm{M}^{s-k+1}\big]_{uv}-\f{1}{m}\right|.
\end{aligned}
\ees
As a result of Lemma \ref{Nedic_lemma}, we have
\bea
\begin{aligned}
	\huaH_s^A\lesssim&\aaa \|g\|_K\qdi(\max_v\pdvi)\sum_{k=1}^s\left\|(\la_1I+L_{K,D})(I-\aaa  W_+'(0) L_{K,D})^{s-k}\right\|\left(\sqrt{m}\gamma_{\bm{M}}^{s-k+1}\wedge1\right).   \label{HtA_est}
\end{aligned}
\eea
On the other hand, due to the fact that $\f{1}{m}\sum_v(L_{K,D}-L_{K,D_v})=0$, we have
\bes
\bar g_{s+1}-\wh g_{s+1}=\aaa W_+'(0)\sum_{k=2}^s\f{1}{m}\sum_v(I-\aaa  W_+'(0) L_{K,D})^{s-k}(L_{K,D}-L_{K,D_v})(g_{k,D_v}-\bar g_k).
\ees
After taking $L_{\rho_\huaX}^2$ norms, we obtain
\bes
\begin{aligned}
\left\|L_K^{1/2}(\bar g_{s+1}-\wh g_{s+1})\right\|_K\leq&\aaa W_+'(0)\sum_{k=2}^s\f{1}{m}\sum_v\left\|L_K^{1/2}(I-\aaa  W_+'(0) L_{K,D})^{s-k}(\la_2I+L_K)^{1/2}\right\|\\
&\times\left\|(\la_2I+L_K)^{-1/2}(L_{K,D}-L_{K,D_v})\right\|\|g_{k,D_v}-\bar g_k\|_K.
\end{aligned}
\ees
Revisiting the procedures of getting \eqref{lk_half1}, we have
\bes
\begin{aligned}
	\left\|g_{k,D_u}-\bar g_{k}\right\|_K\leq& \aaa W_+'(0)\sum_{\ell=1}^{k-1}\sum_v\left|\big[\bm{M}^{k-\ell}\big]_{uv}-\f{1}{m}\right|\\
	&\times\left\|(I-\aaa  W_+'(0) L_{K,D})^{k-\ell-1}(L_{K,D}-L_{K,D_v})\right\|\|g\|_K. 
\end{aligned}
\ees
Finally, we have
\bea
\beal
\left\|L_K^{1/2}(\bar g_{s+1}-\wh g_{s+1})\right\|_K\lesssim&\aaa^2\big(\max_v\pdvii\big)\big(\max_v\pdviii\big)\huaQ_{D,\la_3}\|g\|_K\\
&\sum_{k=2}^s\sum_{\ell=1}^{k-1}\left\|L_K^{1/2}(I-\aaa  W_+'(0) L_{K,D})^{s-k}(\la_2I+L_K)^{1/2}\right\|\\
&\times\left\|(I-\aaa  W_+'(0) L_{K,D})^{k-\ell-1}(\la_3 I+L_K)^{1/2}\right\|(\sqrt{m}\gamma_{\bm{M}}^{k-\ell}\wedge1)=:\huaH_s^B, \label{HB_def}
\eeal
\eea
with index $k\geq2$. Recalling the fact that $\wh g_s=\ww g_{s,D_u}$, and combining \eqref{lk_half1} and \eqref{HB_def} with \eqref{difference}  (recalling the notation $v_0=u$), we have
\bes
\left\|\sum_{v_1,\cdots,v_s}\left(\prod_{\ell=1}^s\big[\bm{M}\big]_{v_{\ell-1}v_\ell}-\f{1}{m^s}\right)L_K^{1/2}\prod(v_{1:s})g\right\|_K\lesssim\huaH_s^A+\huaH_s^B.
\ees
If we consider the sequences $\{g_{s,D_u}\}$, $\{\ww g_{s,D_u}\}$ with $g=(I-\aaa  W_+'(0) L_K)^{\bar t}\Psi_{t-k+1,D}$, and the corresponding sequences $\{\huaH_s^A\}$ and $\{\huaH_s^B\}$ defined in \eqref{lk_half1} and \eqref{HB_def} based on $\{g_{s,D_u}\}$, $\{\ww g_{s,D_u}\}$, we can bound $\huaT_{2,t}^{C_2}$ as
\bes
\huaT_{2,t}^{C_2}\lesssim \huaT_{2,t}^{C_2,A}+\huaT_{2,t}^{C_2,B}
\ees
where 
\bea
\huaT_{2,t}^{C_2,A}=\aaa\sum_{k=2\bar t+1}^{t+1}\huaH_{k-\bar t-1}^A, \ \huaT_{2,t}^{C_2,B}=\aaa\sum_{k=2\bar t+1}^{t+1}\huaH_{k-\bar t-1}^B, \label{TC2_decomposition}
\eea
with $\huaH_s^A$ and $\huaH_s^B$ defined in \eqref{lk_half1} and \eqref{HB_def}, respectively. With these preparations in place, the following subsections will present core estimates for $\huaT_{2,t}^{C_2}$ by estimating $\huaT_{2,t}^{C_2,A}$ and $\huaT_{2,t}^{C_2,B}$.
\subsubsection{Estimates on $\huaT_{2,t}^{C_2,A}$}
The core estimate on $\huaT_{2,t}^{C_2,A}$ is mainly contained in the following proposition.
\begin{pro}\label{T2tC2A_est}
	Assume \eqref{moment_condition}, \eqref{capacity_ass} with $0<s\leq1$, \eqref{regularity_ass} holds with $r>\f{1}{2}$, the stepsize $\aaa$ satisfies $\aaa\leq\f{1}{\kkk^2}\min\{\f{1}{W_+'(0)},\f{1}{C_W}\}$. If $|D_u|=\f{|D|}{m}=n$, $u\in\huaV$, then for $t,\bar t\in\mbb N_+$, $t\geq2\bar t\geq4$,  there holds, for any $0<\delta<1$, with confidence at least $1-\delta$, 
	\bes
	\beal
	\left\|\huaT_{2,t}^{C_2,A}\right\|_{L_{\rho_\huaX}^2}\lesssim_\delta(\aaa\bar t\vee 1)^{\f{1}{2}}\left[(\aaa\bar t\vee 1)^{2}+\aaa t\sqrt{m}\gamma_{\bm{M}}^{\bar t}\right]\aaa t\f{1}{\sqrt{n}}\f{1}{\sqrt{|D|}}\left(\log\f{32}{\delta}\right)^4.
	\eeal
	\ees
	\end{pro}
\begin{proof}
According to the representation of $\huaT_{2,t}^{C_2,A}$ in \eqref{TC2_decomposition}, by taking the $L_{\rho_\huaX}^2$ norm, we have
\bes
\beal
\left\|\huaT_{2,t}^{C_2,A}\right\|_{L_{\rho_\huaX}^2}\lesssim&\aaa^2\sum_{k=2\bar t+1}^{t+1}\max_{t'}\|(I-\aaa  W_+'(0) L_K)^{\bar t}\Psi_{t',D}\|_K\\
&\times\sum_{\ell=1}^{k-\bar t-1}\left\|L_K^{1/2}(I-\aaa  W_+'(0) L_{K,D})^{k-\bar t-\ell-1}(L_{K,D}-L_{K,D_v})\right\|(\sqrt{m}\gamma_{\bm{M}}^{k-\bar t-\ell}\wedge1).
\eeal
\ees
By utilizing the estimates in \eqref{decTC1A_use} and \eqref{HtA_est}, we can perform the following decomposition
\bes
&&\aaa\sum_{\ell=1}^{k-\bar t-1}\left\|L_K^{1/2}(I-\aaa  W_+'(0) L_{K,D})^{k-\bar t-\ell-1}(L_{K,D}-L_{K,D_v})\right\|(\sqrt{m}\gamma_{\bm{M}}^{k-\bar t-\ell}\wedge1)\\
&&\leq2(\max_v\pdvi)\aaa\sum_{\ell=1}^{k-2\bar t-1} \left\|L_K^{1/2}(I-\aaa  W_+'(0) L_{K,D})^{k-\bar t-\ell-1}(\la_1I+L_K)^{1/2}\right\|(\sqrt{m}\gamma_{\bm{M}}^{k-\bar t-\ell}\wedge1)\\
&&+\huaQ_{D,\la_1}(\max _v\pdvi)\aaa\sum_{\ell=k-2\bar t}^{k-\bar t-1}\left\|(\la_1I+L_{K,D})(I-\aaa  W_+'(0) L_{K,D})^{k-\bar t-\ell-1}\right\|.
\ees
 We note that the first term of the right hand side of the above inequality can be bounded by 
\bes
2(\max_v\pdvi)\|L_K^{1/2}(\la_1 I+L_K)^{1/2}\|\aaa t\sqrt{m}\gamma_{\bm{M}}^{\bar t},
\ees
which can be further bounded by 
\bes
2(\max_v\pdvi)(\la_1+1)\aaa t\sqrt{m}\gamma_{\bm{M}}^{\bar t}
\ees
up to absolute positive constants. Meanwhile,
the second term can be bounded by 
\bes
&&\huaQ_{D,\la_1}(\max _v\pdvi)\sum_{\ell=k-2\bar t}^{k-\bar t-1}\Big(\aaa\la_1\left\|(I-\aaa  W_+'(0) L_{K,D})^{k-\bar t-\ell-1}\right\|\\
&&+\aaa\left\|L_{K,D}(I-\aaa  W_+'(0) L_{K,D})^{k-\bar t-\ell-1}\right\|\Big)\\
&&\lesssim\huaQ_{D,\la_1}(\max _v\pdvi)\left(\aaa \la_1\bar t+\log \bar t\right).
\ees
Then we know 
\bes
&&\aaa\sum_{\ell=1}^{k-\bar t-1}\left\|L_K^{1/2}(I-\aaa  W_+'(0) L_{K,D})^{k-\bar t-\ell-1}(L_{K,D}-L_{K,D_v})\right\|(\sqrt{m}\gamma_{\bm{M}}^{k-\bar t-\ell}\wedge1)\\
&&\lesssim (\max_v\pdvi)(\la_1+1)\aaa t\sqrt{m}\gamma_{\bm{M}}^{\bar t}+\huaQ_{D,\la_1}(\max _v\pdvi)\log \bar t(1\vee\la_1\aaa\bar t).
\ees
Accordingly, we have
\bea
\nono&&\left\|\huaT_{2,t}^{C_2,A}\right\|_{L_{\rho_\huaX}^2}\lesssim\aaa t\left(\max_{t'}\|\Psi_{t',D}\|_K\right)\Big[(\max_v\pdvi)(\la_1+1)\aaa t\sqrt{m}\gamma_{\bm{M}}^{\bar t}\\
&&\quad\quad\quad\quad\quad\quad\quad+\huaQ_{D,\la_1}(\max _v\pdvi)\log \bar t(1\vee\la_1\aaa\bar t)\Big].  \label{T2tC2A_eq1}  
\eea
 Let $\la_1=(\aaa\bar t\vee 1)^{-1}$, the above inequality can be simplified as  
\bes
\left\|\huaT_{2,t}^{C_2,A}\right\|_{L_{\rho_\huaX}^2}\lesssim\left(\max_{t'}\|\Psi_{t',D}\|_K\right)\Big[(\max_v\pdvi)\aaa t\sqrt{m}\gamma_{\bm{M}}^{\bar t}+\huaQ_{D,\la_1}(\max _v\pdvi)\Big]\aaa t\log \bar t.
\ees
For a data set $D$ and a real number $\la>0$,  we denote 
\bes
\huaA_{D,\la}=\f{2\kkk}{\sqrt{|D|}}\left(\f{\kkk}{\sqrt{|D|\la}}+\sqrt{\huaN(\la)}\right).
\ees
We know from \cite{glz2017} and \cite{lgz2017} that, with probability $1-\delta$, 
\bea
&&\huaP_{D,\la_1}\leq\huaA_{D,\la_1}\left(\log\f{2}{\delta}\right), \label{PD_def}\\
&&\huaQ_{D,\la_1}\leq2\left[\left(\f{\huaA_{D,\la_1}\log\f{2}{\delta}}{\sqrt{\la_1}}\right)^2+1\right].  \label{QD_def}
\eea
Hence, for $|D_1|=|D_2|=\cdots=|D_m|=n$, with the capacity condition \eqref{capacity_ass} at hand, we obtain that, with probability $1-\delta$, the following inequalities hold simultaneously
\bes &&\pdvi\leq\huaA_{D_v,\la_1}\left(\log\f{4m}{\delta}\right)\lesssim\left[\f{(\aaa\bar t\vee1)^{\f{1}{2}}}{n}+\f{(\aaa\bar t\vee 1)^{\f{s}{2}}}{\sqrt{n}}\right]\left(\log\f{4m}{\delta}\right), \ v=1,2,...,m.\\
&&\huaQ_{D,\la_1}\lesssim(\aaa \bar t\vee 1)^2\left(\log\f{4}{\delta}\right)^2.
\ees
On the other hand, we also note that, with probability $1-\delta$, 
\bes
\sup_{t'}\|\Psi_{t',D}\|_K\lesssim\f{1}{\sqrt{|D|}}\left(\log\f{2}{\delta}\right).
\ees
Based on the above estimates, when the local sample size satisfies $|D_1|=\cdots=|D_m|=n$, we finally obtain that, with probability $1-\delta$,
\bes
\left\|\huaT_{2,t}^{C_2,A}\right\|_{L_{\rho_\huaX}^2}\lesssim\left[\f{(\aaa\bar t\vee1)^{\f{1}{2}}}{n}+\f{(\aaa\bar t\vee 1)^{\f{s}{2}}}{\sqrt{n}}\right]\left((\aaa\bar t\vee 1)^{2}+\aaa t\sqrt{m}\gamma_{\bm{M}}^{\bar t}\right)\aaa t\log \bar t\f{1}{\sqrt{|D|}}\left(\log\f{32}{\delta}\right)^4\left(\log m\right).
\ees
A further simplification implies that, with probability at least $1-\delta$,
\bes
\beal
\left\|\huaT_{2,t}^{C_2,A}\right\|_{L_{\rho_\huaX}^2}\lesssim&(\aaa\bar t\vee 1)^{\f{1}{2}}\left[(\aaa\bar t\vee 1)^{2}+\aaa t\sqrt{m}\gamma_{\bm{M}}^{\bar t}\right]\aaa t\log \bar t\f{1}{\sqrt{n}}\f{1}{\sqrt{|D|}}\left(\log\f{32}{\delta}\right)^4\left(\log m\right)\\
\lesssim_{\delta}&(\aaa\bar t\vee 1)^{\f{1}{2}}\left[(\aaa\bar t\vee 1)^{2}+\aaa t\sqrt{m}\gamma_{\bm{M}}^{\bar t}\right]\aaa t\f{1}{\sqrt{n}}\f{1}{\sqrt{|D|}}\left(\log\f{32}{\delta}\right)^4,
\eeal
\ees
which completes the proof.
\end{proof}

\subsubsection{Estimates on  $\huaT_{2,t}^{C_2,B}$}
The core estimate on $\huaT_{2,t}^{C_2,B}$ is mainly contained in the following proposition.
\begin{pro}\label{T2tC2B_est}
	Assume \eqref{moment_condition}, \eqref{capacity_ass} with $0<s\leq1$, \eqref{regularity_ass} holds with $r>\f{1}{2}$, the stepsize $\aaa$ satisfies $\aaa\leq\f{1}{\kkk^2}\min\{\f{1}{W_+'(0)},\f{1}{C_W}\}$. If $|D_u|=\f{|D|}{m}=n$, $u\in\huaV$, then for $t,\bar t\in\mbb N_+$, $t\geq2\bar t\geq4$,  there holds, for any $0<\delta<1$, with confidence at least $1-\delta$,
	\bes
	\left\|\huaT_{2,t}^{C_2,B}\right\|_{L_{\rho_\huaX}^2}\lesssim_{\delta}\left(\f{(\aaa t)^{\f{s}{2}}}{\sqrt{n}}+\f{(\aaa t)^{\f{1}{2}}}{n}\right)\f{1}{\sqrt{n}}\f{1}{\sqrt{|D|}}\aaa t\left(\aaa t\sqrt{m}\gamma_{\bm{M}}^{\bar t}+\aaa\bar t\right)\left(\log\f{16}{\delta}\right)^4.
	\ees
\end{pro}
\begin{proof}
According to previous analysis of getting \eqref{HB_def}, we know
\bes
\beal
\left\|\huaT_{2,t}^{C_2,B}\right\|_{L_{\rho_\huaX}^2}\lesssim&\aaa^3\sum_{k=2\bar t+1}^{t+1}\left(\max_v\huaP_{D_v,\la_2}\right)\left(\max_v\huaP_{D_v,\la_3}\right)\huaQ_{D,\la_3}\max_{t'}\|\Psi_{t',D}\|_K\\
&\times\sum_{s=2}^{k-\bar t-1}\left\|L_K^{1/2}(I-\aaa  W_+'(0) L_{K,D})^{k-\bar t-s-1}(\la_2I+L_K)^{1/2}\right\|\\
&\times\sum_{\ell=1}^{s-1}\left\|(I-\aaa  W_+'(0) L_{K,D})^{s-\ell-1}(\la_3I+L_K)^{1/2}\right\|(\sqrt{m}\gamma_{\bm{M}}^{s-\ell}\wedge1).
\eeal
\ees
For $2\bar t+1\leq k\leq t+1$, $2\leq s\leq \bar t$, we have
\bes
&&\aaa\sum_{\ell=1}^{s-1}\left\|(I-\aaa  W_+'(0) L_{K,D})^{s-\ell-1}(\la_3I+L_K)^{1/2}\right\|(\sqrt{m}\gamma_{\bm{M}}^{s-\ell}\wedge1)\\
&&\leq\aaa\sum_{\ell=1}^{\bar t}\left\|(\la_3I+L_K)^{1/2}\right\|(\sqrt{m}\gamma_{\bm{M}}^{s-\ell}\wedge1)\leq\aaa\bar t\left\|(\la_3I+L_K)^{1/2}\right\|.
\ees
For $\bar t+1\leq s\leq k-\bar t -1$, we have the following estimates
\bes
&&\aaa\sum_{\ell=1}^{s-1}\left\|(I-\aaa  W_+'(0) L_{K,D})^{s-\ell-1}(\la_3I+L_K)^{1/2}\right\|(\sqrt{m}\gamma_{\bm{M}}^{s-\ell}\wedge1)\\
&&\leq\left\|(\la_3I+L_K)^{1/2}\right\|\aaa\sum_{\ell=1}^{s-\bar t}(\sqrt{m}\gamma_{\bm{M}}^{s-\ell}\wedge1)+\aaa\sum_{\ell=s-\bar t+1}^{s-1}\left\|(I-\aaa  W_+'(0) L_{K,D})^{s-\ell-1}(\la_3I+L_K)^{1/2}\right\|\\
&&\leq\left\|(\la_3I+L_K)^{1/2}\right\|\left(\aaa t\sqrt{m}\gamma_{\bm{M}}^{\bar t}+\aaa\bar t\right).
\ees
From the procedures in estimating $\huaT_{2,t}^{C_2,A}$, we have
\bes
&&\aaa\sum_{s=2}^{k-\bar t-1}\left\|L_K^{1/2}(I-\aaa  W_+'(0) L_{K,D})^{k-\bar t-s-1}(\la_2I+L_K)^{1/2}\right\|\\
&&\leq\aaa\sum_{s=2}^{k-\bar t-1}\left\|(\la_2I+L_K)(I-\aaa  W_+'(0) L_{K,D})^{k-\bar t-s-1}\right\|\lesssim\log \bar t(1\vee\la_2\aaa\bar t).
\ees
Then it follows that
\bes
\beal
\left\|\huaT_{2,t}^{C_2,B}\right\|_{L_{\rho_\huaX}^2}\lesssim&\left(\max_v\huaP_{D_v,\la_2}\right)\left(\max_v\huaP_{D_v,\la_3}\right)\huaQ_{D,\la_3}\max_{t'}\|\Psi_{t',D}\|_K\\
&\times \left\|(\la_3I+L_K)^{1/2}\right\|\aaa t\left(\aaa t\sqrt{m}\gamma_{\bm{M}}^{\bar t}+\aaa\bar t\right)\log \bar t(1\vee\la_2\aaa\bar t).
\eeal
\ees
From \eqref{PD_def}, we know, with probability $1-\delta$, the followings hold simultaneously:
\bes
&&\huaP_{D_v,\la_2}\lesssim\huaA_{D_v,\la_2}\log\f{4m}{\delta}, \ v\in\huaV,\\
&&\huaP_{D_v,\la_3}\lesssim\huaA_{D_v,\la_3}\log\f{4m}{\delta}, \ v\in\huaV.
\ees
Also recall, with probability $1-\delta$, 
\bes
\huaQ_{D,\la_3}\leq2\left[\left(\f{\huaA_{D,\la_3}\log\f{2}{\delta}}{\sqrt{\la_3}}\right)^2+1\right], \ \max_{t'}\|\Psi_{t',D}\|_K\lesssim\f{1}{\sqrt{|D|}}\left(\log\f{4}{\delta}\right).
\ees
We obtain, with probability $1-\delta$,
\bea
\beal
\left\|\huaT_{2,t}^{C_2,B}\right\|_{L_{\rho_\huaX}^2}\lesssim&\left(\max_v\huaA_{D_v,\la_2}\right)\left(\max_v\huaA_{D_v,\la_3}\right)\left[\left(\f{\huaA_{D,\la_3}}{\sqrt{\la_3}}\right)^2+1\right]\left(\log m\right)^2\left(\log\f{16}{\delta}\right)^4\\
&\left\|(\la_3I+L_K)^{1/2}\right\|\f{1}{\sqrt{|D|}}\aaa t\left(\aaa t\sqrt{m}\gamma_{\bm{M}}^{\bar t}+\aaa\bar t\right)\log \bar t(1\vee\la_2\aaa\bar t). \label{T2tC2B_eq1}
\eeal
\eea
When $\la_2=(\aaa t)^{-1}$, $\la_3=\kkk^2$, $|D_1|=\cdots=|D_m|=n$,  we know $\|(\la_3I+L_K)^{1/2}\|\lesssim1$, $\max_v\huaA_{D_v,\la_2}\lesssim(\f{(\aaa t)^{\f{s}{2}}}{\sqrt{n}}+\f{(\aaa t)^{\f{1}{2}}}{n})$, $\max_v\huaA_{D_v,\la_3}\lesssim\f{1}{\sqrt{n}}$, and $\left(\f{\huaA_{D,\la_3}}{\sqrt{\la_3}}\right)^2+1\lesssim1$. Hence, we have, with probability at least $1-\delta$,
\bes
\left\|\huaT_{2,t}^{C_2,B}\right\|_{L_{\rho_\huaX}^2}\lesssim_{\delta}\left(\f{(\aaa t)^{\f{s}{2}}}{\sqrt{n}}+\f{(\aaa t)^{\f{1}{2}}}{n}\right)\f{1}{\sqrt{n}}\f{1}{\sqrt{|D|}}\aaa t\left(\aaa t\sqrt{m}\gamma_{\bm{M}}^{\bar t}+\aaa\bar t\right)\left(\log\f{16}{\delta}\right)^4.
\ees
The proof is complete.
\end{proof}

\section{Estimates on $\huaT_{3,t}$}

Before coming to  estimate $\huaT_{3,t}$, we need to provide a bound for $\{E_{t,D_v}\}_{v\in\huaV}$ in the RKHS norm.  The following result is used to the RKHS norm bound for the sequence $\{E_{t,D_v}\}_{v\in\huaV}$.
\begin{pro}\label{E_{t,D_v}_est}
	Assume \eqref{moment_condition} holds and the windowing function $W$ satisfies basic conditions \eqref{window_ass1} and \eqref{window_ass2}. If the stepsize $\aaa$ satisfies  $0<\aaa \leq\f{1}{\kkk^2}\min\{ \frac{1}{C_W},\f{1}{W_+'(0)}\}$, then we have,  for any $0<\delta<1$, with probability at least $1-\delta$, the following RKHS norm bounds for the sequence $\{f_{t,D_u}\}_{u\in\huaV}$ and $\{E_{t,D_u}\}_{u\in\huaV}$ hold:
	\bes
	&&\left\| f_{t,D_u}\right\|_{ K}\leq \ww M_\rho\sqrt{C_W\aaa t}\log\f{|D|}{\delta},\\
	&&\left\|E_{t,D_v}\right\|_{ K}\leq c_p\kkk \ww M_\rho^{2 p+1} t^{\frac{2 p+1}{2}} \sigma^{-2 p}\left(\log\f{|D|}{\delta}\right)^{2p+1}.
	\ees
	where $\ww M_\rho=4M_\rho+5M_\rho\sqrt{2B_\rho}$.
\end{pro}

\begin{proof}
 We begin by proving the first inequality.
	For $t=1$, the result holds obviously.  According to the convexity of $\|\cdot\|_{ K}^2$ and double stochasticity of the communication matrix $\bm{M}$, we have
	\bes
	\|{f}_{t+1,D_{u}}\|_{ K}^2=\left\|\sum_{v}\big[\bm{M}\big]_{uv}\phi_{t,D_v}\right\|_{ K}^2\leq\sum_{v}\big[\bm{M}\big]_{uv}\left\|\phi_{t,D_v}\right\|_{ K}^2.
	\ees
	According to our algorithm structure, we know	
	\bes
	\begin{aligned}
		\left\|\phi_{t,D_v}\right\|_{ K}^2=&\| f_{t,D_v}\|_{ K}^2-\f{2\aaa}{|D_v|}\sum_{(x,y)\in D_v}W'\left(\f{\xi_{t,D_v}^2(z)}{\sigma^2}\right)\xi_{t,D_v}(z) f_{t,D_v}(x)\\
		&+\f{\aaa^2}{|D_v|^2}\left\|\sum_{(x,y)\in D_v}W'\left(\f{\xi_{t,D_v}^2(z)}{\sigma^2}\right)\xi_{t,D_v}(z) K_{x}\right\|_{ K}^2\\
		\leq&\| f_{t,D_v}\|_{ K}^2-\f{2\aaa}{|D_v|}\sum_{(x,y)\in D_v}W'\left(\f{\xi_{t,D_v}^2(z)}{\sigma^2}\right)\xi_{t,D_v}(z) f_{t,D_v}(x)\\
		&+\f{\aaa^2\kkk^2}{|D_v|}\sum_{(x,y)\in D_v}W'\left(\f{\xi_{t,D_v}^2(z)}{\sigma^2}\right)^2(\xi_{t,D_v}(z))^2\\
		=&\| f_{t,D_v}\|_{ K}^2+\f{\aaa}{|D_v|}\sum_{(x,y)\in D_v}P_{D_v}(x,y),
	\end{aligned}
	\ees	
	where $P_{D_v}(x,y)$, $z=(x,y)\in D_v$, $v\in\huaV$ is defined as
	$$
	\begin{aligned}
		P_{D_v}(x,y)= & {\left[\aaa \kkk^2 W'\left(\frac{\xi_{t,D_v}^2(z)}{\sigma^2}\right)^2-2 W'\left(\frac{\xi_{t,D_v}^2(z)}{\sigma^2}\right)\right]\left({f}_{t, D_v}(x)\right)^2 } \\
		& +2\left[W'\left(\frac{\xi_{t,D_v}^2(z)}{\sigma^2}\right)-\aaa\kkk^2 W'\left(\frac{\xi_{t,D_v}^2(z)}{\sigma^2}\right)^2\right]y {f}_{t,D_v}(x)+\aaa\kkk^2 W'\left(\frac{\xi_{t,D_v}^2(z)}{\sigma^2}\right)^2y^2.
	\end{aligned}
	$$
	The condition $0<\aaa \leq\f{1}{\kkk^2}\min\{ \frac{1}{C_W},\f{1}{W_+'(0)}\}$ implies $\aaa\kkk^2 W'\left(\frac{\xi_{t,D_v}^2(z)}{\sigma^2}\right)^2-2 W'\left(\frac{\xi_{t,D_v}^2(z)}{\sigma^2}\right)<0$. Also note that, by setting random variables $\zeta_i=|y_i|-\mbb E|y_i|$, according to the condition \eqref{moment_condition} and Lemma \ref{unbddnoise_lem_R}, we know, with probability $1-\delta$, $|y_i|\leq\ww M_\rho\log\f{1}{\delta}$. Hence, with probability $1-\delta$,
	$\sup_{i\in\{1,2,...,|D|\}}|y_i|\leq \ww M_\rho\log\f{|D|}{\delta}$, with $\ww M_\rho=4M_\rho+5M_\rho\sqrt{2B_\rho}$. By the property of quadratic function, we have, for any $(x,y)\in D_v$, $v\in\huaV$, with probability at least $1-\delta$,
	\bea
	\begin{aligned}
		P_{D_v}(x,y) & \leq \aaa\kkk^2 W'\left(\frac{\xi_{t,D_v}^2(z)}{\sigma^2}\right)^2y^2-\frac{\left[W'\left(\frac{\xi_{t,D_v}^2(z)}{\sigma^2}\right)-\aaa\kkk^2 W'\left(\frac{\xi_{t,D_v}^2(z)}{\sigma^2}\right)^2\right]^2y^2}{\aaa \kkk^2W'\left(\frac{\xi_{t,D_v}^2(z)}{\sigma^2}\right)^2-2 W'\left(\frac{\xi_{t,D_v}^2(z)}{\sigma^2}\right)} \\
		& =\frac{W'\left(\frac{\xi_{t,D_v}^2\left(z_i\right)}{\sigma^2}\right)y^2}{2-\aaa\kkk^2 W'\left(\frac{\xi_{t,D_v}^2(z)}{\sigma^2}\right)} \leq  \ww M_\rho^2\left(\log\f{|D|}{\delta}\right)^2 C_W. \label{PDv_bdd}
	\end{aligned}
	\eea
	Finally, applying the double stochasticity of $\bm{M}$ again, we have
	\bes
	\begin{aligned}
		\|{f}_{t+1,D_{u}}\|_{ K}^2\leq&\sum_{v}\big[\bm{M}\big]_{uv}\left\|\phi_{t,D_v}\right\|_{ K}^2\\
		\leq&\sum_{v}\big[\bm{M}\big]_{uv}\left[	\| f_{t,D_v}\|_{ K}^2+\f{\aaa}{|D_v|}\sum_{(x,y)\in D_v}P_{D_v}(x,y)\right]\\
		=&\aaa\sum_{k=1}^{t+1}\sum_{v_1,...,v_k}\prod_{s=1}^k\Big[\bm{M}\Big]_{v_{s-1}v_s}\f{1}{|D_{v_k}|}\sum_{(x,y)\in D_{v_k}}P_{D_{v_k}}(x,y).
		\end{aligned}
	\ees	
	Applying \eqref{PDv_bdd} and the double stochasticity of $\bm{M}^k$ which indicates that
	\bes
	\sum_{v_1,...,v_k}\prod_{s=1}^k\Big[\bm{M}\Big]_{v_{s-1}v_s}=1,
	\ees
	 we finally obtain, with probability at least $1-\delta$,
	\bes
		\|{f}_{t+1,D_{u}}\|_{ K}^2\leq\ww M_\rho^2C_W\aaa(t+1)\left(\log\f{|D|}{\delta}\right)^2.
	\ees
	Hence, with probability $1-\delta$,
	\bes
	\left\| f_{t,D_u}\right\|_{ K}\leq \ww M_\rho\sqrt{C_W\aaa t}\log\f{|D|}{\delta}
	\ees holds for each $u\in\huaV$. 
	Now we have, for $v\in\huaV$ and $z=(x,y)\in D_v$, with probability at least $1-\delta$,
		\bes
	&\begin{aligned}
		& \left\|\left(W'\left(\frac{\xi_{t,D_v}^2(z)}{\sigma^2}\right)-W_+'(0)\right)\left({f}_{t, D_v}(x)-y\right) K_x\right\|_{K} \\
		\leq & c_p \frac{\kkk\left(|y|+\kkk\left\|{f}_{t, D_v}\right\|_{K}\right)^{2 p+1}}{ \sigma^{2 p}} \leq c_p \frac{\kkk\left(\ww M_\rho\log\f{|D|}{\delta}+\kkk\ww M_\rho\sqrt{C_W\aaa t}\log\f{|D|}{\delta}\right)^{2 p+1}}{ \sigma^{2 p}}\\
		\leq & c_p\kkk \ww M_\rho^{2 p+1} t^{\frac{2 p+1}{2}} \sigma^{-2 p}\left(\log\f{|D|}{\delta}\right)^{2p+1}.
	\end{aligned}
	\ees	
	Finally, recalling the definition of $E_{t,D_v}$ in \eqref{Etv_def}, we have, with probability at least $1-\delta$
	\bes
	\left\|E_{t,D_v}\right\|_{ K}\leq c_p\kkk \ww M_\rho^{2 p+1} t^{\frac{2 p+1}{2}} \sigma^{-2 p}\left(\log\f{|D|}{\delta}\right)^{2p+1}, \ v\in\huaV.
	\ees
	The proof is complete.
\end{proof}

Proposition \ref{E_{t,D_v}_est} implies, with probability at least $1-\delta$,
\bes
\left\|E_{t-k+1,D_v}\right\|_{ K}\lesssim t^{\f{2p+1}{2}}\sigma^{-2p}\left(\log\f{|D|}{\delta}\right)^{2p+1}, \ v\in\huaV.
\ees
With this estimate in hand, we are ready to provide the core estimate of $\huaT_{3,t}$.
\begin{pro}\label{T3t_est}
Assume \eqref{moment_condition} holds and the windowing function $W$ satisfies basic conditions \eqref{window_ass1} and \eqref{window_ass2}.	If the stepsize $\aaa$ satisfies  $0<\aaa \leq\f{1}{\kkk^2}\min\{ \frac{1}{C_W},\f{1}{W_+'(0)}\}$, then we have, for any $0<\delta<1$, with probability at least $1-\delta$,  for each $u\in\huaV$, $t\in\mbb N_+$,
	\bes
	\left\|\huaT_{3,t}\right\|_{L_{\rho_\huaX}^2}\lesssim \aaa^{\f{1}{2}}\left(t^{p+1}\sigma^{-2p}+\f{1}{\sqrt{n}}t^{p+2}\sigma^{-2p}\right)\left(\log\f{4m}{\delta}\right)\left(\log\f{2|D|}{\delta}\right)^{2p+1}.
	\ees
\end{pro}
\begin{proof}
We start from a representation of $\huaT_{3,t}$ 
\bes
\huaT_{3,t}=\aaa\sum_{k=1}^{t+1}\sum_{v_1,v_2,...,v_k}\prod_{s=1}^k\big[\bm{M}\big]_{v_{s-1}v_s}\prod_{w=1}^{k-1}\left(I-\aaa  W_+'(0) L_{ K,D_{v_w}}\right)E_{t-k+1,D_{v_k}},
\ees
and see that it can be decomposed into 
$\huaT_{3,t}^A$ and $\huaT_{3,t}^B$ where
\bes
&&\huaT_{3,t}^A=\aaa\sum_{k=1}^{t+1}\sum_{v_1,v_2,...,v_k}\prod_{s=1}^k\big[\bm{M}\big]_{v_{s-1}v_s}\left(I-\aaa  W_+'(0) L_K\right)^{k-1}E_{t-k+1,D_{v_k}},\\
&&\huaT_{3,t}^B=\aaa\sum_{k=1}^{t+1}\sum_{v_1,v_2,...,v_k}\prod_{s=1}^k\big[\bm{M}\big]_{v_{s-1}v_s}\wh \prod(v_{1:k-1})E_{t-k+1,D_{v_k}}.
\ees
After taking the $L_{\rho_\huaX}^2$ norm, we have
\bes
\left\|\huaT_{3,t}^A\right\|_{L_{\rho_\huaX}^2}\lesssim \aaa\sum_{k=1}^{t+1}\sum_{v_1,v_2,...,v_k}\prod_{s=1}^k\big[\bm{M}\big]_{v_{s-1}v_s}\left\|L_K^{1/2}\left(I-\aaa  W_+'(0) L_K\right)^{k-1}\right\|\|E_{t-k+1,D_{v_k}}\|_K.
\ees
We know from Lemma \ref{operator_lem} that
\bes
\left\|L_K^{1/2}\left(I-\aaa  W_+'(0) L_K\right)^{k-1}\right\|\lesssim\f{\aaa^{-\f{1}{2}}}{\sqrt{k-1}}, \ k\geq2, \ \text{and} \ \|L_K^{1/2}\|\lesssim\aaa^{-\f{1}{2}}.
\ees
Then it follows from Proposition \ref{E_{t,D_v}_est} that, with probability at least $1-\delta$,
\bes
\left\|\huaT_{3,t}^A\right\|_{L_{\rho_\huaX}^2}\lesssim\aaa^{\f{1}{2}}t^{\f{1}{2}}t^{\f{2p+1}{2}}\sigma^{-2p}\lesssim\aaa^{\f{1}{2}}t^{p+1}\sigma^{-2p}\left(\log\f{|D|}{\delta}\right)^{2p+1}.
\ees
For $\huaT_{3,t}^B$, due to the fact that
\bes
\wh \prod(v_{1:k-1})=\aaa W_+'(0)\sum_{\ell=1}^{k-1}\left\{\prod_{w=1}^{\ell-1}\left(I-\aaa  W_+'(0) L_{ K,D_{v_w}}\right)\right\}\left(L_{ K}-L_{ K,D_{v_\ell}}\right)\left(I-\aaa  W_+'(0) L_K\right)^{k-\ell-1},
\ees
we have
\bes
&&	\left\|L_K^{1/2}\wh \prod(v_{1:k-1})E_{t-k+1,D_{v_k}}\right\|_K\\
&&\leq\aaa W_+'(0)\sum_{\ell=1}^{k-1}\left\|\prod_{w=1}^{\ell-1}\left(I-\aaa  W_+'(0) L_{ K,D_{v_w}}\right)\right\|\left\|L_{ K}-L_{ K,D_{v_\ell}}\right\|
	\left\|L_K^{1/2}\left(I-\aaa  W_+'(0) L_K\right)^{k-\ell-1}\right\|\|E_{t-k+1,D_{v_k}}\|_K\\
	&&\leq\aaa W_+'(0)\sup_{v}\|L_K-L_{K,D_v}\|\sum_{\ell=1}^{k-1}\|L_K^{1/2}(I-\aaa  W_+'(0) L_K)^{k-\ell-1}\|\|E_{t-k+1,D_{v_k}}\|_K.
\ees
Then we have, with probability at least $1-\delta$,
\bes
\left\|L_K^{1/2}\wh \prod(v_{1:k-1})E_{t-k+1,D_{v_k}}\right\|_K\lesssim\aaa^{\f{1}{2}}\f{1}{\sqrt{n}}\sqrt{k}t^{\f{2p+1}{2}}\sigma^{-2p}\left(\log\f{4m}{\delta}\right)\left(\log\f{2|D|}{\delta}\right)^{2p+1},
\ees
and accordingly we have, with probability at least $1-\delta$,
\bes
\left\|\huaT_{3,t}^B\right\|_{L_{\rho_\huaX}^2}\lesssim\aaa^{\f{1}{2}} t^{\f{3}{2}}\f{1}{\sqrt{n}}t^{\f{2p+1}{2}}\sigma^{-2p}\left(\log\f{2m}{\delta}\right)\lesssim\aaa^{\f{1}{2}}\f{1}{\sqrt{n}}t^{p+2}\sigma^{-2p}\left(\log\f{4m}{\delta}\right)\left(\log\f{2|D|}{\delta}\right)^{2p+1}.
\ees
Combining the above estimates for $\huaT_{3,t}^A$ and $\huaT_{3,t}^B$, we have, with probability at least $1-\delta$,
\bes
\left\|\huaT_{3,t}\right\|_{L_{\rho_\huaX}^2}\lesssim \aaa^{\f{1}{2}}\left(t^{p+1}\sigma^{-2p}+\f{1}{\sqrt{n}}t^{p+2}\sigma^{-2p}\right)\left(\log\f{4m}{\delta}\right)\left(\log\f{2|D|}{\delta}\right)^{2p+1}.
\ees
The proof is complete.
\end{proof}
\section{Proofs of main theorems}
This section is dedicated to the proof of the main theorems of this paper. Before proceeding with the proof, we will describe some necessary facts. We recall that, in the reference \cite{lz2018}, when the total number $m$ of local machines is equal to $1$, under a noise condition for $\rho(\cdot|x)$ that
\bea
\int_\huaY\left(e^{\f{|y-f_\rho(x)|}{B}}-\f{|y-f_\rho(x)|}{B}-1\right)d\rho(y|x)\leq\f{M^2}{2B^2}, \ x\in\huaX,  \label{moment_LinZhou}
\eea
for constants $M>0$ and $B>0$ (see, for example, \cite{cv2007}), when $0<\aaa\leq\f{1}{\kkk^2W_+'(0)}$, and $r>\f{1}{2}$, the estimator generated from the classical kernel-based algorithm \eqref{classical_kernelGD} is able to achieve the optimal rates with $\left\|\wh f_{t,D}-f_\rho\right\|_{L_{\rho_\huaX}^2}\lesssim|D|^{-\f{r}{2r+s}}$ and $\left\|\wh f_{t,D}-f_\rho\right\|_{K}\lesssim|D|^{-\f{r-\f{1}{2}}{2r+s}}$.
 When $r>\f{1}{2}$, there holds $f_\rho\in\huaH_K$ and $\|f_\rho\|_\infty\leq\kkk\|f_\rho\|_K$, then we know that, in our setting, the moment condition \eqref{moment_condition} is equivalent to condition \eqref{moment_LinZhou}. Hence, the facts mentioned in this paragraph automatically hold. Finally, we summarize the facts into the following lemma.
	\begin{lem}\label{old_GD_rate}
	Assume that \eqref{moment_condition}, \eqref{capacity_ass} and \eqref{regularity_ass} hold for some $r>1/2$ and $0<s\leq1$. If the stepsize satisfies $0<\aaa\leq\f{1}{\kkk^2W_+(0)}$ and $t=|D|^{\f{1}{2r+s}}$, then for any $0<\delta<1$, with confidence at least $1-\delta$, the sequence $\{\wh f_{t,D}\}$ generated from algorithm \eqref{classical_kernelGD} satisfies
	\bes
	&&\left\|\wh f_{t,D}- f_\rho\right\|_{L_{\rho_\huaX}^2}\leq C_*|D|^{-\f{r}{2r+s}}\left(\log\f{12}{\delta}\right)^4,\\
	&&\left\|\wh f_{t,D}- f_\rho\right\|_K\leq C_*|D|^{-\f{r-\f{1}{2}}{2r+s}}\left(\log\f{12}{\delta}\right)^4.
	\ees
	where $C_*$ is an absolute constant independent of data set $D$.
\end{lem}

Equipped with the results derived previously, we are ready to provide the proof of the main results in this paper.

\begin{proof}[Proof of Theorem \ref{mainthm_errorbdd}]
	Combining Proposition \ref{T1t_pro_second}, Proposition \ref{T2tA_est}, Proposition \ref{T2tB_est}, Proposition \ref{T2tC1_est}, Proposition \ref{T2tC2A_est}, Proposition \ref{T2tC2B_est}, Proposition \ref{T3t_est} and the fact that $\|f_{t,D_u}-\wh f_{t,D}\|_{L_{\rho_\huaX}^2}\leq\|\huaT_{1,t}\|_{L_{\rho_\huaX}^2}+\|\huaT_{2,t}\|_{L_{\rho_\huaX}^2}+\|\huaT_{3,t}\|_{L_{\rho_\huaX}^2}$, after re-scaling on $\delta$, we have, with probability at least $1-\delta$,
	\bes
	\beal
	\left\|f_{t,D_u}-\wh f_{t,D}\right\|_{L_{\rho_\huaX}^2}\lesssim_\delta& \left(\log\f{32}{\delta}\right)\aaa^{\f{1}{2}}\left(\f{1}{1-\gamma_{\bm{M}}}\right)\left(\f{\sqrt{m}}{\sqrt{n}}\right)+\aaa^{\f{3}{2}}\bar t^{\f{3}{2}}\left(\log\f{32m}{\delta}\right)^2\f{1}{n}\\
	&+\aaa^{\f{3}{2}} \bar t^{\f{1}{2}}(t-2\bar t)\left(\log \f{32m}{\delta}\right)^2\f{1}{n}+\left(\log\f{32m}{\delta}\right)\aaa (t-2\bar t)\left(\sqrt{m}\gamma_{\bm{M}}^{\bar t}\wedge 1\right)\f{1}{\sqrt{n}}\\
	&+(\aaa\bar t\vee 1)^{\f{1}{2}}\left[(\aaa\bar t\vee 1)^{2}+\aaa t\sqrt{m}\gamma_{\bm{M}}^{\bar t}\right]\aaa t\f{1}{\sqrt{n}}\f{1}{\sqrt{|D|}}\left(\log\f{256}{\delta}\right)^4\\
	&+\left(\f{(\aaa t)^{\f{s}{2}}}{\sqrt{n}}+\f{(\aaa t)^{\f{1}{2}}}{n}\right)\f{1}{\sqrt{n}}\f{1}{\sqrt{|D|}}\aaa t\left(\aaa t\sqrt{m}\gamma_{\bm{M}}^{\bar t}+\aaa\bar t\right)\left(\log\f{128}{\delta}\right)^4\\
	&+\aaa^{\f{1}{2}}\left(t^{p+1}\sigma^{-2p}+\f{1}{\sqrt{n}}t^{p+2}\sigma^{-2p}\right)\left(\log\f{32m}{\delta}\right)\left(\log\f{16|D|}{\delta}\right)^{2p+1}.
	\eeal
	\ees
	Simplification yields the desired bounds.
	\end{proof}

\begin{proof}[Proof of Theorem \ref{mainthm_bdd_with_sigma}]

Applying condition $\bar t=\f{2\log(|D|t)}{1-\gamma_{\bm{M}}}$, we know $t\sqrt{m}\gamma_{\bm{M}}^{\bar t}\leq1$ and $\aaa t\sqrt{m}\gamma_{\bm{M}}^{\bar t}\leq1\vee\aaa\bar t$. Therefore, according to Proposition \ref{T2tC2A_est} and Proposition \ref{T2tC2B_est}, there holds that
\bes
&&\left\|\huaT_{2,t}^{C_2,A}\right\|_{L_{\rho_\huaX}^2}\lesssim_\delta(\aaa\bar t\vee 1)^{\f{5}{2}}\aaa t\f{1}{\sqrt{n}}\f{1}{\sqrt{|D|}}\left(\log\f{32}{\delta}\right)^4,\\
&&\left\|\huaT_{2,t}^{C_2,B}\right\|_{L_{\rho_\huaX}^2}\lesssim_{\delta}\left(\f{(\aaa t)^{\f{s}{2}}}{\sqrt{n}}+\f{(\aaa t)^{\f{1}{2}}}{n}\right)\f{1}{\sqrt{n}}\f{1}{\sqrt{|D|}}\aaa t\left(\aaa\bar t\vee 1\right)\left(\log\f{16}{\delta}\right)^4.
\ees
Then a simplification for Theorem \ref{mainthm_errorbdd} yields that, with probability at least $1-\delta$,
\bes
\beal
\left\|f_{t,D_u}-\wh f_{t,D}\right\|_{L_{\rho_\huaX}^2}\lesssim_\delta&\left(\log\f{512}{\delta}\right)^{4\vee (2p+2)}\Bigg[\aaa^{\f{1}{2}}\left(\f{1}{1-\gamma_{\bm{M}}}\right)\left(\f{\sqrt{m}}{\sqrt{n}}\right)+\aaa^{\f{3}{2}}\bar t^{\f{3}{2}}\f{1}{n}+\aaa^{\f{3}{2}} \bar t^{\f{1}{2}}t\f{1}{n}\\
&+\aaa t\left(\sqrt{m}\gamma_{\bm{M}}^{\bar t}\wedge 1\right)\f{1}{\sqrt{n}}+(\aaa\bar t\vee 1)^{\f{5}{2}}\aaa t\f{1}{\sqrt{n}}\f{1}{\sqrt{|D|}}\\
&+\left(\f{(\aaa t)^{\f{s}{2}+1}}{n}+\f{(\aaa t)^{\f{3}{2}}}{n^{\f{3}{2}}}\right)\f{\aaa \bar t\vee 1}{\sqrt{|D|}}+\aaa^{\f{1}{2}}\left(t^{p+1}\sigma^{-2p}+\f{1}{\sqrt{n}}t^{p+2}\sigma^{-2p}\right)\Bigg].
\eeal
\ees
Then, when $\aaa\cong1$, $\bar t=\f{2\log(|D|t)}{1-\gamma_{\bm{M}}}$, in order to achieve  the target convergence bound, we only require the following estimates hold simultaneously,
\bes
&&\left(\f{1}{1-\gamma_{\bm{M}}}\right)\left(\f{\sqrt{m}}{\sqrt{n}}\right)\leq|D|^{-\f{r}{2r+s}}, \ \bar t^{\f{3}{2}}\f{1}{n}\leq|D|^{-\f{r}{2r+s}}, \ \bar t^{\f{1}{2}}t\f{1}{n}\leq|D|^{-\f{r}{2r+s}}, \ t\left(\sqrt{m}\gamma_{\bm{M}}^{\bar t}\wedge 1\right)\f{1}{\sqrt{n}}\leq|D|^{-\f{r}{2r+s}},\\
&&\bar t^{\f{5}{2}} t\f{1}{\sqrt{n}}\f{1}{\sqrt{|D|}}\leq|D|^{-\f{r}{2r+s}},\f{t^{\f{s}{2}+1}\bar t}{n\sqrt{|D|}}\leq|D|^{-\f{r}{2r+s}}, \ \text{and} \ \f{t^{\f{3}{2}}\bar t}{n^{\f{3}{2}}\sqrt{|D|}}\leq|D|^{-\f{r}{2r+s}}.
\ees
When $t=|D|^{\f{1}{2r+s}}$, solving these inequalities, we are able to find that these inequalities hold when 
\bes
n\geq\bar t|D|^{\f{2r+\f{s}{2}}{2r+s}}\vee\bar t^{\f{3}{2}}|D|^{\f{r}{2r+s}}\vee\bar t^{\f{1}{2}}|D|^{\f{r+1}{2r+s}}\vee|D|^{\f{2r}{2r+s}}\vee\bar t^5|D|^{\f{2-s}{2r+s}}\vee\bar t|D|^{\f{1}{2r+s}}\vee\bar t^{\f{2}{3}}|D|^{\f{1-\f{s}{3}}{2r+s}}.
\ees
After combining the terms that can be absorbed by the others, we obtain that
\bes
n\geq\bar t|D|^{\f{2r+\f{s}{2}}{2r+s}}\vee\bar t^{\f{3}{2}}|D|^{\f{r}{2r+s}}\vee\bar t^5|D|^{\f{2-s}{2r+s}},
\ees
is enough to ensure 
\bes
\left\|f_{t,D_u}-\wh f_{t,D}\right\|_{L_{\rho_\huaX}^2}\lesssim_\delta\left(\log\f{256}{\delta}\right)^{4\vee (2p+2)}\Bigg[|D|^{-\f{r}{2r+s}}+\left(t^{p+1}\sigma^{-2p}+\f{1}{\sqrt{n}}t^{p+2}\sigma^{-2p}\right)\Bigg].
\ees
Noticing that, from Lemma \ref{old_GD_rate} that when $t=|D|^{\f{1}{2r+s}}$, there holds, with probability at least $1-\delta$, $\|\wh f_{t,D}- f_\rho\|_{L_{\rho_\huaX}^2}\lesssim |D|^{-\f{r}{2r+s}}\left(\log\f{12}{\delta}\right)^4$. Hence, we finally arrive at, with probability at least $1-\delta$,
\bes
\beal
\left\|f_{t,D_u}-f_\rho\right\|_{L_{\rho_\huaX}^2}\lesssim_\delta&\left(\log\f{512}{\delta}\right)^{4\vee (2p+2)}\Bigg[|D|^{-\f{r}{2r+s}}+\left(|D|^{\f{p+1}{2r+s}}\sigma^{-2p}+\f{1}{\sqrt{n}}|D|^{\f{p+2}{2r+s}}\sigma^{-2p}\right)\Bigg],
\eeal
\ees
which completes the proof of the first statement of the theorem.

Now we turn to prove the second part of the theorem.
	Based on the previous analysis, to achieve optimal learning rates, we only require
	\bes
	&&|D|^{\f{p+1}{2r+s}}\sigma^{-2p}\leq|D|^{-\f{r}{2r+s}}, \ \text{and} \ \f{1}{\sqrt{n}}|D|^{\f{p+2}{2r+s}}\sigma^{-2p}\leq|D|^{-\f{r}{2r+s}},
	\ees
	which holds when 
		\bes
	\sigma\geq|D|^{\f{p+r+1}{2p(2r+s)}}\vee\f{|D|^{\f{p+r+2}{2p(2r+s)}}}{n^{\f{1}{4p}}},
	\ees
which is exactly \eqref{sigma_condition_L2norm}, and we complete the proof.
\end{proof}
\begin{proof}[Proof of Theorem \ref{mainthm_prediction_error}]
	We recall that there holds the basic relation 
	\bes
	\huaR(f_{t,D_u})-\huaR(f_\rho)=\|f_{t,D_u}-f_\rho\|_{L_{\rho_\huaX}^2}^2, \ u\in\huaV.
	\ees
	Then it follows that
	\bes
	\huaR(f_{t,D_u})-\huaR(f_\rho)\lesssim\left \|f_{t,D_u}-\wh f_{t,D}\right\|_{L_{\rho_\huaX}^2}^2+\left\|\wh f_{t,D}-f_\rho\right\|_{L_{\rho_\huaX}^2}^2.
	\ees
	Utilizing the convexity of $\|\cdot\|_{L_{\rho_\huaX}^2}^2$ and the previous estimates, we have
	\bes
	\beal
	\huaR(f_{t,D_u})-\huaR(f_\rho)\lesssim& \left\|\huaT_{1,t}\right\|_{L_{\rho_\huaX}^2}^2+\left\|\huaT_{2,t}^A\right\|_{L_{\rho_\huaX}^2}^2+\left\|\huaT_{2,t}^B\right\|_{L_{\rho_\huaX}^2}^2+\left\|\huaT_{2,t}^{C_1}\right\|_{L_{\rho_\huaX}^2}^2+\left\|\huaT_{2,t}^{C_2,A}\right\|_{L_{\rho_\huaX}^2}^2\\
	&+\left\|\huaT_{2,t}^{C_2,B}\right\|_{L_{\rho_\huaX}^2}^2+\left\|\huaT_{3,t}\right\|_{L_{\rho_\huaX}^2}^2+\left\|\wh f_{t,D}-f_\rho\right\|_{L_{\rho_\huaX}^2}^2.
	\eeal
	\ees
	Combining the estimates in above propositions for each term and after re-scaling $\delta$, we arrive at, with confidence at least $1-\delta$, 
	\bes
	\beal
	\huaR(f_{t,D_u})-\huaR(f_\rho)\lesssim_\delta&\left(\log\f{512}{\delta}\right)^{8\vee (4p+4)}\Bigg[|D|^{-\f{2r}{2r+s}}+\left(|D|^{\f{2p+2}{2r+s}}\sigma^{-4p}+\f{1}{n}|D|^{\f{2p+4}{2r+s}}\sigma^{-4p}\right)\Bigg].
	\eeal
	\ees
	Moreover, after utilizing the rule for $\sigma$ in \eqref{sigma_condition_L2norm}, we finally obtain with confidence $1-\delta$,
	\bes
	\huaR(f_{t,D_u})-\huaR(f_\rho)\lesssim_\delta \left(\log\f{512}{\delta}\right)^{8\vee (4p+4)}|D|^{-\f{2r}{2r+s}},
	\ees
	which completes the proof.
\end{proof}
Before we proceed to prove Theorem \ref{mainthm_K_norm}, we make some crucial observations that lead to the RKHS norm bounds of the previous key terms $\huaT_{1,t}$, $\huaT_{2,t}^A$, $\huaT_{2,t}^B$, $\huaT_{2,t}^{C_1}$, $\huaT_{2,t}^{C_2,A}$, $\huaT_{2,t}^{C_2,B}$, $\huaT_{3,t}$. Throughout the analysis framework of this paper, we note that the key difference between taking the $L_{\rho_\huaX}^2$ norm and taking the RKHS norm for these key terms  primarily lies in the involvement of the operator $L_K^{1/2}$. For  $\huaT_{1,t}$, it is easy to see that, in \eqref{T1L2_first}, when estimating the $L_{\rho_\huaX}^2$ norm of $\huaT_{1,t}$, $L_K^{1/2}$ only participates in $\|L_{ K}^{1/2}\left(I-\aaa W_+'(0) L_{ K}\right)^{k-1}\|$. Hence, when $\aaa\cong1$, we note that after taking RKHS norm of $\huaT_{1,t}$, $\|\huaT_{1,t}\|_K$ and $\|\huaT_{1,t}\|_{L_{\rho_\huaX}^2}$ obviously share the same bound, and we summarize this result in the following lemma.
\begin{lem}\label{T1t_K_norm}
	Assume \eqref{moment_condition}, \eqref{regularity_ass} holds with $r>\f{1}{2}$, the stepsize $\aaa$ satisfies $0<\aaa\leq\f{1}{\kkk^2}\min\{\f{1}{W_+'(0)},\f{1}{C_W}\}$ and $\aaa\cong1$. Then for $t\in\mbb N_+$, if $|D_u|=\f{|D|}{m}=n$, $u\in\huaV$,  there holds, for any $0<\delta<1$, with probability at least $1-\delta$, 
	\bes
	\|\huaT_{1,t}\|_K\lesssim_\delta \left(\f{1}{1-\gamma_{\bm{M}}}\right)\left(\f{\sqrt{m}}{\sqrt{n}}\right)\left(\log\f{4}{\delta}\right).
	\ees
\end{lem}
We turn to analyze the RKHS norm of $\huaT_{2,t}^A$. Corresponding to procedures from \eqref{T2t_eq1} to \eqref{T2t_eq2} in Proposition \ref{T2tA_est}, if we take the RKHS norm $\huaT_{2,t}^A$, then the operator $L_K^{1/2}$ would be removed in these procedures. This would result in an additional $\f{1}{2}$ order for index $k$, and hence for index $\bar t$. Therefore,  the bound  for $\|\huaT_{2,t}^A\|_K$ would require an additional $\bar t^{\f{1}{2}}$, compared with the bound for $\|\huaT_{2,t}^A\|_{L_{\rho_\huaX}^2}$.  Hence we have the following lemma.
\begin{lem}\label{T2tA_K_norm}
		Assume \eqref{moment_condition}, \eqref{regularity_ass} holds with $r>\f{1}{2}$, the stepsize $\aaa$ satisfies $0<\aaa\leq\f{1}{\kkk^2}\min\{\f{1}{W_+'(0)},\f{1}{C_W}\}$ and $\aaa\cong1$. Then for $t,\bar t\in\mbb N_+$, $t\geq2\bar t\geq4$, if $|D_u|=\f{|D|}{m}=n$, $u\in\huaV$,  there holds, for any $0<\delta<1$, with probability at least $1-\delta$,
			\bes
		\left\|\huaT_{2,t}^A\right\|_K\lesssim_\delta\bar t^{2}\f{1}{n}\left(\log\f{4}{\delta}\right)^2.
		\ees
	\end{lem}
For $\huaT_{2,t}^B$, recalling the procedures in Proposition \ref{T2tB_est}, if we take the RKHS norm instead of $L_{\rho_\huaX}^2$ norm, once we remove the $L_K^{1/2}$ operator in \eqref{T2tB_eq1}, an additional $\f{1}{2}$ order will be given to $\bar t$. This fact would result in an additional $\bar t^{\f{1}{2}}$ term for bounding $\|\huaT_{2,t}^B\|_K$. Hence we have the following lemma.
\begin{lem}\label{T2tB_K_norm}
	Assume \eqref{moment_condition}, \eqref{regularity_ass} hold with $r>\f{1}{2}$, the stepsize $\aaa$ satisfies $0<\aaa\leq\f{1}{\kkk^2}\min\{\f{1}{W_+'(0)},\f{1}{C_W}\}$ and $\aaa\cong1$. Then for $t,\bar t\in\mbb N_+$, $t\geq2\bar t\geq4$, if $|D_u|=\f{|D|}{m}=n$, $u\in\huaV$,  there holds, for any $0<\delta<1$, with probability at least $1-\delta$,
	\bes
	\left\|\huaT_{2,t}^B\right\|_{K}\lesssim_\delta\bar tt\f{1}{n}\left(\log \f{4}{\delta}\right)^2.
	\ees
\end{lem}
It is obvious to see from the proof of Proposition \ref{T2tC1_est} that $\|\huaT_{2,t}^{C_1}\|_{K}$ and $\|\huaT_{2,t}^{C_1}\|_{L_{\rho_\huaX}^2}$ share the same high probability bound after replacing the RKHS norm. We put this fact in the following lemma.
\begin{lem}\label{T2tC1_K_norm}
	Assume \eqref{moment_condition}, \eqref{regularity_ass} hold with $r>\f{1}{2}$, the stepsize $\aaa$ satisfies $0<\aaa\leq\f{1}{\kkk^2}\min\{\f{1}{W_+'(0)},\f{1}{C_W}\}$ and $\aaa\cong1$. Then for $t,\bar t\in\mbb N_+$, $t\geq2\bar t\geq4$, if $|D_u|=\f{|D|}{m}=n$, $u\in\huaV$,  there holds, for any $0<\delta<1$, with probability at least $1-\delta$,
\bes
\left\|\huaT_{2,t}^{C_1}\right\|_{K}\lesssim_\delta  t\left(\sqrt{m}\gamma_{\bm{M}}^{\bar t}\wedge 1\right)\f{1}{\sqrt{n}}\left(\log\f{4}{\delta}\right).
\ees
\end{lem}
Now we turn to analyze the term $\huaT_{2,t}^{C_2,A}$. From the procedures of estimating $\|L_K^{1/2}g_{t+1,D_u}-\bar g_{t+1}\|_K$ in \eqref{lk_half1}, once  removing the operation of $L_K^{1/2}$, we note that the only influence is that an additional $\f{1}{\sqrt{\la_1}}$ would appear in the decomposition of $\|(I-\aaa  W_+'(0) L_{K,D})^{s-k}(L_{K,D}-L_{K,D_v})\|$, compared with the decomposition \eqref{huaH_A_dec} for $\|L_K^{1/2}(I-\aaa  W_+'(0) L_{K,D})^{s-k}(L_{K,D}-L_{K,D_v})\|$. Then, corresponding to the proof of Proposition \ref{T2tC2A_est}, once taking the RKHS norm instead of $L_{\rho_\huaX}^2$ norm for $\huaT_{2,t}^{C_2,A}$, we are able to achieve an estimate of
\bes
\beal
\left\|\huaT_{2,t}^{C_2,A}\right\|_{K}\lesssim&\aaa t\left(\max_{t'}\|\Psi_{t',D}\|_K\right)\Big[(\max_v\pdvi)(\sqrt{\la_1}+1)\aaa t\sqrt{m}\gamma_{\bm{M}}^{\bar t}\\
&+\f{1}{\sqrt{\la_1}}\huaQ_{D,\la_1}(\max _v\pdvi)\log \bar t(1\vee\la_1\aaa\bar t)\Big],   
\eeal
\ees
which is an RKHS norm counterpart of \eqref{T2tC2A_eq1}. Then, if $\la_1=(\aaa \bar t\vee1)^{-1}$, an additional term $(\aaa \bar t\vee1)^{\f{1}{2}}$ would appear in the estimate of $\|\huaT_{2,t}^{C_2,A}\|_{K}$, compared to the previous estimate for $\|\huaT_{2,t}^{C_2,A}\|_{L_{\rho_\huaX}^2}$. Following similar procedures as in the remaining parts after \eqref{T2tC2A_eq1} in the proof of Proposition \ref{T2tC2A_est}, we arrive at the following lemma.
\begin{lem}\label{T2tC2A_K}
	Assume \eqref{moment_condition}, \eqref{capacity_ass} with $0<s\leq1$, \eqref{regularity_ass} holds with $r>\f{1}{2}$, the stepsize $\aaa$ satisfies $0<\aaa\leq\f{1}{\kkk^2}\min\{\f{1}{W_+'(0)},\f{1}{C_W}\}$ and $\aaa\cong1$. Then for $t,\bar t\in\mbb N_+$, $t\geq2\bar t\geq4$, if $|D_u|=\f{|D|}{m}=n$, $u\in\huaV$,  there holds, for any $0<\delta<1$, with probability at least $1-\delta$,
	\bes
	\beal
	\left\|\huaT_{2,t}^{C_2,A}\right\|_K\lesssim_\delta\bar t^{\f{1}{2}}\left[\bar t^{\f{5}{2}}+ t\sqrt{m}\gamma_{\bm{M}}^{\bar t}\right] t\f{1}{\sqrt{n}}\f{1}{\sqrt{|D|}}\left(\log\f{32}{\delta}\right)^4.
	\eeal
	\ees
\end{lem}
For $\huaT_{2,t}^{C_2,B}$, by following similar ideas of getting bound for $\|\huaT_{2,t}^{C_2,A}\|_K$, and according to the previous procedures of Proposition \ref{T2tC2B_est}, corresponding to \eqref{T2tC2B_eq1}, the cost is two additional terms $\f{1}{\sqrt{\la_2}}$ and $\f{1}{\sqrt{\la_3}}$. Accordingly, we have, the counterpart of \eqref{T2tC2B_eq1}
\bes
\beal
\left\|\huaT_{2,t}^{C_2,B}\right\|_{K}\lesssim&\left(\max_v\huaA_{D_v,\la_2}\right)\left(\max_v\huaA_{D_v,\la_3}\right)\left[\left(\f{\huaA_{D,\la_3}}{\sqrt{\la_3}}\right)^2+1\right]\left(\log m\right)^2\left(\log\f{16}{\delta}\right)^4\f{1}{\sqrt{\la_3}}\\
&\left\|(\la_3I+L_K)^{1/2}\right\|\f{1}{\sqrt{|D|}}\aaa t\left(\aaa t\sqrt{m}\gamma_{\bm{M}}^{\bar t}+\aaa\bar t\right)\log \bar t(1\vee\la_2\aaa\bar t)\f{1}{\sqrt{\la_2}}. 
\eeal
\ees
Once taking $\la_2=(\aaa t)^{-1}$, $\la_3=\kkk^2$, we know an additional $t^{\f{1}{2}}$ term will appear in the final estimate for $\|\huaT_{2,t}^{C_2,B}\|_{K}$, compared to the bound for $\|\huaT_{2,t}^{C_2,B}\|_{L_{\rho_\huaX}^2}$. 
\begin{lem}\label{T2tC2B_K}
	Assume \eqref{moment_condition}, \eqref{capacity_ass} with $0<s\leq1$, \eqref{regularity_ass} holds with $r>\f{1}{2}$, the stepsize $\aaa$ satisfies $0<\aaa\leq\f{1}{\kkk^2}\min\{\f{1}{W_+'(0)},\f{1}{C_W}\}$ and $\aaa\cong1$. Then for $t,\bar t\in\mbb N_+$, $t\geq2\bar t\geq4$, if $|D_u|=\f{|D|}{m}=n$, $u\in\huaV$,  there holds, for any $0<\delta<1$, with probability at least $1-\delta$,
	\bes
	\left\|\huaT_{2,t}^{C_2,B}\right\|_K\lesssim_{\delta}\left(\f{ t^{\f{s}{2}}}{\sqrt{n}}+\f{t^{\f{1}{2}}}{n}\right)\f{1}{\sqrt{n}}\f{1}{\sqrt{|D|}} t^{\f{3}{2}}\left( t\sqrt{m}\gamma_{\bm{M}}^{\bar t}+\bar t\right)\left(\log\f{16}{\delta}\right)^4.
	\ees
\end{lem}
Finally, let us deal with $\|\huaT_{3,t}\|_K$. Revisiting the proof of Proposition \ref{T3t_est}, we note that, based on previous insights, the cost of replacing the $L_{\rho_\huaX}^2$ norm with RKHS  norm for $\huaT_{3,t}$ results in an additional $t^{\f{1}{2}}$ term in the final bound for $\|\huaT_{3,t}\|_K$. We summarize this result in the following lemma.
\begin{lem}\label{T3t_K_norm}
	Assume \eqref{moment_condition} holds and the windowing function $W$ satisfies basic conditions \eqref{window_ass1} and \eqref{window_ass2}.	If the stepsize $\aaa$ satisfies  $0<\aaa \leq\f{1}{\kkk^2}\min\{ \frac{1}{C_W},\f{1}{W_+'(0)}\}$ and $\aaa\cong1$, then, for each $u\in\huaV$, $t\in\mbb N_+$, we have, for any $0<\delta<1$, with probability at least $1-\delta$, 
	\bes
	\left\|\huaT_{3,t}\right\|_{K}\lesssim_\delta \left(t^{p+\f{3}{2}}\sigma^{-2p}+\f{1}{\sqrt{n}}t^{p+\f{5}{2}}\sigma^{-2p}\right)\left(\log\f{4m}{\delta}\right)\left(\log\f{2|D|}{\delta}\right)^{2p+1}.
	\ees
\end{lem}
Equipped with these lemmas, we are ready to provide the proof of Theorem \ref{mainthm_K_norm}.

\begin{proof}[Proof of Theorem \ref{mainthm_K_norm}] Combining Lemma \ref{T1t_K_norm}-Lemma \ref{T3t_K_norm}, noticing the fact that, when $\bar t=\f{2\log|D|t}{1-\gamma_{\bm{M}}}$, there holds $t\sqrt{m}\gamma_{\bm{M}}^{\bar t}\leq1$ and re-scaling $\delta$, we obtain the desired high probability bound for $\|f_{t,D_u}-\wh f_{t,D}\|_K$.
\end{proof}

\begin{proof}[Proof of Theorem \ref{mainthm_K_with_sigma}]
	According to the result of Theorem \ref{mainthm_K_norm}, we know, to achieve that, with probability at least $1-\delta$,
	\bes
	\beal
	\left\|f_{t,D_u}- \wh f_{t,D}\right\|_K\lesssim_\delta&\left(\log\f{512}{\delta}\right)^{4\vee (2p+2)}\Bigg[|D|^{-\f{r-\f{1}{2}}{2r+s}}+\left(|D|^{\f{p+\f{3}{2}}{2r+s}}\sigma^{-2p}+\f{1}{\sqrt{n}}|D|^{\f{p+\f{5}{2}}{2r+s}}\sigma^{-2p}\right)\Bigg],
	\eeal
	\ees 
	we only require 
\bes
\bar t\left(\f{\sqrt{m}}{\sqrt{n}}\right)\vee\bar t^{2}\f{1}{n}\vee\bar tt\f{1}{n}\vee\f{1}{\sqrt{n}}\vee\bar t^3t\f{1}{\sqrt{n}}\f{1}{\sqrt{|D|}}\vee\f{\bar tt^{\f{s+3}{2}}}{n|D|^{\f{1}{2}}}\vee\f{\bar tt^2}{n^{\f{3}{2}}|D|^{\f{1}{2}}}\leq|D|^{-\f{r-\f{1}{2}}{2r+s}}.
\ees
When $t=|D|^{\f{1}{2r+s}}$, the above inequality holds when 
\bes
n\geq\bar t|D|^{\f{2r+\f{s}{2}-\f{1}{2}}{2r+s}}\vee\bar t^2|D|^{\f{r-\f{1}{2}}{2r+s}}\vee\bar t|D|^{\f{r+\f{1}{2}}{2r+s}}\vee|D|^{\f{2r-1}{2r+s}}\vee\bar t^6|D|^{\f{1-s}{2r+s}}\vee \bar t|D|^{\f{1}{2r+s}}\vee\bar t^{\f{2}{3}}|D|^{\f{1-\f{s}{3}}{2r+s}}.
\ees
It can be verified that
$|D|^{\f{2r-1}{2r+s}}$, $\bar t^6|D|^{\f{1-s}{2r+s}}$ and $\bar t^{\f{2}{3}}|D|^{\f{1-\f{s}{3}}{2r+s}}$ can be absorbed by other components, hence we can simplify the above inequality as
\bes
n\geq\bar t|D|^{\f{2r+\f{s}{2}-\f{1}{2}}{2r+s}}\vee\bar t^2|D|^{\f{r-\f{1}{2}}{2r+s}}\vee\bar t|D|^{\f{r+\f{1}{2}}{2r+s}}\vee\bar t^6|D|^{\f{1-s}{2r+s}},
\ees
which is exactly the condition \eqref{n_condition_for__Knorm}.  
On the other hand, according to Lemma \ref{old_GD_rate}, we know, with probability at least $1-\delta$, there holds
\bes
\left\|\wh f_{t,D}-f_\rho\right\|_K\lesssim|D|^{-\f{r-\f{1}{2}}{2r+s}}\left(\log\f{12}{\delta}\right)^4.
\ees
Based on the above estimates, after re-scaling on $\delta$, we have proved the first desired bound in Theorem \ref{mainthm_K_with_sigma}.

We turn to prove the second part of the theorem.
	To ensure optimal rates  $\huaO(|D|^{-\f{r-\f{1}{2}}{2r+s}})$ for $\|f_{t,D_u}-f_\rho\|_K$, we only require 
	\bes
	|D|^{\f{p+\f{3}{2}}{2r+s}}\sigma^{-2p}\vee\f{1}{\sqrt{n}}|D|^{\f{p+\f{5}{2}}{2r+s}}\sigma^{-2p}\leq|D|^{-\f{r-\f{1}{2}}{2r+s}}.
	\ees
	By solving this inequality, we obtain
	\bes
	\sigma\geq|D|^{\f{p+r+1}{2p(2r+s)}}\vee\f{|D|^{\f{p+r+2}{2p(2r+s)}}}{n^{\f{1}{4p}}},
	\ees
	which is exactly the condition \eqref{sigma_K_rule} and the proof is complete.
\end{proof}

\section{Numerical experiments}

In this section, we conduct numerical experiments to illustrate various features of the decentralized robust kernel-based gradient descent (DKBRGD) algorithm, with the windowing function $W$ and robustness scaling parameter $\sigma$, which is defined by $f_{0,D_v}=0$, $v\in\huaV$ (initialization for each local node) and 
\bea
&&\phi_{t,D_v}={f}_{t, D_v}-\frac{\aaa}{|D_v|} \sum_{z \in D_v}W'\left(\frac{\xi_{t,D_v}^2\left(z\right)}{ \sigma^2}\right) \xi_{t,D_v}\left(z\right) K_{x}, \\
&&{f}_{t+1,D_{u}}=\sum_{v}\big[\bm{M}\big]_{uv}\phi_{t,D_v},  
\eea
where $\xi_{t,D_v}\left(z\right)={f}_{t, D_v}(x)-y, z=(x, y)$.
The dataset $D_v=\{x_{v,s},y_{v,s}\}_{s=1}^{|D_v|}$ is known only to agent $v$. We adopt the following configurations in our numerical simulations. The Mercer kernel $K$ is chosen as Gaussian kernel $K(x,y)= \exp\left(-\frac{\|x-y\|_2^2}{\gamma^2}\right)$, $x,y \in \mathbb{R}$. The hyperparameters in the DKBRGD algorithm are chosen as follows: we consider the number of agents $m=25$, the number of samples in each agent $|D_1|= |D_2| = \cdots = |D_m| = \frac{|D|}{m}$, the bandwidth of Gaussian kernel $\gamma=0.1$, and the robustness scaling parameter $\sigma=10$. The input $x_{v,s}$ is generated from the uniform distribution over the interval $[-5,5]$, and the response is generated by
$$y_{v,s} = \frac{\sin(x_{v,s}) }{x_{v,s}} + \epsilon,$$
where the noise $\epsilon$ is drawn i.i.d. from a Laplace distribution $L(0,1)$. The robust loss is chosen as the Cauchy loss with $W(x)=\log(1+\frac{x}{2})$. The stepsize is selected as $\alpha=0.1 $. In our experiments, we evaluate the excess generalization error $ \{\huaR(f_{T,D_v})-\huaR(f_\rho)\}_{v=1}^m$, where $T=O(\sqrt{|D|})$. The excess generalization error of the local estimators $f_{T,D_v}$ is calculated using a testing sample of size 5000. All the experiment results are based on the average of 10 runs.

\begin{figure}
	\centering
	\includegraphics[width=0.85\textwidth]{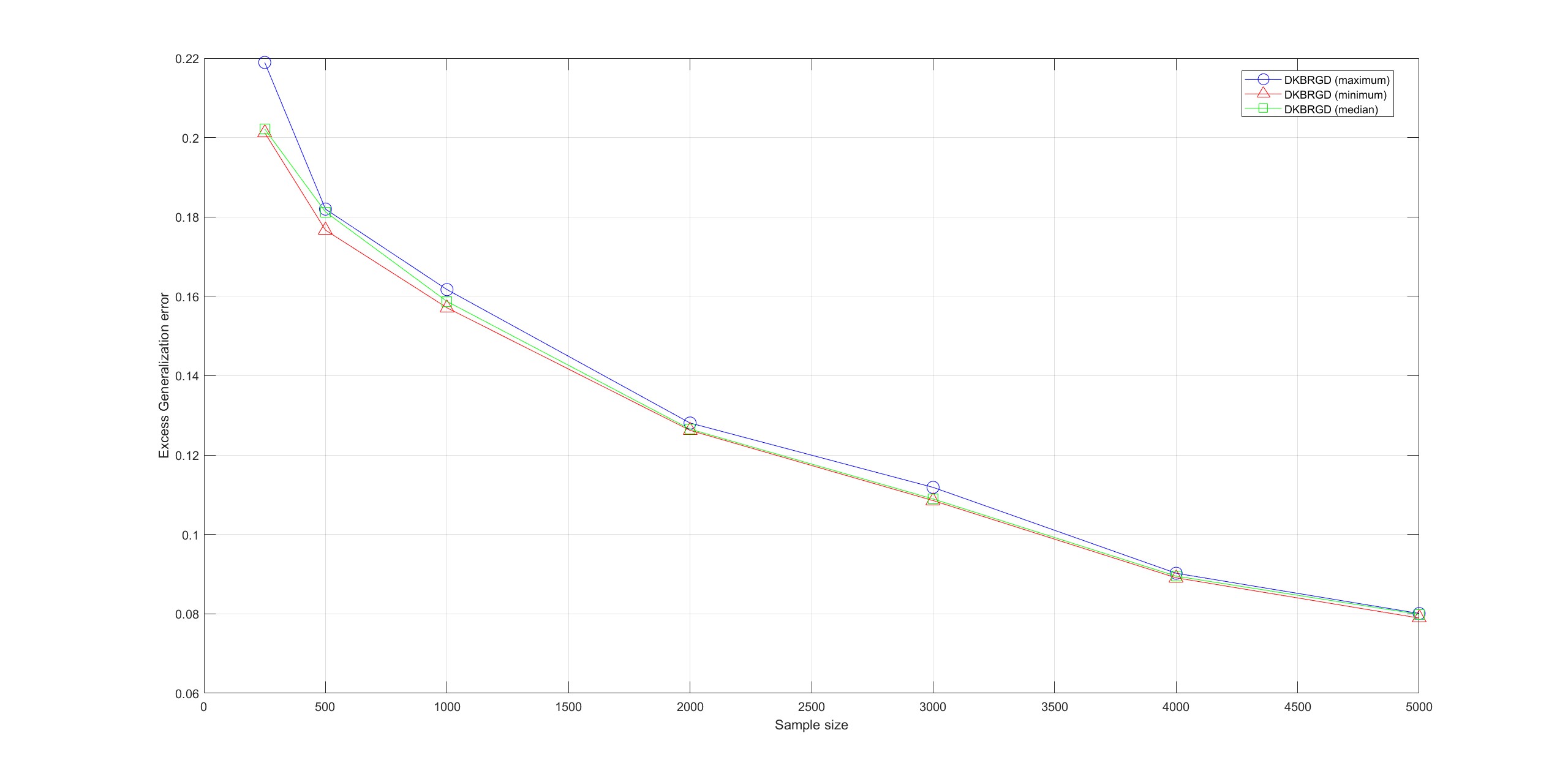}
	\caption{\small{Maximum, minimum, and median of the excess generalization errors across all agents versus the data sample size $|D|$ of the DKBRGD algorithm.}}
	\label{figure1}
\end{figure}

We illustrate the convergence behavior of the DKBRGD algorithm by plotting the maximum, minimum, and median of the excess generalization errors $\{\huaR(f_{T,D_v})-\huaR(f_\rho)\}_{v=1}^m$, against the total number of data sample size $|D|$. The results are depicted in Fig. \ref{figure1}, demonstrating that the excess generalization errors of all agents in the DKBRGD algorithm converge at a consistent rate.

\begin{figure}
	\centering
	\includegraphics[width=0.85\textwidth]{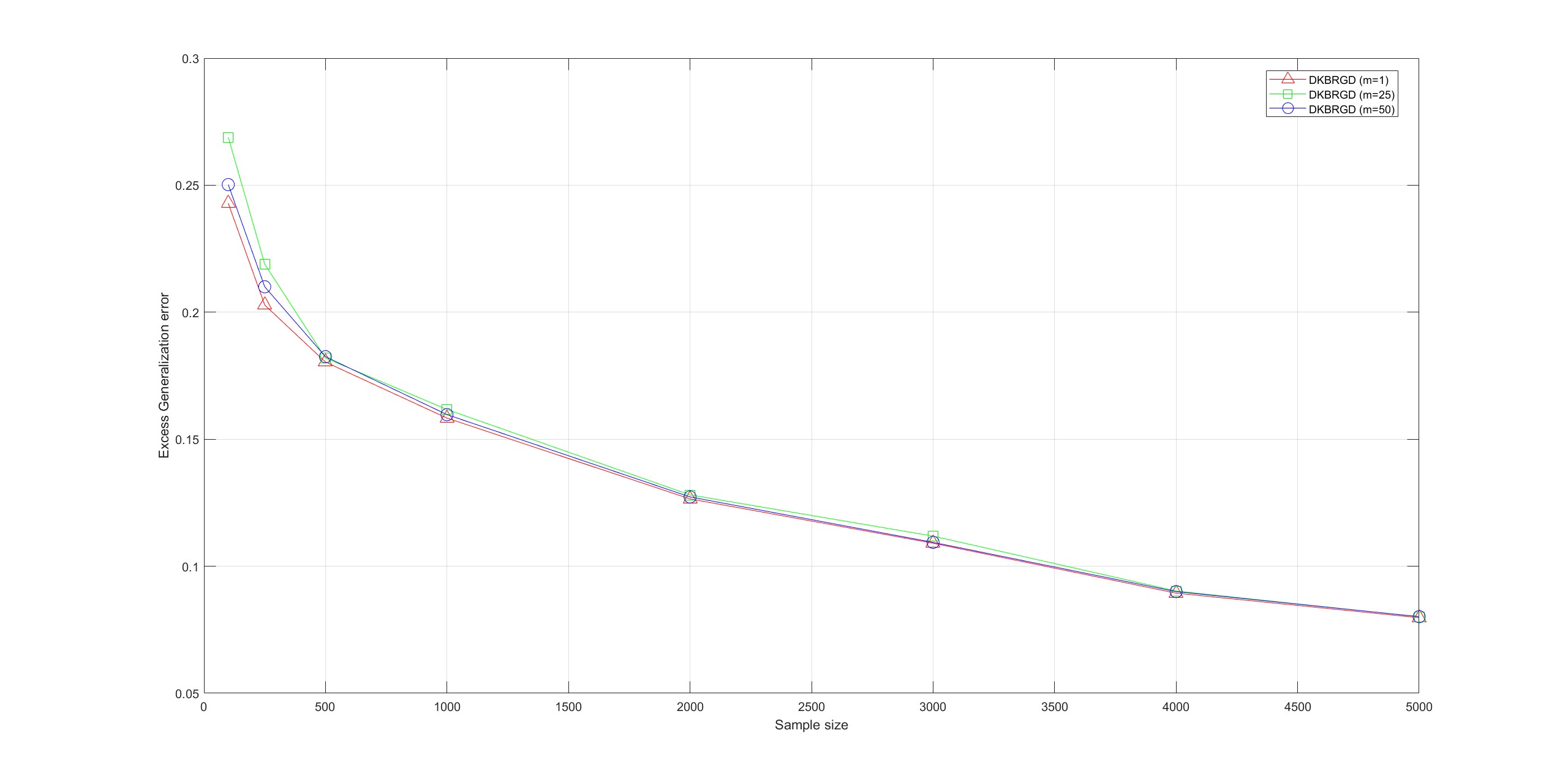}
	\caption{\small{Maximum of the excess generalization errors across all agents versus the data sample size $|D|$ of the DKBRGD algorithm for three different choices of the number of agents $m$.}}
	\label{figure2}
\end{figure}

Next, we examine how the network size (number of agents $m$) affects the learning rate of the DKBRGD algorithm. We analyze three distinct numbers of agents in a ring network: $m = 1$,   $m=25$, and $m = 50$, and present the graphs of the maximum of excess generalization errors across all agents versus the number of data sample size $|D|$. The results illustrated in Fig. \ref{figure2} indicate that the excess generalization error of varying network size converges at a satisfying rate. However, when the network size is large and the data sample size is small, the excess generalization error might be slightly worse due to a limited local sample size.

\begin{figure}
	\centering
	\includegraphics[width=0.85\textwidth]{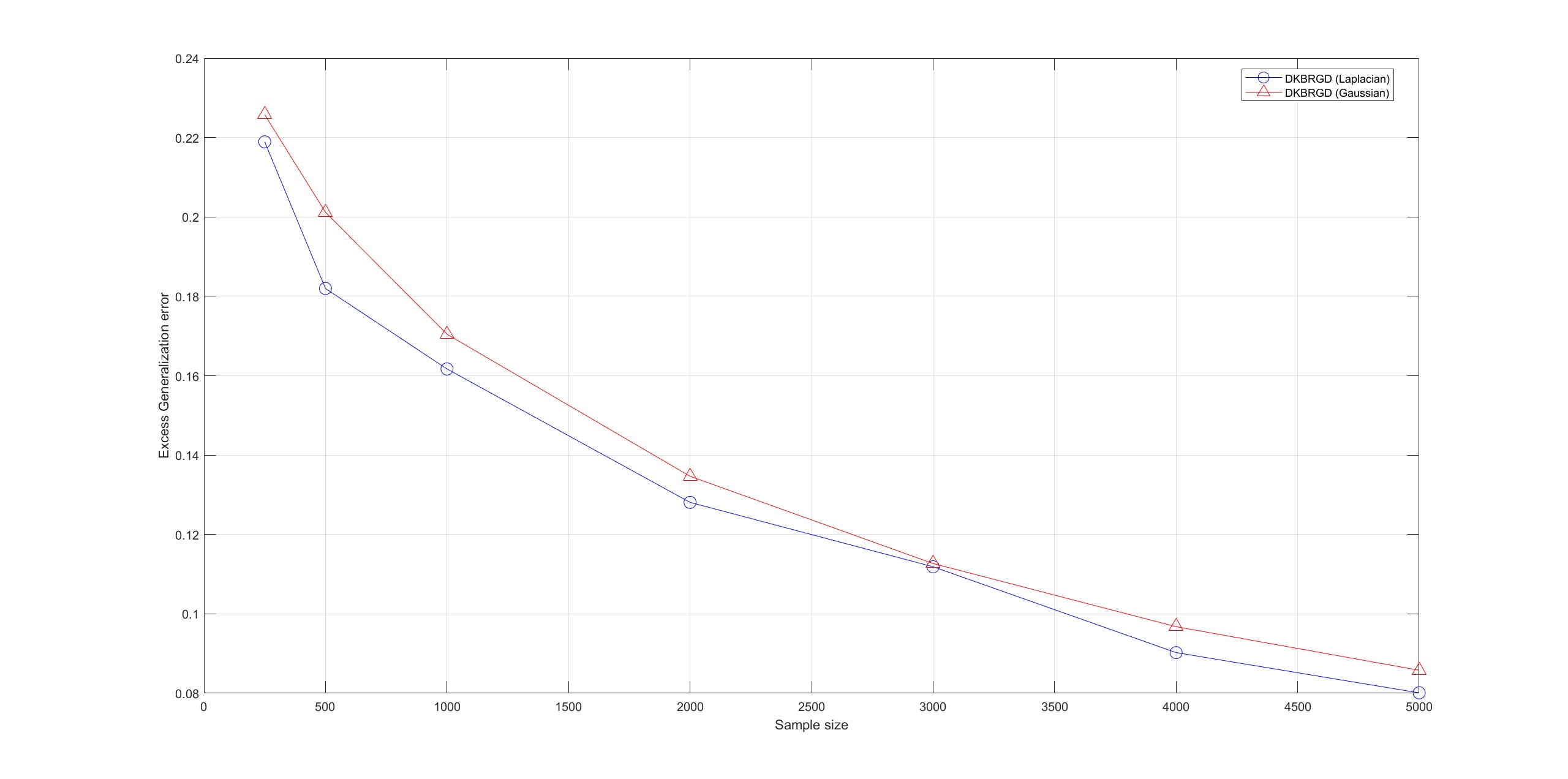}
	\caption{\small{Maximum of the excess generalization errors across all agents versus the data sample size $|D|$ of the DKBRGD algorithm for two different choices of the noise distribution.}}
	\label{figure3}
\end{figure}

We also investigate the influence of noise $\epsilon$ on the learning rate of the DKBRGD algorithm. We select two different noise distributions: the Laplace distribution $L(0,1)$ and the Gaussian distribution $N(0,0.3^2)$, and plot the maximum of the excess generalization errors across all agents against the number of data sample size $|D|$. The findings shown in Fig. \ref{figure3} clearly indicate that DKBRGD performs well in both noise distributions and therefore is capable of dealing with non-Gaussian noise distributions effectively as well.

\begin{figure}
	\centering
	\includegraphics[width=0.85\textwidth]{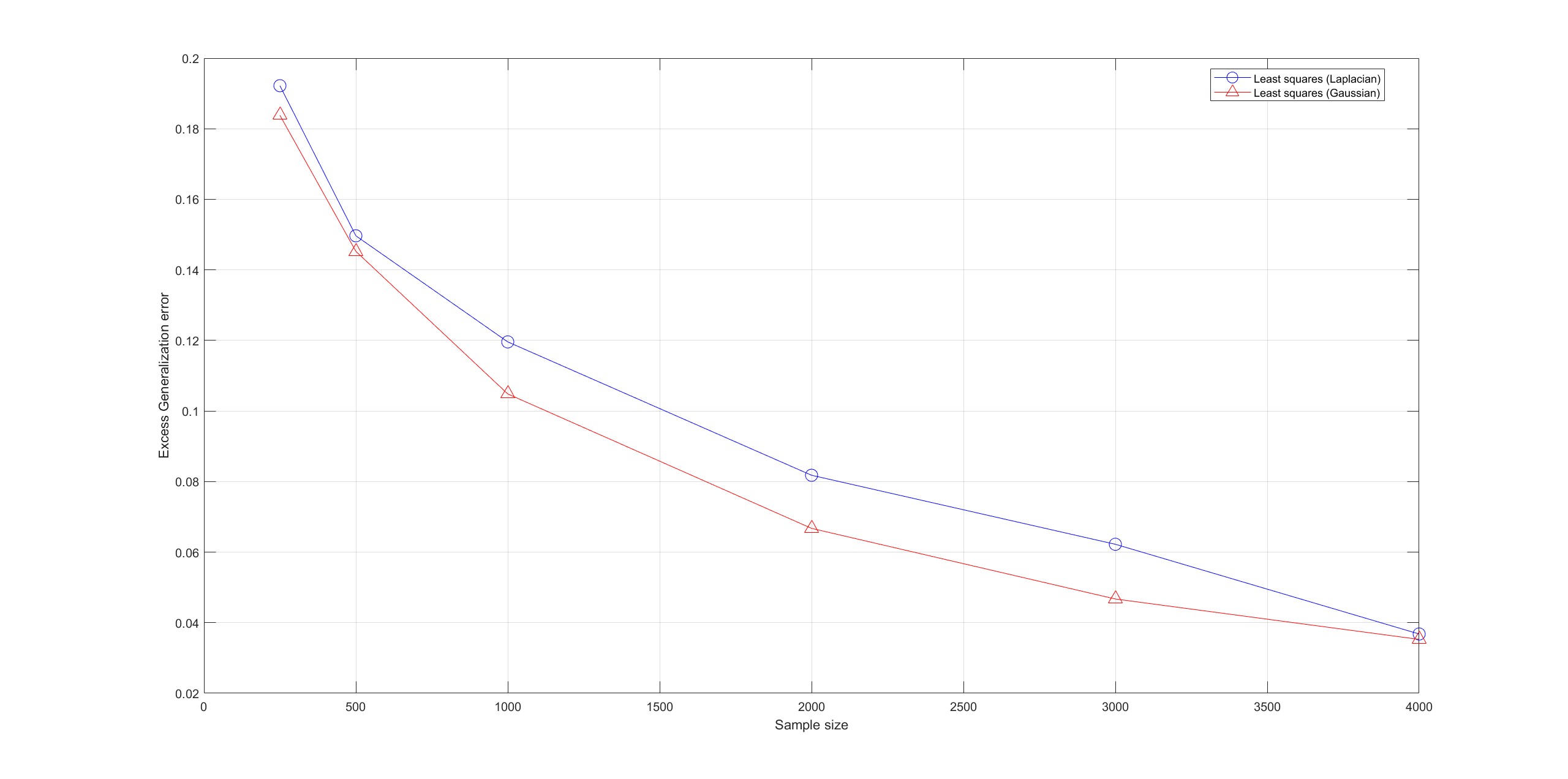}
	\caption{\small{Maximum of the excess generalization errors across all agents versus the data sample size $|D|$ of the DKBGD algorithm for two different choices of the noise distribution.}}
	\label{figure4}
\end{figure}

Additionally, we investigate the effect of noise $\epsilon$ on the learning rate of the decentralized kernel-based gradient descent (DKBGD) algorithm, which is defined by $f_{0,D_v}=0$, $v\in\huaV$ (initialization for each local node) and 
\bea
&&\phi_{t,D_v}={f}_{t, D_v}-\frac{\aaa}{|D_v|} \sum_{z \in D_v}\xi_{t,D_v}\left(z\right) K_{x}, \\
&&{f}_{t+1,D_{u}}=\sum_{v}\big[\bm{M}\big]_{uv}\phi_{t,D_v}.
\eea
The configurations remain consistent with the previous experiments. We select the same two distinct noise distributions: the Laplace distribution $L(0,1)$ and the Gaussian distribution $N(0,0.3^2)$. In Fig. \ref{figure4}, we display plots of the maximum excess generalization errors across all agents against the number of data sample size $|D|$. The results suggest that the DKBGD algorithm exhibits poorer performance when dealing with non-Gaussian noise distributions. This indicates that the DKBRGD algorithm has a significant advantage in enhancing robustness against non-Gaussian noise compared to standard least squares learning methods.




\bibliography{ref}

\end{document}